%% file: main.tex
\documentclass[11pt,letterpaper,english]{article}

\usepackage{multirow}
\usepackage[para,online,flushleft]{threeparttable}

\usepackage{subcaption}
\usepackage{graphicx}   
\usepackage{caption}    
\usepackage[utf8]{inputenc} 
\usepackage[T1]{fontenc}    
\usepackage{hyperref}       
\usepackage{url}            
\usepackage{booktabs}       
\usepackage{array}
\newcolumntype{M}[1]{>{\centering\arraybackslash}m{#1}}

\usepackage{amsfonts}       
\usepackage{nicefrac}       
\usepackage{microtype}      
\usepackage{lipsum}         
\usepackage{graphicx}
\usepackage{amsmath}
\usepackage{mathtools}
\usepackage{setspace}
\usepackage[small, compact]{titlesec}
\usepackage{bm}
\usepackage{cleveref} 
\usepackage{amssymb}
\usepackage{amsthm}
\newtheorem{theorem}{Theorem}
\newtheorem{definition}{Definition}
\newtheorem{assumption}{Assumption}

\newtheorem{lemma}{Lemma}
\newtheorem{corollary}{Corollary}
\newtheorem{remark}{Remark}
\usepackage[most]{tcolorbox}
\tcbuselibrary{skins}
\usepackage{xcolor}
\usepackage{tikz}
\usepackage{lmodern}
\usepackage[T1]{fontenc}
\usepackage{newtxtext} 
\usepackage[authoryear]{natbib}
\usepackage{tikz}
\usetikzlibrary{arrows.meta,positioning,calc}
\usetikzlibrary{fit,backgrounds}
\usetikzlibrary{decorations.pathreplacing, calligraphy}
\usepackage{abstract}

\usepackage[hang,flushmargin]{footmisc}
\usepackage{fullpage} 
\usepackage{hyperref}
\usepackage{xcolor}
\usepackage[small,compact]{titlesec} 
\usepackage[toc,page]{appendix}
\usepackage{bibunits}
\usepackage{bibunits}
\defaultbibliographystyle{abbrvnat} 
\defaultbibliography{references} 
\usepackage{listings}
\bibpunct[, ]{(}{)}{;}{a}{}{,}

\usepackage{graphicx} 

\usepackage{doi}
\usepackage{enumitem}
\usepackage{adjustbox}
\usepackage[linesnumbered,ruled,vlined,noend]{algorithm2e}
\DontPrintSemicolon
\SetKwComment{Comment}{\# }{} 

\definecolor{promptbg}{RGB}{245,245,245} 

\newcommand{\squishlist}{
   \begin{list}{$\bullet$}
    { \setlength{\itemsep}{0pt} \setlength{\parsep}{1pt}
      \setlength{\topsep}{1pt} \setlength{\partopsep}{1pt}
      \setlength{\leftmargin}{1.5em} \setlength{\labelwidth}{1em}
      \setlength{\labelsep}{0.5em} } }

\newcommand{\squishlisttwo}{
   \begin{list}{$\bullet$}
    { \setlength{\itemsep}{0pt} \setlength{\parsep}{0pt}
      \setlength{\topsep}{0pt} \setlength{\partopsep}{0pt}
      \setlength{\leftmargin}{1em} \setlength{\labelwidth}{1.5em}
      \setlength{\labelsep}{0.5em} } }

\newcommand{\squishend}{
    \end{list}  }

\def\hy#1{{\bf \color{blue}#1}}

\begin{document}
\title{TextBO: Bayesian Optimization in Language Space for\\ Eval-Efficient Self-Improving AI}

\author{
Enoch Hyunwook Kang\thanks{We thank Khaled Boughanmi
 for detailed comments that have improved this work. Thanks are also due to the participants of the 2025 Stanford AI-ML conference and the 2025 Columbia AIML conference. Please address all correspondance to: ehwkang@uw.edu and hemay@uw.edu.}\\ University of Washington
\and
Hema Yoganarasimhan\\
University of Washington
}

\date{}

\maketitle

\begin{abstract} \small
Large Language Models (LLMs) have enabled self-improving AI systems that iteratively generate, evaluate, and refine their outcomes. Recent studies show that prompt-optimization-based self-improvement can outperform state-of-the-art reinforcement-learning fine-tuning of LLMs, but performance is typically measured by \emph{generation} efficiency. However, in many applications, the constraint is \emph{evaluation} efficiency: obtaining reliable feedback is far more costly than generating candidates. To optimize for evaluation efficiency, we extend Upper Confidence Bound--Bayesian Optimization (UCB-BO), a framework known for optimal evaluation-efficiency guarantees, to the language domain. Doing so is challenging for two reasons: (i) gradients needed for UCB-BO are ill-defined in discrete prompt space; and (ii) UCB-style exploration relies on a surrogate model and acquisition function, which only live implicitly in the LLM. We overcome these challenges by proving that combining simple textual gradients (LLM-proposed local edits) with the \textsc{Best-of-N} selection strategy statistically emulates ascent along the gradient of the canonical UCB acquisition function. Based on this result, we propose \textsc{TextBO}, a simple, evaluation-efficient self-improving algorithm that operates purely in language space without explicit surrogates or calibrated uncertainty models. We empirically validate \textsc{TextBO} on automated ad-alignment tasks using a persona-induced preference distribution, demonstrating superior performance per evaluation compared to strong baselines such as \textsc{Best-of-N} and \textsc{GEPA}. We also evaluate \textsc{TextBO}’s \textsc{Best‑of‑N} multi‑step textual‑gradient mechanism on agentic AI benchmarks by augmenting \textsc{GEPA} with it and show that it performs better than standard \textsc{GEPA}. In sum, \textsc{TextBO} is a simple and principled framework for AI self-improving system design that bridges prompt optimization with classical Bayesian optimization.

\end{abstract}

\newpage

\section{Introduction}
\label{sec:intro}

\subsection{Self-improving AI for Data-driven Decision-making}

Data-centric decision making has long relied on an iterative, human-driven cycle of proposing solutions, testing them, and learning from the results. This paradigm underlies many business applications, ranging from website design \citep{hauser2009website} to product recommendations \citep{nandy2021b}. We illustrate the cycle using ad optimization (left panel of Figure~\ref{fig:human-ai-ad-cycle}). An ad agency or marketer first generates a slate of creatives from a brief, drawing on product knowledge, theories of consumer behavior, market conditions, and feedback from prior campaigns. This \textit{generation} step often takes weeks. The candidates are then \textit{evaluated} (e.g., on advertising platforms or via consumer surveys) to identify the best performers, typically over days to weeks. Finally, the manager \textit{analyzes} campaign outcomes and uses the insights to guide the next round of creative generation.

This workflow is effective, but it is also constrained by two longstanding limitations. First, iteration speed depends heavily on how quickly humans can \textit{generate} new candidate solutions, which has historically been the main bottleneck. Second, while evaluation has been studied extensively---with tools such as A/B testing and Multi-Armed Bandits (MAB) \citep{kohavi2020trustworthy, fiez_etal_2024}---there is no comparable formal framework for how practitioners should systematically \textit{analyze} results and translate them into improved candidates; in practice, this step is often driven by managerial intuition.

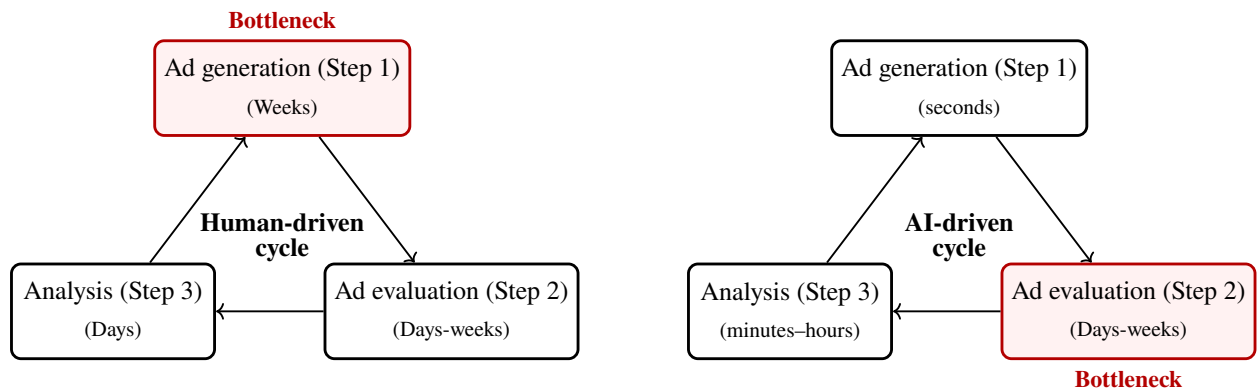
\begin{figure}[!ht] 
  \centering
  \scalebox{0.9}{%
    \begin{tikzpicture}[
      box/.style = {
        rectangle,
        rounded corners,
        draw=black,
        very thick,
        minimum width=3cm,
        minimum height=1.4cm,
        align=center,
        inner sep=4pt
      },
      bottleneck/.style = {
        rectangle,
        rounded corners,
        draw=red!70!black,
        very thick,
        fill=red!5,
        minimum width=3cm,
        minimum height=1.4cm,
        align=center,
        inner sep=4pt
      },
      middlearrow/.style = {
        line width=1.6pt,
        -{Stealth[length=7pt,width=9pt]},
        line cap=round
      }
    ]

    \begin{scope}[shift={(-5,0)}]
      \node[bottleneck] at (0,0)      (genL)      {Ad generation (Step 1)\\[2pt]\footnotesize (Weeks)};
      \node[box]        at (-2.5,-3.3)  (analysisL) {Analysis  (Step 3)\\[2pt]\footnotesize (Days)};
      \node[box]        at ( 2.5,-3.3)  (evalL)     {Ad evaluation (Step 2)\\[2pt]\footnotesize (Days-weeks)};

      \node[font=\bfseries] at (0,-2.2) {\shortstack{Human-driven\\cycle}};

      \draw[->, thick] (analysisL) -- (genL);
      \draw[->, thick] (genL)      -- (evalL);
      \draw[->, thick] (evalL)     -- (analysisL);

      \node[font=\small\bfseries, text=red!70!black] at (0,1.0) {Bottleneck};

      \coordinate (midL) at (4.3,-1.5);
    \end{scope}

    \begin{scope}[shift={(5,0)}]
      \node[box]        at (0,0)      (genR)      {Ad generation (Step 1)\\[2pt]\footnotesize (seconds)};
      \node[box]        at (-2.5,-3.3)  (analysisR) {Analysis (Step 3)\\[2pt]\footnotesize (minutes--hours)};
      \node[bottleneck] at ( 2.5,-3.3)  (evalR)     {Ad evaluation (Step 2)\\[2pt]\footnotesize (Days-weeks)};

      \node[font=\bfseries] at (0,-2.2) {\shortstack{AI-driven\\cycle}};

      \draw[->, thick] (analysisR) -- (genR);
      \draw[->, thick] (genR)      -- (evalR);
      \draw[->, thick] (evalR)     -- (analysisR);

      \node[font=\small\bfseries, text=red!70!black] at (2.5,-4.3) {Bottleneck};

      \coordinate (midR) at (-4.3,-1.5);
    \end{scope}


    \end{tikzpicture}
  }

  \caption{Comparison of human-driven and AI-driven ad self-improvement cycles in digital advertising.}
  \label{fig:human-ai-ad-cycle}
\end{figure}

Recent advances in generative AI have begun to transform this human-driven paradigm by addressing the issues discussed above. Candidate generation is now fast and scalable: in digital advertising, AI models can produce high-quality ad copy, images, and campaign concepts that adhere to brand guidelines in minutes \citep{jansen2024automated, hartmann2025power}.\footnote{A large share of ad creatives is now AI-generated, reflecting improvements in both quality and speed \citep{coffee2025aiadbuying, Majic2025VibeMarketing}.} In parallel, LLM-based systems can autonomously conduct analyses by summarizing performance patterns and hypothesizing drivers of success \citep{ferber2024context, guo2024ds, ghosh2024exploring, wiedemer2025videomodelszeroshotlearners}. Together, these capabilities make it feasible to automate the full generation--evaluation--analysis loop at scale, allowing systems to update themselves continuously in response to new data (right panel of Figure~\ref{fig:human-ai-ad-cycle}). We refer to such automated, continuously updating systems as \textit{self-improving} or \textit{self-evolving} AI \citep{MantiaChatterjeeLee2025, chen2025xbench, silver2025welcome}. While there can be many types of self-improving AI  systems, a central line of progress has been on \textit{iterative prompt optimization}, where the AI agent repeatedly refines its own prompts based on performance feedback. Indeed, recent papers show that such approaches can often outperform reinforcement-learning based fine-tuning \citep{agrawal2025gepa}; see $\S$\ref{ssec:related_lit} for a detailed discussion. The key performance objective that this stream of work focuses on is \textit{generation efficiency}---the number of LLM calls (or candidate solution generations) needed to reach a target performance level.

However, in most business settings, the key bottleneck is a limited evaluation budget rather than a generation budget. For instance, in digital advertising, generating ad creatives is now cheap and fast with modern LLMs/AI models, but ad evaluation, i.e., measuring whether the ad is effective or not, still requires significant human feedback in the form of field deployment (e.g., A/B tests or bandits in ad platforms) or consumer surveys. This is also true for other marketing decisions such as package design, product descriptions, and promotions -- in all these cases, generating new candidate solutions is now easy, but evaluating the performance of each solution can take significant time, effort, and human feedback. Beyond business applications, this pattern is also common in many social and scientific domains: AI models like AlphaFold can propose protein structures quickly \citep{jumper2021highly}, but their functional properties must still be tested through slow and noisy wet-lab experiment; similarly, AlphaChip can generate chip layouts in silico \citep{goldie2024chip}, but performance ultimately depends on fabrication and benchmarking. In sum, the primary goal for self-improving AI systems in  business and social settings should be \textit{evaluation efficiency}. However, as discussed earlier, the existing work has largely focused on generation efficiency as the main objective.



\subsection{Research Agenda and Challenges}
In this paper, we seek to design an evaluation-efficient, theoretically grounded, and empirically effective framework for prompt-optimization-based self-improving AI.
We define \textit{evaluation efficiency} as: reaching high-performing solutions using as few costly evaluations as possible.


Our key and novel insight is that iterative prompt-optimization-based self-improving AI can be cast as \textit{Bayesian Optimization (BO)} in the language space. This recasting then allows us to inherit the key properties of classical  BO, in particular, it's evaluation efficiency. While intuitive and appealing, we nevertheless need to overcome technical challenges in order to translate BO to the language space. 

We now briefly summarize canonical BO and discuss the parallels between our problem and canonical BO and provide some intuition for why iterative prompt optimization can be viewed as BO in language space, and the describe the challenges that we need to address to fully leverage the BO framework for our problem.

\paragraph{Bayesian Optimization (BO).} Classical BO is a principled framework for tackling black-box optimization problems with two defining features: (1) There is no closed-form or analytical expression for the objective function or its derivative. In such cases, optimization has to rely on evaluating the function by sampling different points and getting a noisy response. (2) Function evaluation is costly. If evaluation is cheap, then this problem can be solved easily using methods such as grid search or numerical gradient estimation. However, if function evaluation is expensive, then it is important to minimize the number of samples drawn from the black box function. This is the regime in which BO is particularly effective \citep{frazier2018tutorial}. In BO, we maintain two models: (1) a surrogate model, that is used for approximating the objective function (given a history of prior evaluations). This model produces a posterior mean and an uncertainty estimate at each point in the space. (2) Acquisition model, that helps guide which point should be chosen for evaluation in the next iteration; this model trades-off predicted value and uncertainty (exploration-exploitation). Among many variants of BO, Upper Confidence Bound (UCB)-BO \citep{srinivas2012information} is known to be evaluation-efficient, i.e., no algorithm can achieve better in terms of both simple regret and expected cumulative regret \citep{whitehouse2023sublinear, wang2023regret}. In practice, the UCB acquisition function is often non-concave, so practitioners usually use \textit{parallel gradient-based} UCB-BO \citep{wilson2018maximizing}: start from multiple initial points, take gradient steps to reach several local optima, and evaluate those optima. Parallel gradient-based BO can be especially effective in high-dimensional settings \citep{papenmeier2025understanding}.

\paragraph{Iterative prompt optimization as BO.} The iterative-prompt optimization problem can be summarized as follows: in each iteration, the agent repeatedly selects a prompt, observes a noisy performance score, and decides what to try next. The overarching goal is to reach a high-performing prompt (or candidate solution) in as few evaluations as possible. We note that this problem shares some key features of the black-box optimization where BO typically performs well: the prompt-space is high-dimensional and there is no closed-form prompt-to-score model, exploration must be balanced with exploitation when we choose the next prompt to evaluate, and the goal is to maximize progress per costly evaluation. Building on this insight, we propose to cast iterative prompt optimization as \emph{parallel gradient-based UCB Bayesian optimization in language space}: prompts from prior rounds serve as initializations, gradient-style updates move toward local optima of an implicit UCB acquisition, and in-context learning plays the role of posterior updating.

\paragraph{Challenges.} However, realizing this idea requires overcoming two key challenges. First, in prompt space, gradients are ill-defined because text lacks the inner product structure of a Hilbert space required to measure the directions and magnitudes of change \citep{young1988introduction}. Second, even if we can define a meaningful ``textual gradient,'' we do not have direct access to the UCB acquisition function (or its derivatives). The surrogate posterior over prompt quality that induces the UCB lives implicitly in the LLM, rather than as an explicit Gaussian-process-style model that exposes closed-form posterior means, variances, and their derivatives. Consequently, standard gradient-based BO techniques that differentiate the UCB constructed from a known surrogate cannot be applied naively: we must design a mechanism that elicits UCB-like, optimism-in-the-face-of-uncertainty behavior using only black-box text interactions and critic feedback, without ever computing the UCB function or its gradient in closed form.


\subsection{\textsc{TextBO}: Algorithm, Theory, and Numerical Experiments}

We propose \textsc{TextBO}, a simple evaluation-efficient prompt-optimization based self-improving AI framework that adapts parallel gradient-based UCB-BO to language space by addressing the above challenges.


To address the first challenge, \textsc{TextBO} builds on \textsc{TextGrad} \citep{textgrad}, which adapts gradient descent to discrete text by using an auxiliary LLM \emph{critic} to propose targeted edits; applying these edits induces a local, directional update (a ``textual gradient'') in language space. For example, starting from \emph{``A simple photo of a plant-based burger on a white plate,''} A critic LLM might suggest:
\emph{``Make the burger appear delicious, and show it being enjoyed at a barbecue.''} Incorporating this edit yields:
\emph{``A photorealistic image of a juicy plant-based burger with grill marks, served at a vibrant summer barbecue
with friends.''} Such edits serve as language-native analogues of gradient steps. To address the second challenge, \textsc{TextBO} selects gradient updates via the \textsc{Best-of-N} principle \citep{snell2024scaling}, producing what we call the \textit{\textsc{Best-of-N} gradient}. A single textual gradient step exploits expected improvements implied by past results but does not account for uncertainty. \textsc{Best-of-N} gradient injects exploration into the acquisition by sampling multiple local edits and moving in the direction judged most promising by the critic---an optimism-based rule designed to emulate ascent on an implicit UCB acquisition gradient in language space.

\paragraph{Algorithm.} We now briefly describe our \textsc{TextBO} algorithm. Each iteration in the algorithm consists of three steps. In Step 1 ({\it meta-reflection phase}), we use the critic model to analyze the history of previously evaluated solutions and their scores, and update a meta-reflection that summarizes what has tended to work well or poorly. In Step 2 ({\it \textsc{Best-of-N} textual-gradient phase}), we perform a sequence of \textsc{Best-of-N} gradient updates from the current
prompt. That is, for each gradient step, (i) we propose $N$ textual-gradient edits that transform the current prompt into $N$ candidate prompts, (ii) each candidate prompt is passed to a solution-generation model to produce a corresponding candidate solution, (iii) conditioned on the meta-reflections, the critic model predicts which candidate solution is most likely to perform best, and (iv) we set the prompt associated with this predicted-best solution as the new current prompt. After Step 2, we obtain a refined prompt and its associated solution to evaluate. In Step 3 ({\it evaluation phase}), we evaluate this solution to get a scalar score,  which \textsc{TextBO} then uses to update its history and meta-reflection for subsequent optimization steps. 


 

\paragraph{Theory.} We provide theoretical guarantees on the performance of our approach. Under mild regularity conditions, we prove that the \textsc{Best-of-N} gradient computed from locally sampled textual edits induces (in probability) an ascent direction of the UCB acquisition function with exploration parameter $\beta_N = \Theta(\sqrt{\ln N})$ (Theorem~\ref{thm:main-ucb}). Consequently, \textsc{TextBO} effectively implements parallel, gradient-based UCB Bayesian Optimization in the implicit embedding space induced by the LLM, despite operating purely in language space and without constructing an explicit surrogate or uncertainty model. It therefore inherits the sublinear-regret and evaluation-efficiency guarantees established for UCB-BO, up to constant factors and approximation error due to textual gradients.


\paragraph{Numerical experiments.} We consider two sets of numerical experiments to empirically demonstrate the performance of \textsc{TextBO} -- (1) a advertising setting where the goal is to optimize ad creatives, and (2) a series of agentic AI benchmarks from the recent self-improving AI literature.

In the ad optimization experiments, the goal is to iteratively refine image-generation prompts for ad creatives to maximize responsiveness for a population of customers. We test and measure how well and how fast a method optimizes for (i.e., aligns with) the preference distribution induced by LLM-based simulation of personas in the Digital Twins persona dataset called Twin-2k-500 \citep{toubia2025twin}. We consider eight advertising scenarios and compare the performance of \textsc{TextBO} to two strong baselines -- (1) \textsc{Best-of-N} \citep{snell2024scaling}, a popular and well-performing test-time alignment method, and (2) \textsc{GEPA} \citep{agrawal2025gepa}, a state-of-the-art prompt optimization based self-improving AI algorithm, which has been shown to outperform reinforcement learning fine-tuning algorithms such as GRPO \citep{liu2024deepseek}.

We show that \textsc{TextBO} significantly outperforms both these baselines, across all eight scenarios. Further, in an ablation study, we show that these gains are not simply due to the persona data -- even when we simulate preferences without any persona information, i.e., when the evaluator is a single LLM judge with no auxiliary context -- \textsc{TextBO} outperforms baseline approaches. Our numerical experiments show that \textsc{TextBO} can reach good performance both in settings where there is significant heterogeneity in the preferences of the target population as well as in settings with homogeneous preferences.

Next, we consider a series of agentic AI benchmark experiments in order to assess whether the same \textsc{Best‑of‑N} multi‑step update principle improves evaluation efficienc. Specifically, we consider the experiments used to evaluate \textsc{GEPA} in \citet{agrawal2025gepa} -- HotpotQA, HoVer, and Pupa. We then compare the performance of \textsc{TextBO‑GEPA} (\textsc{GEPA} augmented with Best‑of‑N gradient sampling and multiple gradient steps) to \textsc{GEPA}. Across all experiments, we show that \textsc{TextBO-GEPA} consistently outperforms \textsc{GEPA}. 

Together, these experiments establish TextBO as an effective and general purpose self-improving AI algorithm that can be used for data driven decision-making for both business decisions such as ad creatives, product descriptions, as well as agentic tasks.

\subsection{Contributions}

In sum, our paper makes four main contributions to the literature on AI, optimization, and data-driven decision making. First, from a conceptual perspective, we identify evaluation efficiency as the key objective to optimize in self-improving AI design in societal and business systems. Second, we propose \textsc{TextBO}, a purely text-based self-improving AI framework that combines textual-gradient search with \textsc{Best-of-N} selection to automate and optimize business and marketing decisions (e.g., ad creatives, product descriptions). \textsc{TextBO} is simple and easy to implement: it requires no explicit uncertainty modeling, embeddings, or Gaussian process surrogates. Third, we provide theory showing that \textsc{Best-of-N} over locally sampled textual edits selects, in probability, ascent directions of the UCB acquisition with exploration weight $\beta=\Theta(\sqrt{\ln N})$; thus, we prove that \textsc{TextBO} emulates gradient-based UCB and inherits evaluation-efficiency guarantees. Finally, we demonstrate empirical gains using a wide variety of numerical experiments, including automated ad-alignment across eight scenarios as well as the standard agentic AI benchmarks used in the self-improving AI literature -- throughout,  \textsc{TextBO} improves faster and attains higher final scores compared to strong baselines, and remains effective both in settings with and without contextual information.



\section{Related Literature}
\label{ssec:related_lit}

Our paper relates and contributes to multiple streams of literature, including the literature on AI and LLMs in computer science, the operations literature on optimization, and the marketing literature on ad optimization. We now discuss each one below.

The literature on self-improving AI has often been motivated by an important observation: the progress in AI based only on human data is hitting limits, so AI must create and learn from its own data. In other words, powerful AI ``should have their own stream of experience that progresses, like humans, over a long time-scale''  \citep{silver2025welcome}. That is, we need to design the self-improving AI that iteratively 1) generate data of outcomes, evaluate outcomes with grounded signals, and 2) self-tune the three knobs -- \textit{models}, \textit{tools}, and \textit{prompts} -- in the system to achieve long-term objectives under real-world feedback \citep{wang2025maestro}. Here, tuning models is about altering language models' input-output correspondences \citep{huang-etal-2023-large, lee2023rlaif, acikgoz2025self}; tuning tools is about specifying the tools the AI has access to, including memory tools \citep{hou2025model, ouyang2025reasoningbank, qiu2025alita}; tuning prompts is about engineering prompts in the model's context window, including system messages, task instructions, constraints, schema, and few-shot exemplars \citep{zhou2022large, khattab2023dspy, pryzant2023automatic, yang2024largelanguagemodelsoptimizers,textgrad, agrawal2025gepa,  zhang2025agentic, wang2025maestro}. 

In particular, iterative prompt-optimization based self-improving AI have become increasingly popular. In a recent paper, \citet{agrawal2025gepa} demonstrate their method, \textsc{GEPA} (Genetic-Pareto) is able to outperform state-of-the-art model fine-tuning techniques such as GRPO \citep{liu2024deepseek} under a fixed budget of LLM generation calls across a range of challenging tasks including multi-document QA, constrained instruction following, retrieval-augmented fact verification, and privacy-related objectives. At a high-level, both \textsc{GEPA} and other prompt-based self-improvement \citep{agrawal2025gepa, zhang2025agentic, wang2025maestro} is specified as procedural heuristic search: maintain a pool of prompts, generate variants via LLM-driven mutation/crossover/reflection, score candidates on held-out data or rollouts, retain top performers, and iterate \citep{10.5555/3692070.3692611,guo2024connecting,agrawal2025gepa,better_together}. This mirrors classical heuristic and evolutionary optimization, where performance is assessed as a function of {\it generation efficiency}  rather than {\it evaluation efficiency} \citep{nguyen2016understanding, ding2023quality, hu2024automated, liu2024evolution}. Yet, as discussed earlier, in many real-world deployments, the scarce resource is not candidate generation but evaluation. Thus, in this paper, we focus on developing an iterative prompt-optimization framework that treat evaluation calls as the budget to allocate and optimize.


Our paper also relates to the literature on relating LLMs to Bayesian Optimization (BO). A series of papers focuses on applying LLMs to enhance BO for numerical optimization problems. \cite{liularge} and \cite{cisse2025language} focus on leveraging LLMs' contextual understanding and few-shot learning capabilities to improve BO’s efficiency, enabling better warm-starting and surrogate modeling; \cite {agliettifunbo} uses an LLM to discover new acquisition functions for Bayesian optimization. The LLM writes candidate acquisition function formulas as code, which are evaluated on various optimization problems. \cite{singh2025mechanistic} applies BO to enhance LLM’s internal behavior to boost the model’s zero-shot and few-shot performance. In contrast to these papers, we focus on the prompt optimization problem that arises when formulating LLM-based self-improving AI system design as a Bayesian Optimization problem. In line with this focus, \cite{schneider2025hyperbandbased} and \cite{kong2025metapromptoptimizationllmbasedsequential} formulate prompt optimization among the fixed set of candidate prompts as Bayesian optimization. While \cite{schneider2025hyperbandbased} utilizes ideas from adversarial bandits, \cite{kong2025metapromptoptimizationllmbasedsequential} feeds each prompt text into an embedding model, and the GP surrogate operates on those fixed embedding features to guide selection. However, unlike these papers, we actively {\it generate} new candidate prompts as the process proceeds, rather than pre-specifying a fixed set of candidate prompts. 

Next, our paper relates to the literature on efficient ad evaluation in digital marketing using bandit methods. Here, the goal is to iteratively improving the ads served given a fixed set of candidate ads. \cite{schwartz2017customer} adopts Thompson Sampling with a hierarchical general linear model for a real display advertising campaign for a financial-services firm. \cite{geng2020online} partitions JD.com’s platform users into disjoint sub-populations and applies a contextual Thompson sampling algorithm to maximize the expected advertiser payoff. \cite{aramayo2023multiarmed} developed a contextual Thompson sampling algorithm to dynamically determine the list of house ads to display on the homepage of an electronic retailer to maximize the accumulated click-through rate of the ads. \cite{ba2022advertising} used parametric Thompson Sampling, where a parametric structure links features (media, audience attributes) to conversion likelihood to optimally allocate ad-media and target-audience combinations in high-dimensional settings with low success rates. The main difference between these bandit-based papers and ours is the existence of a pre-defined set of candidate arms. Our paper focuses on creating new solutions by analyzing previous results. On the other hand, bandit methods aim to balance exploration of a pre-defined set of candidate ads optimally. The scope of our problem is also quite different. Our self-improving AI iterates over three steps--generation, evaluation, and analysis--whereas, bandits focus on efficiently executing the evaluation step. Indeed, one could employ a bandit method, such as best-arm-identification, in our evaluation step (instead of using A/B tests). 

Finally, our paper relates to the literature on ad content. Researchers have studied many different features of ad creatives and their effect on consumer response, e.g., creativity \citep{rosengren2020meta}, discount information \citep{biswas_2025}, emotionality of content \citep{macinnis_etal_2002}, and specific text phrases \citep{rutz_etal_2017}. However, a limitation of this literature is that each study is designed to test one specific ad feature in one specific advertising format/context (e.g., TV ads, search ads), and there is not much consensus on how various features and estimated effects should be combined to optimize the ad creative and whether these effects can be ported to other formats and audiences. In this paper, we move away from manually manipulating individual ad features and instead treat the ad-creative optimization problem as a high-dimensional optimization problem in the prompt-space; and our self-improving AI framework can be used to automate and optimize ad-creative generation at scale for any given combination product, ad-format, and audience.



\section{Problem Definition}
\label{sec:prelim}


Let $\Phi$ denote a generative AI model (or more generally, an AI system) that, given a prompt $\pi \in \Pi$, produces a candidate solution $\Phi(\pi)$.\footnote{LLM-based AI systems can also be modified by tuning the LLM's internal parameters using LLM fine-tuning methods such as GRPO \citep{liu2024deepseek}. In this paper, we focus on prompt optimization, which fixes those internal parameters throughout.} This candidate solution $\Phi(\pi)$, in conjunction with some environmental variables $x \in \mathcal{X}$ produces an output $y(\Phi(\pi), x)$, where $y$ is drawn from a distribution over an output space $\mathcal{Y}$. For example, in the ad creative optimization example from $\S$\ref{sec:intro}, $\Phi$ denotes an image generation AI model (e.g., Gemini Flash), $\pi$ is an image generation prompt that reflects what the marketer wants the ad to depict, $\Phi(\pi)$ is the ad creative generated by prompt $\pi$, $x \in \mathcal{X}$ denotes a user on the ad platform and/orother contextual features that affect the ad's effectiveness (e.g., seasonality or time of day), and a realization of $y(\Phi(\pi), x)$ represents that user/context $x$'s reaction (e.g., clicks, conversions, spend) to the ad $\Phi(\pi)$. 

Throughout, we allow for the distribution over $\mathcal{X}$ to depend on the prompt $\pi$. This is important for practical applications where such dependencies are common; for example, in digital advertising platforms, recommendation systems often determine which users are shown the ad $\Phi(\pi)$ based on the ad content, which depends on $\pi$. To capture such dependencies, we denote the input distribution over $\mathcal{X}$ by $\mathcal{D}_{\mathcal{X}}(\pi)$. In a fully randomized test, the distribution over $\mathcal{X}$ is independent of $\pi$, and we can denote the distribution of $\mathcal{X}$ as $\mathcal{D}_{\mathcal{X}}$.

To evaluate the output, we use a scoring rule: a function 
$r: \mathcal{Y} \to [0,1]$ that assigns a scalar effectiveness score to each $y\in\mathcal{Y}$.  
In the digital marketing example, $r$ is a marketer-defined function that maps observed user responses 
(clicks, conversions, spend) to a normalized scalar effectiveness score. Then, our optimization objective is to discover the prompt $\pi^*$ that maximizes the objective function:
\begin{equation}
    J(\pi):=\mathbb{E}[r(y(\Phi(\pi), x))].
\end{equation}
$J$ cannot be optimized by directly computing gradients, because the input distribution $\mathcal{D}_{\mathcal{X}}(\pi)$ can shift with $\pi$; even when $J(\pi)$ is differentiable with respect to $\pi$ and $y$'s randomness is not dependent on $\pi$, we have:
\begin{align}
    \nabla_{\pi} J(\pi)=
\underbrace{\mathbb{E}\big[\nabla_{\pi} r(y(\Phi(\pi),x))\big]}_{\text{model term}}
\;+\;
\underbrace{\mathbb{E}\big[r(y(\Phi(\pi),x)) \,\nabla_{\pi}\log d_{\pi}(x)\big]}_{\text{distribution term}},
\end{align}
where $d_{\pi}$ is the density of $\mathcal{D}_{\mathcal{X}}(\pi)$ \citep{williams1992simple}.  
Since $d_{\pi}$ and $\nabla_{\pi}\log d_{\pi}(x)$ are unobservable, we cannot compute gradients from observed rewards. This motivates the use of an external LLM that can analyze prior data and propose directions for improvement without explicitly computing gradients, leading to self-improving AI algorithms we describe below that iteratively optimize for $J$.\footnote{In fully randomized experiments, the distribution term drops out and only the model term remains.} 

\paragraph{Self-improving AI algorithm:} Given a scoring rule $r$, a \textit{self-improving AI algorithm} $\mathcal{A}$ optimizes $J(\pi):=\mathbb{E}[r(y(\Phi(\pi), x))]$ by iteratively analyzing the previous history, proposing a new prompt, and observing its noisy score. Formally, starting from the initial prompt $\pi_0$, a self-improving AI algorithm proceeds in discrete optimization iterations, indexed by $t = 1,2,\ldots$; at the beginning of iteration $t$, the algorithm has access to the history of past prompts and their observed scores, then selects a new prompt $\pi_t$, deploys the corresponding candidate solution $\Phi(\pi_t)$, and receives a score feedback on its performance. We define the score $s_t$ of the prompt $\pi_t$ as the empirical average reward at $t$, i.e.,
\begin{equation}
    s_t:=\frac{1}{L} \sum_{l=1}^L r\left(y\left(\Phi\left(\pi_t\right),x_t^{(l)}\right)\right), 
    \quad x_t^{(l)} \sim \mathcal{D}_{\mathcal{X}}\left(\pi_t\right),
    \label{eq:emp_score}
\end{equation}
where $L$ denotes the number of input samples $x_t^{(l)}$ we evaluate for each $t$. Note that we compare the performance of self-improving algorithms in terms of the progress of $s_t$ across iterations $t$.

\begin{definition}[Self-improving AI algorithm]

Fix a generative model $\Phi$ and a scoring rule $r$. For each iteration $t \ge 1$, let the history of past evaluations be
    $H_{t-1} \;:=\; \{(\pi_\tau, \Phi(\pi_{\tau}), s_\tau)\}_{\tau=0}^{t-1}$, where $s_\tau$ is the empirically evaluated score of $\pi_\tau$, defined as
\begin{align}
       s_t = \frac{1}{L} \sum_{l=1}^L r\bigl(y(\Phi(\pi_t), x_t^{(l)})\bigr), 
    \quad x_t^{(l)} \sim \mathcal{D}_{\mathcal{X}}(\pi_t).
\end{align}
Then, a self-improving AI algorithm is defined as function $\mathcal{A}$ that, at each iteration $t$, maps the history $H_{t-1}$ to a new prompt $\pi_t$, i.e., 
\[
    \pi_t = \mathcal{A}_t(H_{t-1}).
\]
\end{definition}

Operationally, given access to the history $H_{t-1}$, a self-improving AI algorithm $\mathcal{A}$ goes through following three steps each iteration $t$:
\vspace{-0.2cm}
\begin{enumerate}[leftmargin=0.5cm, noitemsep]
    \item \textit{Analysis:} Analyze the previous history $H_{t-1}$.
    \item \textit{Generation:} Propose the next prompts $\pi_t$ based on the analysis.
    \item \textit{Evaluation:} The system outcome $\Phi(\pi_t)$ is deployed, and its performance is measured, yielding a score $s_t$. The new observation is added to the history: $H_t = H_{t-1} \cup \{(\pi_t, \Phi\left(\pi_t\right), s_t)\}$.
\end{enumerate}
We will be more precise about how Steps 1 and 2 (analysis and generation) are implemented in our self-improving AI framework in $\S$\ref{sec:BoNAndTheory} and $\S$\ref{sec:TBoN}.

\paragraph{Optimization objective:} In this paper, we formulate the goal of self-improving AI algorithms $\mathcal{A}$ as improving prompts   $\left(\pi_t\right)_{t=1}^T$ over the evaluation budget $T$ to efficiently optimize the objective $J(\pi)$. Specifically, we discuss evaluation efficiency using two standard regret criteria. Let $\pi^{\star} \in \arg \max _{\pi \in \Pi} J(\pi)$. Then,
\squishlist
    \item  The (expected) cumulative regret is $\mathbb{E}\left[\sum_{t=1}^T\left(J\left(\pi^{\star}\right)-J\left(\pi_t\right)\right)\right]$.
    \item The (expected) simple regret is $\mathbb{E}\left[J\left(\pi^{\star}\right)-\max _{1 \leq t \leq T} J\left(\pi_t\right)\right]$.
\squishend
In the next section, we show how a self-improving AI algorithm can be viewed as UCB-based Bayesian Optimization, which is optimal under both regret criteria \citep{whitehouse2023sublinear, wang2023regret}. 

\section{Our Approach: Recasting Iterative Prompt Optimization as Bayesian Optimization in Language Space}
\label{sec:bo_recast}

As discussed in $\S$\ref{sec:intro}, our key idea is to model prompt-based self-improving AI as a Bayesian Optimization (BO) problem in language space. At a high level, we want a self-improving AI algorithm that repeatedly chooses a prompt, observes a noisy scalar performance signal (e.g., ad effectiveness), and decides which prompt to try next under a limited evaluation budget. We note that our optimization problem has no closed-form expression mapping the prompt to a performance score and a key constraint is that evaluating the function is costly. This is precisely the regime where BO has been shown to be effective, i.e., black-box objective functions where function evaluation is costly \citep{frazier2018tutorial, srinivas2012information, papenmeier2025understanding}. 

\paragraph{Canonical UCB Bayesian Optimization.}

To fix ideas, we start with a quick overview of BO. Classical BO considers a black-box function $f:\mathcal{Z}\to\mathbb{R}$ defined over a continuous input space $\mathcal{Z}$. The algorithm maintains a surrogate model---typically a Gaussian Process or related probabilistic regressor---that, given a history of evaluations $H_{t-1} = \{(z_\tau, y_\tau)\}_{\tau=1}^{t-1}$, produces for each $z\in\mathcal{Z}$ a posterior mean $\mu_{t-1}(z)$ and an uncertainty estimate $\sigma_{t-1}(z)$. Rather than maximizing $\mu_{t-1}(z)$ directly, BO selects the next query by maximizing an \emph{acquisition function} that trades off exploitation (high predicted value) and exploration (high uncertainty). A widely used acquisition is the Upper Confidence Bound (UCB)
\begin{equation}
    A_{\beta_t}(z)
    \;=\;
    \mu_{t-1}(z) \;+\; \beta_t \,\sigma_{t-1}(z),
\end{equation}
where $\beta_t > 0$ controls the strength of exploration \citep{srinivas2012information}. Figure~\ref{fig:UCBdetail} illustrates how $A_{\beta_t}$ combines the surrogate mean and uncertainty. Intuitively, $A_{\beta_t}(z)$ implements an ``optimism in the face of uncertainty'' rule: points with large uncertainty receive a positive bonus even if their current mean is modest.

\begin{figure}[htp!]
    \centering
    \includegraphics[width=0.65\linewidth]{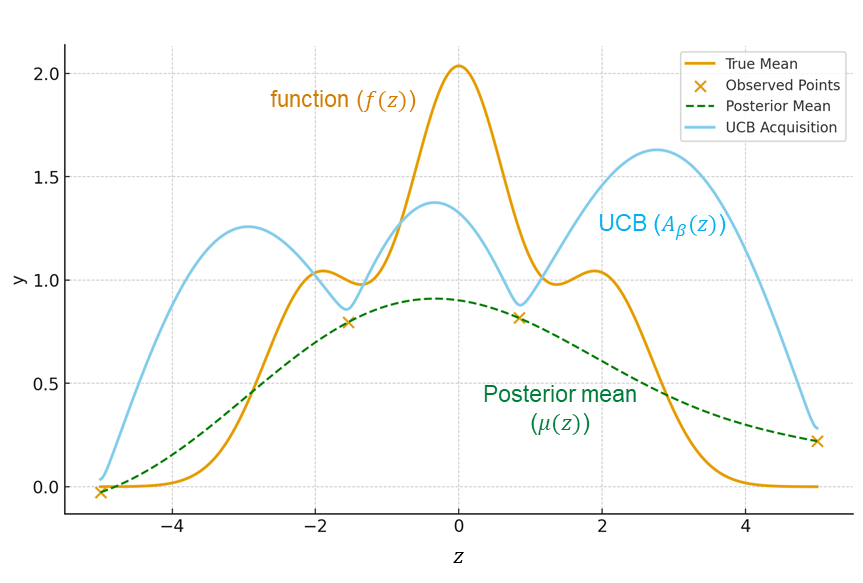}
    \caption{Illustration of the upper confidence bound (UCB) acquisition function in Bayesian optimization. The unknown objective $f(z)$ (orange) is evaluated at a set of observed points ($\times$). A surrogate model yields the posterior mean $\mu(z)$ (green dashed) and $\sigma(z)$. The UCB acquisition $A_\beta(z)=\mu(z)+\beta \sigma(z)$ (blue) trades off exploitation and exploration by favoring locations with high predicted value and high uncertainty.}
    \label{fig:UCBdetail}
\end{figure}

\begin{algorithm}[ht!]
\caption{Canonical UCB Bayesian Optimization}
\label{alg:ucb-bo}
\KwIn{Evaluation budget $T$, exploration parameter schedule $\beta$, initial history $H_0 = \emptyset$}
\KwOut{Sequence of evaluation points $\{z_t\}_{t=1}^T$ and observations $\{y_t\}_{t=1}^T$}
\For{$t = 1,\dots,T$}{
  Update surrogate model $p(f \mid H_{t-1})$ on past data\;
  Construct the UCB acquisition
  $$
      A_{\beta}(z) \;=\; \mu_{t-1}(z) \;+\; \beta\,\sigma_{t-1}(z).
  $$
  Select the next evaluation point
  \[
      z_t \;\in\; \operatorname{argmax}_{z} A_{\beta}(z).
  \]
  Evaluate $f$ at $z_t$ to obtain a noisy observation
  \[
      y_t = f(z_t) + \text{noise}.
  \]
  Augment the data: $H_t \leftarrow H_{t-1} \cup \{(z_t,y_t)\}$\;
}
\end{algorithm}
Algorithm~\ref{alg:ucb-bo} summarizes the resulting UCB-BO loop. At each iteration $t$, the surrogate is updated on $H_{t-1}$, the acquisition $A_{\beta_t}$ is constructed, the next point $z_t$ is chosen by maximizing $A_{\beta_t}$, the black-box function $f$ is evaluated at $z_t$ to obtain $y_t$, and the history is augmented. Under mild conditions, UCB-BO achieves sublinear regret and enjoys strong evaluation-efficiency guarantees: no algorithm can perform uniformly better in terms of regret scaling \citep{srinivas2012information, whitehouse2023sublinear, wang2023regret}.

In many practical settings, the global maximizer of $A_{\beta_t}$ is not available in closed form, and the acquisition is highly non-convex. A standard remedy is \emph{parallel gradient-based BO} \citep{wilson2018maximizing, papenmeier2025understanding}: one starts from multiple initial points, performs gradient ascent on $A_{\beta_t}$ from each initialization to find local maxima, and then evaluates $f$ at these local optima in parallel. This procedure has shown surprising empirical effectiveness in high-dimensional numerical optimization \citep{papenmeier2025understanding}.

\paragraph{Mapping self-improving AI to UCB-BO.}
We now conceptually align our iterative prompt-optimization problem with the BO template above. The correspondence is:
\squishlist
    \item The BO input space $\mathcal{Z}$ corresponds to the prompt space $\Pi$.
    \item The black-box function $f(z)$ corresponds to the system-level objective $
        f(\pi) \equiv J(\pi) := \mathbb{E}[r(y(\Phi(\pi), x))]$, where $x\sim\mathcal{D}_{\mathcal{X}}(\pi)$ and $r$ is the scoring rule defined in \S~\ref{sec:prelim}.
    \item The evaluation of $f(\pi)$ is the costly step: it requires deploying $\Phi(\pi)$, interacting with the environment (e.g., ad platform users), and computing $s_t$ as an empirical estimate of $J(\pi)$.
    \item The cheap computation is everything the LLM-based system can do \emph{without} new evaluations: analyzing past history, proposing textual edits, or internally comparing candidate ads.
\squishend

If we had an explicit surrogate model over prompts, providing $\mu_t(\pi)$ and $\sigma_t(\pi)$ for each $\pi\in\Pi$, we could in principle construct a UCB acquisition
\begin{equation}
    A_{\beta_t}(\pi)
    \;=\;
    \mu_t(\pi) \;+\; \beta_t \,\sigma_t(\pi),
\end{equation}
and then run parallel gradient-based BO over $\Pi$: start from multiple initial prompts, ascend the gradient $\nabla_\pi A_{\beta_t}(\pi)$ to obtain local optima, and evaluate $J(\pi)$ only at those local optima. This would give us a self-improving AI algorithm that is, by construction, evaluation-efficient. However, directly instantiating this plan in language space runs into two central challenges (originally discussed in $\S$\ref{sec:intro}).

\paragraph{Challenge 1: No native gradients in prompt space.}
Prompts are discrete sequences of tokens, not elements of a Euclidean vector space. As such, prompt space lacks the inner-product structure of a Hilbert space that underlies standard notions of directions, norms, and gradients \citep{young1988introduction}. Expressions such as $\nabla_\pi A_{\beta_t}(\pi)$ are therefore not well-defined in the usual calculus sense. 

A seemingly natural workaround is to embed prompts into a continuous space, take gradients there, and then ``project back'' to text: 
\begin{equation}
\pi \;\xrightarrow{\text{embed}}\; e=\mathsf{E}(\pi)\in\mathbb{R}^d
\;\xrightarrow{\text{gradient step}}\; e' 
\;\xrightarrow{\text{decode}}\; \pi'=\mathsf{E}^{-1}(e').
\end{equation}
where $\mathsf{E}$ is the prompt embedding. 
Such a strategy is indeed theoretically plausible thanks to the invertibility of prompt embedding  \citep{nikolaou2025language}. However, in practice, it encounters the well-known soft-to-discrete projection problem: learning the inverse prompt embedding ($\mathsf{E}^{-1}$), i.e., learning to decode an embedding into a prompt, is highly challenging \citep{cui2025automatic}. As a result, we cannot reliably implement gradient ascent in prompt space by differentiating through embeddings.

\paragraph{Challenge 2: No explicit UCB or uncertainty model.}
Even if we had a notion of gradient in language space, we would still face a second obstacle: the UCB acquisition function itself is not explicitly available. In standard BO, the surrogate model exposes both a posterior mean $\mu_t(\cdot)$ and an uncertainty estimate $\sigma_t(\cdot)$, from which we build $A_{\beta_t}(\cdot)$ and compute its gradient. 

In our setting, by contrast, there is no explicit Gaussian process or parametric surrogate over prompts. The relevant beliefs about prompt quality and uncertainty live \emph{implicitly} inside the LLM and its critic: they are encoded in how the model generalizes from past evaluations, not in an externally accessible $\mu_t$ or $\sigma_t$. Consequently, we cannot simply write down $A_{\beta_t}(\pi)$, evaluate it at arbitrary prompts, or compute $\nabla_\pi A_{\beta_t}(\pi)$ in closed form. Naive proxies for uncertainty (e.g., self-consistency variance, embedding distances) tend to be poorly calibrated or brittle, and training a separate Gaussian Process over high-dimensional text embeddings is often expensive, data-hungry, and empirically unstable.

\paragraph{Textual gradients as a partial solution.}
Recent work on \emph{textual gradients} offers a way to tackle the first challenge. The \textsc{TextGrad} framework \citep{textgrad} shows that one can approximate gradient ascent on an implicit objective by asking an auxiliary LLM \emph{critic} to propose local natural-language edits that are predicted to improve a scalar score. Applying these edits to the current prompt yields a sequence of locally improving prompts, which empirically and theoretically behaves as gradient ascent on the posterior mean $\mu(\pi)$ in an implicit embedding space thanks to bi-Lipschitzness of LLM's prompt embedding \citep{nikolaou2025language, tang2025understanding}. 

However, textual gradients as in \citet{textgrad} are inherently \emph{exploitative}: they aim to climb the posterior mean $\mu(\pi)$, ignoring epistemic uncertainty $\sigma(\pi)$. They, therefore, correspond to gradient ascent on $\mu(\pi)$, not on the UCB acquisition $A_{\beta_t}(\pi) = \mu(\pi) + \beta_t \sigma(\pi)$. To recover the evaluation-efficiency guarantees of UCB-BO, we need an update rule whose effective search direction behaves like $\nabla_\pi A_{\beta_t}(\pi)$, even though we never observe $\sigma(\pi)$ or $\nabla_\pi \sigma(\pi)$ explicitly.

\paragraph{Our approach.}
The central technical idea of this paper is to address Challenge~1 and Challenge~2 \emph{jointly} by combining textual gradients with a \textsc{Best-of-N} selection mechanism. Rather than differentiating an explicit acquisition function, we (i) use a critic LLM to generate multiple stochastic textual edits around the current prompt, (ii) apply these edits to obtain nearby candidate prompts and system outcomes, and (iii) ask the critic, conditioned on past history, to pick the most promising candidate in a \textsc{Best-of-N} fashion. 


In $\S$\ref{sec:BoNAndTheory}, we formalize this procedure as a \emph{\textsc{Best-of-N} gradient} and show that, under mild regularity conditions, the direction selected by this mechanism converges (in probability) to the ascent direction of a UCB-style acquisition function in the LLM’s implicit embedding space, with exploration parameter $\beta_N = \Theta(\sqrt{\ln N})$ (Theorem~\ref{thm:main-ucb}). This theoretical result gives us a principled way to emulate parallel gradient-based UCB-BO entirely in language space using only black-box text interactions. We then leverage this equivalence to design a concrete self-improving AI algorithm in $\S$\ref{sec:TBoN}, \textsc{TextBO}, that is evaluation-efficient, simple to implement, and readily deployable in real-world applications such as digital advertising.

\section{\textsc{Best-of-N} gradient and Theory}
\label{sec:BoNAndTheory}

As discussed in $\S$\ref{sec:bo_recast}, our aim is to implement UCB-style Bayesian Optimization in prompt space using only black-box access to an LLM critic. The theoretical challenge is to show that the combination of (i) \emph{textual gradients} and (ii) \emph{\textsc{Best-of-N}} selection actually induces gradient ascent on a UCB acquisition function in the LLM’s implicit embedding space. This section formalizes that idea and proves our main result.

We proceed in three steps. In $\S$\ref{ssec:critic-llm}, we formalize the role of the 
\emph{critic LLM}, which provides the primitive operations—pairwise judgments, textual gradients, 
and meta-reflections—that enable search in language space. 
$\S$\ref{sec:BoNgrad} then introduces the \emph{\textsc{Best-of-N} gradient}, our key mechanism for 
turning these critic capabilities into a stochastic, locally improving update rule. 
Finally, in $\S$\ref{sec:theory}, we prove that \textsc{Best-of-N} gradients asymptotically emulate 
gradient ascent on a UCB acquisition function in the LLM’s implicit embedding space. 
This establishes the theoretical justification for the \textsc{TextBO} framework developed in the next section.

\subsection{Critic LLM}
\label{ssec:critic-llm}

The use of a separate LLM as a \emph{critic} is now a standard pattern in the prompt-optimization and
self-improving AI literature: critic models are routinely used to provide structured feedback, propose improvements, or act as reward models and judges that guide search over prompts and outputs
\citep{zhou2022large, textgrad, agrawal2025gepa, zhang2025agentic, wang2025maestro, shi2024judging}.
We therefore treat the critic LLM as a natural and practically realistic component that supplies the
basic operations needed to implement Bayesian-optimization-style updates in language space.

Recall that we denote a prompt as $\pi$ and the corresponding candidate solution (e.g., an ad creative) as $\Phi(\pi)$, the optimization history up to iteration $t$ as $H_t = \{(\pi_\tau, \Phi(\pi_{\tau}), s_\tau)\}_{\tau \le t}$, where $s_\tau$ is the evaluated score of $\pi_\tau$. Then, let $M_{\text{critic}}$ denote the critic LLM, which can perform the following four operations:
\begin{enumerate}[leftmargin=0.5cm]
    \item {\it Pairwise Judge Operator.} Given two prompts $\pi^{(1)}$ and $\pi^{(2)}$, and their corresponding outcomes $\Phi(\pi^{(1)})$ and $\Phi(\pi^{(2)})$, predict which is better given reflection $R$:
    \begin{equation}
       \textsc{PairwiseJudge}(\Phi(\pi^{(1)}), \Phi(\pi^{(2)}); R, M_{\text{critic}}) \in \{\Phi(\pi^{(1)}), \Phi(\pi^{(2)})\}.
    \end{equation}

    \item {\it Meta-reflection Operator.} Compress the entire history into a small set of natural-language rules or guidelines:
    \begin{equation}
        R_t = \textsc{MetaReflect}(H_t; M_{\text{critic}}).
    \end{equation}

    \item {\it Textual Gradient Operator.} Stochastically generate a local edit $\delta$ intended to improve $\pi$ under reflection $R$:
    \begin{equation}
        \delta \sim \nabla_{\text{text}}(\pi; R, M_{\text{critic}}).
    \end{equation}

    \item {\it Apply Operator.} Apply an edit $\delta$ to obtain a new prompt:
    \begin{equation}
        \pi' = \operatorname{Apply}(\pi, \delta; M_{\text{critic}}).
    \end{equation}
\end{enumerate}

These four operators map directly onto well-documented LLM capabilities in the prompt-optimization and evaluative-feedback literature. The \textsc{PairwiseJudge} operator aligns with the growing use of LLMs as reliable comparison judges for creative and reasoning tasks \citep{shi2024judging, calderon2025alternative}. The \textsc{MetaReflect} operator captures reflection-based prompt-improvement methods that distill past outcomes into actionable guidelines (e.g., in the ad optimization case, it could suggest that “social settings outperform isolated product shots") \citep{jiang2024many, ferber2024context, wang2025maestro}. The \textsc{TextualGradient} operator corresponds to LLM-driven proposals for targeted, locally improving edits (e.g., “make the lighting warmer and emphasize people enjoying the product”), as formalized in \citet{textgrad} and applied across domains such as scientific reasoning and code refinement \citep{li2025test}. Finally, the \textsc{Apply} operator reflects standard LLM-based prompt editing widely used in prompt engineering \citep{zhou2022large, khattab2023dspy}. Together, these capabilities make an LLM critic a natural, practically grounded component for enabling BO-style search over prompt space. 


\subsection{\textsc{Best-of-N} gradient}
\label{sec:BoNgrad}

The notion of a \textsc{Best-of-N} gradient builds on the widely used \textsc{Best-of-N} strategy for test-time
alignment, where an LLM trades additional inference-time compute for higher-quality outputs by
sampling multiple candidates and then selecting the one that appears most aligned with a downstream
objective \citep{snell2024scaling, beirami2024theoretical}. In standard \textsc{Best-of-N} prompting,
an LLM generates $N$ candidate prompts, a generation model produces $N$ candidate outcomes, and a
reward model selects the highest-scoring one. While effective for choosing a single output, our goal
here is to improve the \emph{prompt} itself. To accomplish this, we introduce the concept of \emph{\textsc{Best-of-N}
gradient}. Rather than sampling $N$ full prompts, we sample $N$ distinct \emph{textual gradients}, apply them to
the current prompt, and then select the candidate whose predicted outcome is judged most promising.
This selection defines a localized “direction of improvement” in prompt space, and can be viewed as an {\it acquisition function}.

Formally, given a current prompt $\pi$ and reflection $R$, a single \textsc{Best-of-N} gradient step proceeds
as follows:
\begin{enumerate}[leftmargin=0.5cm,itemsep=2pt]
  \item[1)] \textit{Textual gradient generation.}  
  Use $M_{\text{critic}}$ to generate $N$ independent textual edits:
  \begin{equation}
      \delta^{(i)} \sim \nabla_{\text{text}}(\pi; R, M_{\text{critic}}), \qquad i=1,\dots,N.
  \end{equation}

  \item[2)] \textit{Candidate prompt generation.}  
  Apply each edit to produce $N$ candidate prompts:
  \begin{equation}
      \pi^{(i)} = \operatorname{Apply}(\pi, \delta^{(i)}; M_{\text{critic}}).
  \end{equation}

  \item[3)] \textit{Candidate system outcome generation.}  
  Produce the associated system outcomes:
  \begin{equation}
     \Phi(\pi^{(i)}), \qquad i=1,\dots,N.
  \end{equation}

  \item[4)] \textit{\textsc{Best-of-N} selection.}  
  Identify the most promising candidate using the critic:
  \begin{equation}
      i^\star = \textsc{Best-of-N}\big(\{(\pi^{(i)}, \Phi(\pi^{(i)})\}_{i=1}^N; R, M_{\text{critic}}\big).
  \end{equation}

For the \textsc{Best-of-N} selection, we employ a pairwise tournament \citep{liu2025pairjudgermperformbestofn}. 
The $N$ candidates are first randomly permuted into tournament buckets. Each match compares two 
outcomes using the \textsc{PairwiseJudge} operator, with the comparison repeated $K$ times and the 
order of presentation swapped to mitigate positional bias \citep{shi2024judging} (see the prompt 
template in Figure~\ref{fig:tornament_prompt} in Web Appendix~\ref{appssec:critic_llm_prompts} for the ad-optimization case). Ties are broken at 
random. Winners advance round by round until a single final winner remains; its associated prompt and outcome are then carried forward to the next gradient step.

\end{enumerate}

The resulting prompt $\pi^{(i^\star)}$ defines the update made by the \textsc{Best-of-N} gradient for a given step. Intuitively, each textual edit represents a noisy proposal for a local improvement direction; \textsc{Best-of-N} filters these proposals through an optimism-based rule, selecting the direction most likely to improve performance under the critic LLM’s current understanding.

\subsection{Theory: \textsc{Best-of-N} gradient's Asymptotic Equivalence to UCB gradient}
\label{sec:theory}

We now prove that the \textsc{Best-of-N} over textual gradients selects, in probability, the ascent direction of the UCB acquisition with exploration weight $\beta$ scaling as $\Theta(\sqrt{\ln N})$ (Theorem \ref{thm:main-ucb}). What allows us to do rigorous theoretical analysis is that LLM prompt embeddings are invertible and bi-Lipschitz: LLM's prompt embeddings are known to be injective, and therefore invertible \citep{nikolaou2025language}; LLM embeddings are Lipschitz \citep{tang2025understanding}, implying they are bi-Lipschitz. Hence, small textual edits can be represented as small changes in the embedding space and vice versa. Consequently, small textual edits correspond to small perturbations in the embedding space, allowing us to analyze \textsc{Best-of-N} gradient steps in $\mathbb{R}^d$ and then translate the theoretical guarantees back to the prompt space.

\subsubsection{Intuitive Explanation} 

Denote $e$ to be an embedding of a prompt. Pick a small radius $\varepsilon>0$ around $e$ and sample $N$ nearby edits $e_i=e+\varepsilon u_i$ with $u_i$ on the unit sphere. Next, denote the critic LLM's ($M_{critic}$ from $\S$\ref{ssec:critic-llm}) internal scoring of each edit, $Y_i$, as follows: 
\begin{equation}
Y_i=\mu(e_i)+\xi_i \, \sigma(e_i),
\end{equation}
where $\mu$ is the (implicit) posterior mean, $\sigma$ is the (implicit) epistemic uncertainty, and $\xi_i$ is an i.i.d. sub-Gaussian noise variable  independent of $u_i$. We note that $Y_i$ is in line with the BO scoring function from $\S$\ref{sec:bo_recast}.

Let $M_N = \max_i \xi_i$ and let $q_N = F^{-1}(1-1/N)$ be the $(1-1/N)$-quantile of the noise
distribution. For Gaussian noise, classical extreme-value theory implies that $M_N$ concentrates
around $q_N$ and that $q_N = \Theta(\sqrt{\ln N})$. In other words, as $N$ grows large, the
largest noise term among $\{\xi_i\}_{i=1}^N$ behaves like $q_N$ up to lower-order fluctuations.
Consequently, the index $i^\star$ selected by \textsc{Best-of-N} approximately maximizes
\begin{equation}
\mu(e_i) + q_N\,\sigma(e_i).
\end{equation}
Consequently, for large $N$, selecting the \textsc{Best-of-N} index $i^{\star}$ is asymptotically equivalent to selecting the maximizer of UCB-style scores with
exploration weight $q_N$.

Now consider a small sphere around $e$ in embedding space and the UCB acquisition
$A_\beta(e') = \mu(e') + \beta\,\sigma(e')$. On this sphere, there is a single dominant ascent
direction for $A_\beta$, which we denote by $v_\beta$ (the direction in which $A_\beta$ increases
the fastest to first order). Draw a narrow cone around $v_\beta$. For sufficiently large $N$, at
least one of the sampled directions $u_i$ falls inside this cone with high probability, and the
\textsc{Best-of-N} winner $e_{i^\star}$ comes from this cone. Thus, among the locally sampled perturbations $e_i=e+\varepsilon u_i$, the \textsc{Best-of-N} rule selects a candidate whose direction is (with high probability, for large enough $N$) very close to the ascent direction of the UCB acquisition function with $\beta_N:=q_N$.

\subsubsection{Rigorous Derivation}

Let $\Pi$ be the set of syntactically valid prompts. 
Suppose that there exists a text embedding $\mathsf{E}:\Pi\to\mathbb{R}^{d}$, where
$
e \;:=\; \mathsf{E}(\pi)\in\mathbb{R}^{d}
$ is the embedding of a prompt $\pi\in\Pi$. We first state the assumptions motivated by how language model embeddings behave in practice, and then state the main theorem, Theorem \ref{thm:main-ucb}. 

At each gradient step of \textsc{TextBO}, we sample edits that correspond to $u_1,\ldots,u_N \stackrel{\text{i.i.d.}}{\sim} \rho$ on the embedding space $\subset\mathbb{R}^d$, where $\rho$ is assumed to have a density bounded below by a positive constant with respect to surface measure. Assumption \ref{ass:A0} implies that small textual edits induce near‑linear moves in $\mathbb{R}^d$, consistent with the recent findings that prompt embeddings preserve phrase similarity under small changes \citep{reimers2019sentence, he2025position, tang2025understanding, nikolaou2025language}. 

\begin{assumption}[Embedding map and edit realization]
\label{ass:A0}
For the current prompt $\pi$ with $e=\mathsf{E}(\pi)$, constants $C<\infty$ and $\varepsilon_0>0$ such that for every unit vector $u\in\mathbb{R}^{d}$ and every $\varepsilon\in(0,\varepsilon_0]$ the edit operator $\mathrm{Apply}$ realizes a first–order move:
\begin{align}
\mathsf{E}\!\big(\mathrm{Apply}(\pi,\varepsilon u)\big)
= e + \varepsilon u + r(\varepsilon,u),
\qquad
\sup_{\|u\|=1}\|r(\varepsilon,u)\| \le C\,\varepsilon^2 .
\end{align}
\end{assumption}


Next, Assumption \ref{ass:A1} is a standard regularity condition: function $f:\mathbb{R}^d\to\mathbb{R}$ is of $C^{1,1}$ if it is continuously differentiable and its gradient is Lipschitz:
$\exists L<\infty$ such that $\|\nabla f(y)-\nabla f(z)\|\le L\|y-z\|$ for all $y,z$. 

\begin{assumption}[$\mu,\sigma$ are of $C^{1,1}$]
\label{ass:A1}
Let $\mu(e)$ and $\sigma(e)$ denote the mean and standard deviation of the critic LLM's internal score of embedding $e$. We assume $\mu, \sigma$ are $C^{1,1}$ locally.
We use $g=\nabla\mu(e)$ and $h=\nabla\sigma(e)$.
\end{assumption}

Assumption \ref{ass:A3} specifies a heteroskedastic sub-Gaussian model for candidate scores (with i.i.d.\ mean-zero, unit-variance, unbounded-support noise independent of the sampled directions) and an ideal \textsc{Best-of-N} choice oracle.

\begin{assumption}[Noise model and choice oracle]
\label{ass:A3}
Let $\{\xi_i\}_{i=1}^N$ be i.i.d.\ mean-zero, unit-variance, Gaussian random variables with unbounded support, independent of $\{u_i\}$.
For candidates $e_i := x + \varepsilon u_i$, the critic LLM observes its internal scoring
\begin{equation}
\label{eq:Y-def}
Y_i \;=\; \mu(e_i)\;+\;\sigma(e_i)\,\xi_i,
\end{equation}
and the \textsc{Best-of-N} selection oracle (often built using pairwise tournaments in practice; see $\S$\ref{sec:BoNgrad}) returns corresponding $i^\star\in\arg\max_{i\in[N]} Y_i$.
\end{assumption}

\begin{assumption}[Small-step regime and limit order]
\label{ass:A4}
We take limits with $N\to\infty$ first, then $\varepsilon\downarrow 0$.
\end{assumption}

\begin{lemma}[Gaussian maxima and spacing \citep{vershynin2018high}]
\label{lem:maxima}
Let $\xi_1,\dots,\xi_N$ be i.i.d.\ mean-zero, unit-variance Gaussian, $M_N=\max_{i\le N}\xi_i$, $S_N$ the second largest, and $q_N:=F^{-1}(1-1/N)$ the $(1-1/N)$-quantile of $\xi_1$. Then $q_N=\Theta(\sqrt{\ln N})$, $M_N-q_N\to 0$ in probability, and $M_N-S_N=O_p(1/q_N)$.
\end{lemma}

\begin{theorem}[\textsc{Best-of-N} asymptotically induces a UCB gradient direction]
\label{thm:main-ucb}
Under Assumptions \ref{ass:A0}-\ref{ass:A4}, define $\beta_N:=q_N$ from Lemma~\ref{lem:maxima}. Let $i^\star\in\arg\max_i Y_i$ and set $\widehat u_N:=u_{i^\star}$. Define the step
\begin{align}
    \Delta e^{(N)} \;:=\; \mathsf{E}\!\big(\mathrm{Apply}(\pi,\varepsilon \widehat u_N)\big) - e
\;=\; \varepsilon \widehat u_N + r(\varepsilon,\widehat u_N).
\end{align}
Then, with the limit order $N\to\infty$ first and $\varepsilon\downarrow 0$ second,
\begin{align}
    \frac{\Delta e^{(N)}}{\|\Delta e^{(N)}\|}
\;=\;\widehat u_N
\;\xrightarrow{p}\;
\frac{\nabla A_{\beta_N}(e)}{\|\nabla A_{\beta_N}(e)\|}
\;=\; \frac{g+\beta_N h}{\|g+\beta_N h\|}.
\end{align}
Moreover,
\begin{align}
    A_{\beta_N}(e+\Delta e^{(N)}) \;\ge\; A_{\beta_N}(e) \;+\; \varepsilon\,\|\nabla A_{\beta_N}(e)\| \;-\; o_p(\varepsilon).
\end{align}
\end{theorem}
Thus, each \textsc{Best-of-N} gradient step asymptotically performs a first-order improvement on an implicit UCB
acquisition function, with exploration strength $\beta_N$ controlled entirely by $N$. The full proof is given in Web Appendix~\ref{sec:proof}.

Theorem~\ref{thm:main-ucb} provides the central conceptual and technical contribution of this paper: it shows that a procedure that looks like a purely heuristic rule in language space \emph{implicitly implements} the same optimism-driven search direction as UCB BO in the model's embedding space. This is novel and important because it removes the main obstacle to bringing UCB BO's evaluation-efficiency guarantees into prompt optimization: we never construct an explicit surrogate model, never compute $\sigma(\cdot)$, and never differentiate a closed-form acquisition function, yet the effective update direction converges to the UCB gradient direction.
Moreover, the result establishes a new use case for the ubiquitous \textsc{Best-of-N} principle as an \emph{interpretable exploration mechanism}, where $N$ becomes a principled ``exploration knob'' for the exploration weight $\beta_N=\Theta(\sqrt{\ln N})$.
In short, Theorem~\ref{thm:main-ucb} is what turns \textsc{Best-of-N} gradient from an empirical recipe into a principled BO-style optimizer: it justifies viewing our language-space self-improvement loop as parallel gradient-based UCB-BO and is the reason \textsc{TextBO} can inherit BO's evaluation-efficiency benefits.


\section{\textsc{TextBO} Algorithm}
\label{sec:TBoN}

We now describe the self-improving AI algorithm proposed in this paper: \textsc{TextBO}. \textsc{TextBO} is designed to implement parallel gradient-based UCB-BO in the space of natural language prompts. Concretely, \textsc{TextBO} adopts the parallel gradient-based UCB-BO scheme that reuses the final gradient point from one iteration as the initialization for the next. Thus, each initial prompt gives rise to a \emph{trajectory}: a sequence of candidate prompts and outcomes generated by repeatedly taking gradient-ascent steps and then evaluating the resulting system outcome.

The overview of the \textsc{TextBO} algorithm is shown in Figure \ref{fig:tbon-summary}. 
For each iteration $t\in[1, T]$ and each trajectory $j\in[1, J]$, we evaluate one prompt; to determine the one prompt to evaluate, we take $G$ steps of \textsc{Best-of-N} gradients beforehand, without any evaluation. We present the pseudocode of the \textsc{TextBO} algorithm (Algorithm \ref{algo:tbon}) and provide the line-by-line description of the details of the procedure below.

\begin{figure}[ht!]
\centering
\scalebox{0.85}{%
\begin{tikzpicture}[
  font=\small,
  x=1cm,y=1cm,
  box/.style={draw, rounded corners=2pt, minimum width=2.4cm, minimum height=0.9cm, align=center},
  thinbox/.style={draw, rounded corners=2pt, minimum width=2.0cm, minimum height=0.8cm, align=center},
  cand/.style={draw, rounded corners=2pt, minimum width=2.2cm, minimum height=0.8cm, align=center},
  arrow/.style={-Latex, line width=0.7pt},
  faded/.style={opacity=0.45},
  sel/.style={line width=0.9pt, draw=black},
  call/.style={draw, dashed, rounded corners=3pt}
]

\node[box, faded] (it1) at (0, 4.2) {Iteration 1};
\node at (0, 3.25) {$\vdots$};
\node[box, faded] (itprev) at (0, 2.3) {Iteration $t\!-\!1$};
\node[box, sel, fill=white] (itcur) at (0, 1.2) {Iteration $t$};
\node[box, faded] (itnext) at (0, 0.0) {Iteration $t\!+\!1$};
\node at (0,-0.95) {$\vdots$};
\node[box, faded] (itT) at (0, -1.9) {Iteration $T$};

\node[above=2pt of it1] {\textbf{Optimization loop: $t=1\ldots T$}};

\node[fit=(it1)(itT), inner sep=6pt] (iterCol) {};
\draw[decorate, decoration={brace, amplitude=10pt, mirror}]
  ($(iterCol.north east)+(0.55,0)$) -- ($(iterCol.south east)+(0.55,0)$)
  node[midway, xshift=1.3cm] {};

\coordinate (col2) at (4.0,0);

\node[thinbox, faded] (j1)   at ($(col2)+(0, 3.2)$) {Trajectory $j=1$};
\node                 at ($(col2)+(0, 2.45)$) {$\vdots$};
\node[thinbox, faded] (jprev) at ($(col2)+(0, 1.7)$) {Trajectory $j\!-\!1$};
\node[thinbox, sel, fill=white] (jcur)  at ($(col2)+(0, 0.8)$) {Trajectory $j$};
\node[thinbox, faded] (jnext) at ($(col2)+(0,-0.1)$) {Trajectory $j\!+\!1$};
\node                 at ($(col2)+(0,-0.85)$) {$\vdots$};
\node[thinbox, faded] (jJ)    at ($(col2)+(0,-1.6)$) {Trajectory $J$};

\node[above=2pt of j1] {\textbf{Trajectories: $j=1\ldots J$}};

\node[fit=(j1)(jJ), inner sep=6pt] (trajCol) {};
\draw[decorate, decoration={brace, amplitude=10pt, mirror}]
  ($(trajCol.north east)+(0.65,0)$) -- ($(trajCol.south east)+(0.65,0)$)
  node[midway, xshift=1.5cm] {};

\coordinate (col3) at (8.5,0);

\node[thinbox, faded] (g1)   at ($(col3)+(0, 3.2)$) {Step $g\!=\!1$};
\node                 at ($(col3)+(0, 2.45)$) {$\vdots$};
\node[thinbox, faded] (gprev) at ($(col3)+(0, 1.7)$) {Step $g\!-\!1$};
\node[thinbox, sel, fill=white] (gcur)  at ($(col3)+(0, 0.8)$) {Step $g$};
\node[thinbox, faded] (gnext) at ($(col3)+(0,-0.1)$) {Step $g\!+\!1$};
\node                 at ($(col3)+(0,-0.85)$) {$\vdots$};
\node[thinbox, faded] (gG)    at ($(col3)+(0,-1.6)$) {Step $G$};

\node[above=2pt of g1] {\textbf{Gradient steps: $g=1\ldots G$}};

\node[fit=(g1)(gG), inner sep=6pt] (stepCol) {};
\draw[decorate, decoration={brace, amplitude=10pt, mirror}]
  ($(stepCol.north east)+(0.85,0)$) -- ($(stepCol.south east)+(0.85,0)$)
  node[midway, xshift=1.3cm] {};

\coordinate (col4) at (13.0,0.8);

\node[cand] (c1)   at ($(col4)+(0, 2.1)$) {Candidate $i=1$};
\node        (dotsT) at ($(col4)+(0, 1.3)$) {$\vdots$};
\node[cand] (ci)   at ($(col4)+(0, 0.5)$) {Candidate $i$};
\node        (dotsB) at ($(col4)+(0, -0.3)$) {$\vdots$};
\node[cand] (cN)   at ($(col4)+(0, -1.1)$) {Candidate $i=N$};

\node[above=10pt of c1] {\textbf{A \textsc{Best-of-N} gradient}};

\draw[call] ($(c1.north west)+(-0.25,0.25)$) rectangle ($(cN.south east)+(0.25,-0.25)$);
\node[align=center] at ($(c1)!0.5!(cN)+(-2.1,0)$) {};

\node[align=left, anchor=west] at ($(cN.south west)+(-0.3,-1.1)$)
{\hspace{-0.75em}Sample $N$ TextGrad edits\\
Apply edits to $\{\pi_{t,g}^{j, (i)})\}$\\
\hspace{-2em} Critic finds $\pi_{t,g}^{j, (i)})$ with best $\Phi(\pi_{t,g}^{j, (i)})$};

\node[fit=(g1)(gG)(c1)(cN), inner sep=6pt] (noevalRegion) {};
\draw[decorate, decoration={brace, amplitude=7pt, mirror}, draw=blue]
  ($(noevalRegion.south west)+(0,-0.4)$) -- ($(noevalRegion.south east)+(0,-0.4)$)
  node[midway, yshift=-15pt, text=blue]{Take $G$ \textsc{Best-of-N} gradients (No evaluation)};

\end{tikzpicture}%
}
\caption{A pictorial overview of \textsc{TextBO} (Algorithm~\ref{algo:tbon}) with $T$ iterations, $J$ trajectories, and $G$ \textsc{Best-of-N} gradient steps per iteration and trajectory.}
\label{fig:tbon-summary}
\end{figure}
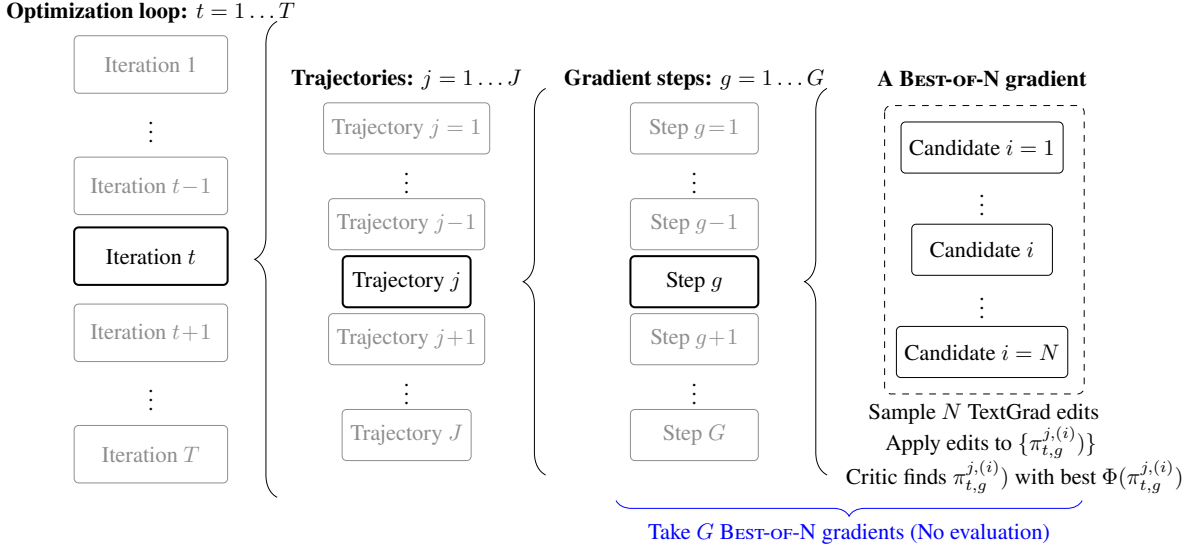

\begin{figure*}[ht!]
\centering
\adjustbox{valign=t, scale=0.9}{%
  \begin{minipage}{0.9\linewidth}
 
\begin{algorithm}[H]
\SetAlgoLined
\KwIn{System $\Phi$; $\{\pi_{0}^{j}\}_{j=1}^{J}$; critic $M_{\text{critic}}$; candidates per step $N$; gradient steps $G$; iterations $T$; $J$ trajectories}
\KwOut{Optimized triples $\{(\pi^{j}_T,c^{j}_T,s^{j}_T)\}_{j=1}^{J}$}
\For{$j = 1$ \KwTo $J$}{
  $c_0^{j} \leftarrow \Phi(\pi_0^{j})$\;
  $s_0^{j} \leftarrow \text{Eval}(c_0^{j}, \mathcal{D}_{\mathcal{X}}(\pi_0))$\;
  $H_0 \leftarrow \{(\pi_0^{j},c_0^{j},s_0^{j})\}$, $R \leftarrow \emptyset$\; 
}

\For{$t = 1$ \KwTo $T$}{
  \For{$j = 1$ \KwTo $J$}{
    $\pi^j_{t, 0} \leftarrow \pi^j_{t-1}$\;
    \For{$g = 1$ \KwTo $G$}{
      \For{$i = 1$ \KwTo $N$}{
        $\delta_{t,g}^{j, (i)} \leftarrow \nabla_{\text{text}}(\pi_{t,g-1}^{j}; R, M_{\text{critic}})$ \;
        $\pi_{t,g}^{j, (i)} \leftarrow \operatorname{Apply}(\pi_{t,g-1}^{j}, \delta_{t,g}^{j, (i)})$ \;
        $c_{t,g}^{j, (i)} \leftarrow \Phi(\pi_{t,g}^{j, (i)})$ \;
      }
      $(\pi_{t,g}^{j},c_{t,g}^{j}) \leftarrow \textsc{Best-of-N}(\{(\pi_{t,g}^{j, (i)},c_{t,g}^{j, (i)})\}, R_{t-1}, M_{\text{critic}})$\;
    }
    $c_t^j\leftarrow c^j_{t,G}$, $\pi_t^j\leftarrow \pi_{t,G}^j$\;
    $s_{t}^{j} \leftarrow \text{Eval}(c_{t}^{j}, \mathcal{D}_{\mathcal{X}}(\pi_t))$\;
    $H_t \leftarrow H_{t-1} \cup \{(\pi_{t}^{j},c_{t}^{j},s_{t}^{j})\}$\;
    \If{$s_{t}^{j} < s_{t-1}^{j}$}{
      $\pi_{t}^{j} \leftarrow \pi_{t-1}^{j}$, $c_{t}^{j} \leftarrow c_{t-1}^{j}$, $s_{t}^{j} \leftarrow s_{t-1}^{j}$\;
    }
  }
  $R_t \leftarrow \textsc{MetaReflect}(H_t)$\;
}
\KwRet{$\{c_t^{j_t^*}\}_{t=1}^T$, where $j_t^*:=\operatorname{argmax}_j s_t^j$}
\caption{\textsc{TextBO}}
\label{algo:tbon}
\end{algorithm}

  \end{minipage}
}
\end{figure*}

\paragraph{Initialization ($t=0$).}
At iteration $t=0$, we first evaluate all initial prompts. For each trajectory $j=1,\dots,J$, we first generate the initial system outcome $ \Phi(\pi_0^j)$. Next, we evaluate it on the corresponding input distribution $\mathcal{D}_{\mathcal{X}}(\pi_0^j)$ to obtain a scalar score $s_0^j$. The shared history is then initialized as
$H_0 \;=\;\bigcup_{j=1}^J \{(\pi_0^j, \Phi(\pi_0^j), s_0^j)\}$, and and the shared reflection is initialized as $R_0 = \emptyset$.

\paragraph{Per-iteration update ($t \ge 1$).}
At each optimization iteration $t\in\{1,\dots,T\}$, the algorithm updates all trajectories in parallel. For each trajectory $j\in\{1,\dots,J\}$, we set the starting prompt for this iteration to be the final prompt from the previous iteration, i.e., $\pi_{t,0}^j \leftarrow \pi_{t-1}^j$. We then perform the following three steps.

\begin{itemize}[leftmargin=0.5cm,itemsep=2pt]

\item[1.] \textbf{\textsc{Best-of-N} gradients phase.} (Detailed in $\S$\ref{sec:BoNgrad}.)  
For each trajectory $j$, we take $G$ successive \textsc{Best-of-N} textual-gradient steps, which do not involve any evaluation. That is, for $j=1,\ldots J$, for $g = 1,\dots,G$:

\begin{enumerate}[leftmargin=0.7cm,itemsep=2pt]
  \item[1)] \emph{Textual gradient generation.}  
  From the current prompt $\pi_{t,g-1}^j$, query the textual-gradient operator
  \[
      \delta_{t,g}^{j,(i)} \sim \nabla_{\text{text}}(\pi_{t,g-1}^j; R_{t-1}, M_{\text{critic}})
  \]
  independently $N$ times to obtain textual gradients $\{\delta_{t,g}^{j,(i)}\}_{i=1}^N$.

  \item[2)] \emph{Candidate prompt generation.}  
  Apply each textual gradient to the current prompt $\pi_{t,g-1}^j$ to obtain candidate prompts $\{\pi_{t,g}^{j,(i)}\}_{i=1}^N$:
  \[
      \pi_{t,g}^{j,(i)} \leftarrow \operatorname{Apply}(\pi_{t,g-1}^j, \delta_{t,g}^{j,(i)}),
      \qquad i=1,\dots,N.
  \]

  \item[3)] \emph{Candidate system outcome generation.}  
  For each candidate prompt $\pi_{t,g}^{j,(i)}$, generate the corresponding system outcome:
  \[
      \Phi(\pi_{t,g}^{j,(i)}), \qquad i=1,\dots,N.
  \]

  \item[4)] \emph{\textsc{Best-of-N} selection.}  
  Using the critic model and the shared reflection $R_{t-1}$, select the candidate that is predicted to be the most promising by the critic model among
  $\{(\pi_{t,g}^{j,(i)},  \Phi(\pi_{t,g}^{j,(i)})\}_{i=1}^N$, and set the winner as
  $(\pi_{t,g}^j,  \Phi(\pi_{t,g}^{j}))$.
\end{enumerate}

After $G$ gradient steps, trajectory $j$ has an updated prompt-outcome pair $(\pi_{t,G}^j,  \Phi(\pi_{t,G}^{j}))$ for this iteration.

Geometrically, as illustrated in Figure~\ref{fig:TBoN_real}, at iteration $t$ the prompts
$\{\pi_{t-1}^j\}_{j=1}^J$ (equivalently, the initial points $\{\pi_{t,0}^j\}_{j=1}^J$) act as starting locations in the implicit UCB landscape over prompt space. For each trajectory $j$, the $G$ \textsc{Best-of-N} gradient steps climb this landscape, producing the updated prompts $\{\pi_{t,G}^j\}_{j=1}^J$ that approximate local maxima of the UCB acquisition function.\footnote{For $\{\pi_{t,G}^j\}_{j=1}^J$ to be considered local optima in the usual gradient-based BO sense, $G$ should be large enough.} These local optima are then passed to the evaluation phase (Phase~2).


\begin{figure}[ht!]
    \centering
    \includegraphics[width=0.63\linewidth]{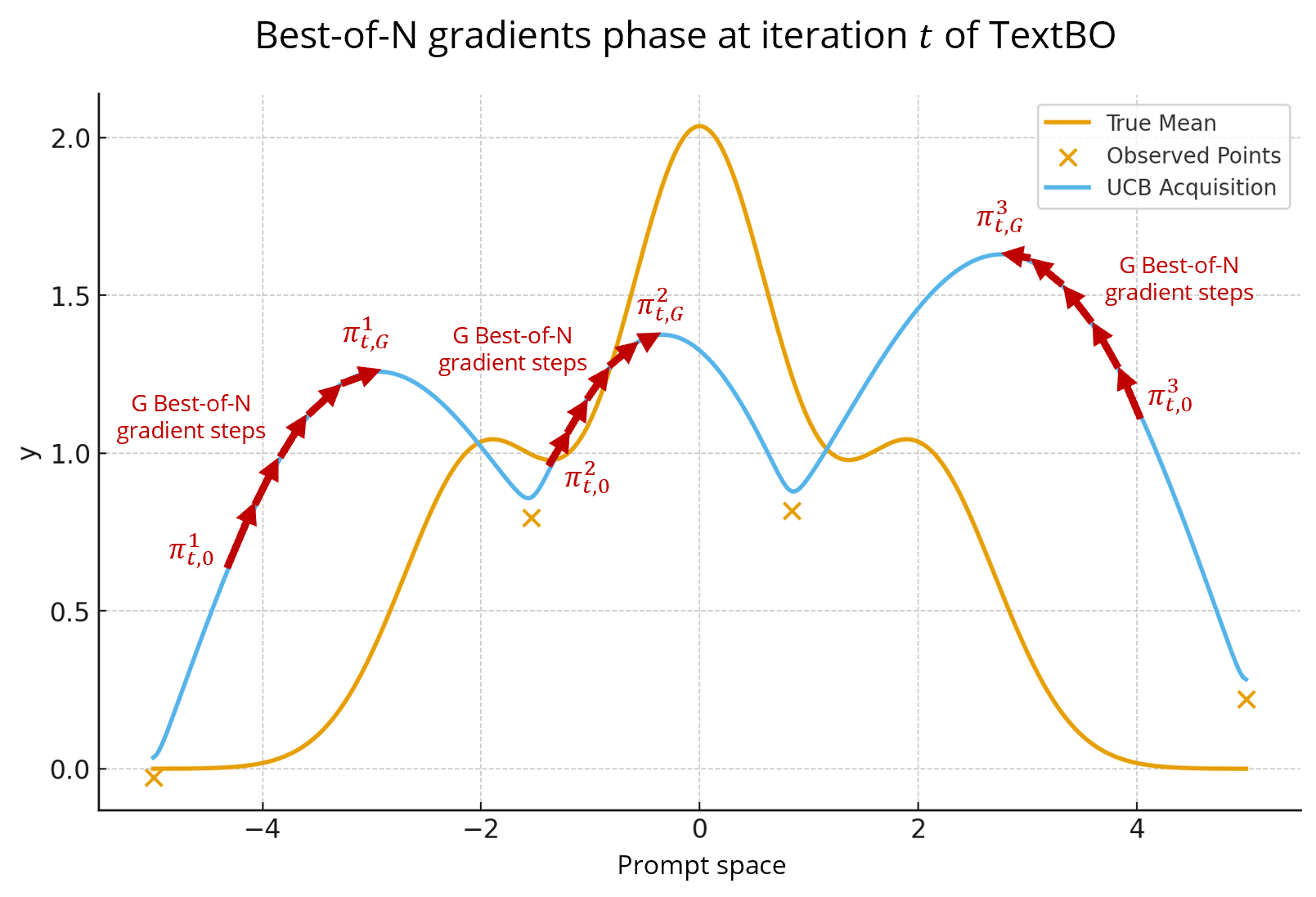}
    \caption{Illustration of how \textsc{TextBO} (Algorithm \ref{algo:tbon}) implements \textsc{Best-of-N} gradients phase at iteration $t$. The prompts identified as local optima in this phase are evaluated in the later evaluation phase. }
    \label{fig:TBoN_real}
\end{figure}

\item[2.] \textbf{Evaluation phase.}  
For each trajectory $j$:

\begin{enumerate}[leftmargin=0.7cm,itemsep=2pt]
    \item[1)] Set the trajectory’s candidate for evaluation to the final output of the \textsc{Best-of-N} gradient phase:
    \[
        \pi_t^j := \pi_{t,G}^j
    \]
    \item[2)] Evaluate $c_t^j$ to obtain the scalar score
    \begin{align}
        s_t^j
        &:= \mathrm{Eval}(\Phi(\pi_t^j), \mathcal{D}_{\mathcal{X}}(\pi_t^j)) = \frac{1}{L} \sum_{l=1}^L r\!\left( y\!\left(\Phi(\pi_t^j), x_t^{j,(l)}\right)\right), \;\;\textrm{where} \;\;  x_t^{j,(l)} \sim  \mathcal{D}_{\mathcal{X}}(\pi_t^j). 
        \label{eq:evaluation} 
    \end{align}
    Note that this is the step where the evaluation happens; and in each iteration, we conduct $J$ evaluations -- one for each trajectory.
    \item[3)] Append the new triple $(\pi_t^j, \Phi(\pi_t^j), s_t^j)$ to the shared history:
    \[
        H_t \leftarrow H_{t-1} \cup \{(\pi_t^j, \Phi(\pi_t^j), s_t^j)\}.
    \]
    \item[4)] \emph{Acceptance rule.} If the new score does not improve, revert this trajectory to its previous state:
    \[
        \text{if } s_t^j \le s_{t-1}^j \text{ then }
        (\pi_t^j, \Phi(\pi_t^j), s_t^j) \leftarrow (\pi_{t-1}^j, \Phi(\pi_{t-1}^j), s_{t-1}^j).
    \]
\end{enumerate}

This acceptance rule follows \citet{textgrad}: it enforces monotone non-decreasing scores along each trajectory by rolling back any non-improving update.

\item[3.] \textbf{Shared meta-reflection update.}  

At the end of iteration $t$, after all $J$ trajectories have been evaluated and the history updated, we recompute the global reflection:
\[
    R_t \leftarrow \textsc{MetaReflect}(H_t; M_{\text{critic}}).
\]
The operator $\textsc{MetaReflect}(H_t; M_{\text{critic}})$ compresses the full history $H_t$---prompts, system outcomes, and scores---into a concise natural-language reflection $R_t$. The reflection consists of a small set of actionable rules summarizing cross-trajectory patterns. In the ad-generation setting, such rules might take the form “emphasize social context,” “avoid cluttered backgrounds,” or “prefer warm lighting.” This reflection serves as a global memory: in subsequent iterations, the critic LLM conditions on $R_t$ when producing textual gradients and performing \textsc{Best-of-N} judgments.

There are many ways to construct such reflections. In our experiments, following \citet{li2025test}, we instruct $M_{\text{critic}}$ to compare the five lowest- and five highest-performing outcomes from previous iterations (or all available outcomes if fewer than ten exist)\footnote{Limiting the number of outcomes in the critic model's context window helps mitigate context rotting \citep{hong2025context}.}, identify recurring visual and textual patterns, and distill these patterns into a small number of generalizable rules.
\end{itemize}

\noindent After $T$ optimization iterations, we report as the algorithm’s output the sequence of best-performing outcomes
\[
\{\Phi(\pi_t^{j_{t}^*})\}_{t=1}^T,
\qquad j_t^\ast := \operatorname{argmax}_{j} s_t^j,
\]
i.e., the best trajectory at each iteration, according to the evaluation scores.


Note that, in each iteration, \textsc{TextBO} generates $N \times G$ prompt/outcome candidates for each trajectory before conducting one evaluation. This deliberate trade-off is the key to its evaluation efficiency: the algorithm uses inexpensive generation and critic-based comparisons to identify promising directions, so that each actual evaluation is performed on a small set of carefully selected, high-potential candidates. 


\section{Case Study: Digital Ad Optimization}
\label{sec:case}

We now instantiate our framework in a digital advertising setting, showing how
an advertiser (or ad platform acting on behalf of an advertiser) can use \textsc{TextBO} to automate the
 ad generation–analysis–evaluation loop. Here, the goal is to refine the ad-generation prompts
so that the resulting creatives are better aligned with a target population's preference
distribution, under a limited budget of costly ad evaluations. We specialize the general setup from $\S$\ref{sec:prelim} to ad optimization as follows:
\squishlist
    \item $\Phi$ is an ad-generation AI system that can generate ads given a prompt $\pi$.\footnote{In practice, this image-generation model can be fine-tuned for a particular brand and/or customized to the platform. Our approach takes any pre-training as given.} 

    \item For a prompt $\pi$, $\Phi(\pi)$ is the resulting ad creative deployed on the platform.
    
    \item $\mathcal{X}$ is the population of platform users, and $\mathcal{D}_{\mathcal{X}}(\pi)$ is the (unknown) distribution of users who are exposed to $\Phi(\pi)$ when this ad is run.

    \item $\mathcal{Y}$ is the space of observed ad responses (clicks, conversions, spend). For each user $x\in\mathcal{X}$, $y(\Phi(\pi),x)$ is the random ad response when $x$ is shown ad $\Phi(\pi)$.
    \item $r:\mathcal{Y}\to[0,1]$ is a scoring rule that maps responses to a normalized effectiveness score (e.g., a weighted function of click, conversion, and spend).
\squishend

The campaign objective is the same as the general objective introduced in $\S$\ref{sec:prelim}: $ J(\pi)
= \mathbb{E}\big[\,r\big(y(\Phi(\pi), x)\big)\big]$, where the expectation taken over $x \sim \mathcal{D}_{\mathcal{X}}(\pi)$ and the
stochasticity in $y(\Phi(\pi), x)$.

\subsection{Ad Campaign Scenarios}
We consider eight synthetic ad campaign scenarios that cover a diverse range of products
across distinct categories, each defined by a creative brief that outlines the strategic
and creative direction (see Web Appendix~\ref{ssec:scenarios} for the full creative briefs):
\begin{itemize}[leftmargin=0.5cm,itemsep=0pt]
    \item Scenario 1: ``GreenBite,'' a new plant-based burger patty.
    \item Scenario 2: ``AuraSonics X1,'' high-end, noise-canceling wireless earbuds.
    \item Scenario 3: ``Odyssey E-SUV,'' a new all-electric family SUV.
    \item Scenario 4: ``Oasis Eco-Lodge,'' a secluded, luxury resort with beautiful natural surroundings.
    \item Scenario 5: ``Momentum,'' a mobile-first banking app for freelancers and the gig economy.
    \item Scenario 6: ``MindGarden,'' a subscription-based meditation and mindfulness app.
    \item Scenario 7: ``Aeterno,'' a classic, automatic Swiss-made wristwatch with a heritage design.
    \item Scenario 8: ``SyncFlow,'' a project management and collaboration software platform for remote teams.
\end{itemize}

\subsection{Implementation Details}\label{ssec:implementDetail}
There are five main components of the algorithm: (1) initial prompt ($\pi_0$), and (2) Ad-generation module ($\Phi$), (3) Critic LLM and its four operators defined in $\S$\ref{ssec:critic-llm}, (4) the population distribution  $\mathcal{D}_{\mathcal{X}}(\pi)$ and the evaluation operator $\mathrm{Eval}()$, as defined in Equation \eqref{eq:evaluation}. We now describe how each of these components is implemented in our experiments. 

\paragraph{Initial prompts ($\pi_0^j$).}
For each scenario, we first created 64 high-quality prompts from its creative brief using a meta-prompt. In Web Appendix \ref{appssec:initial_prompts}, see Figure \ref{fig:init_gen_prompt} for the details of the meta-prompt and Figure~\ref{fig:best_ad_prompt} for an example of two prompts in scenario 1. Among those 64 ads, we chose five as initial prompts; we will describe the details of the choice rule in $\S$\ref{ssec:application_tbonbo_ad}.

\paragraph{Ad-generation module ($\Phi$).} 
We utilize a state-of-the-art image generation model, Google Imagen~4,\footnote{See
Vertex AI's description
(\url{https://console.cloud.google.com/vertex-ai/publishers/google/model-garden/imagen-4.0-generate-preview-06-06})
for details.} Imagen 4 delivers high-fidelity, photorealistic text-to-image generation at up to 2K resolution, and excels in rendering intricate details such as realistic textures, lighting, and camera effects. Its features include typography accuracy for legible in-image text and precise adherence to complex prompts.

\paragraph{Critic LLM and its Operators.} 

Unless otherwise noted, we use Gemini~2.5~Flash \footnote{See
Vertex AI's description
(\url{https://docs.cloud.google.com/vertex-ai/generative-ai/docs/models/gemini/2-5-flash})
for details.} as the multimodal critic model
    $M_{\text{critic}}$ and implement the four operators discussed in $\S$\ref{ssec:critic-llm} using the meta-prompts described in Web Appendix \ref{appssec:critic_llm_prompts}: (1) \textsc{PairwiseJudge} operation; see meta-prompt in Figure \ref{fig:tornament_prompt}, (2)  $\textsc{MetaReflect}$ operation; see meta-prompt in Figure \ref{fig:reflection_prompt}, (3) Textual gradient operation $\nabla_{\text{text}}$; see meta-prompt in Figure \ref{fig:prompt_improvement_prompts}, and (4) Rewrite operation \textsc{Apply}; see meta-prompt in Figure \ref{fig:apply_prompt}.

\paragraph{Evaluation module using persona data and LLM-based simulation.}We now describe how we implement the evaluation operator $\mathrm{Eval}$ in this case study. To evaluate how well an ad performs on the target population and to update prompts, we need (i) a target population ($\mathcal{D}_{\mathcal{X}}(\pi)$.) and (ii) a way to measure that population's preferences for a given ad ($\mathrm{Eval}(\cdot)$). In principle, there are three ways define the evaluation operator -- (1) field experiments/deployment in a real ad platform, (2) consumer surveys, and (3) simulated responses from a pre-defined target population. The former two options (field deployment and consumer surveys) are costly in time and resources, and can require large samples to identify effective differences across ads given low CTRs in digital ad settings; see \citep{lewis2015unfavorable, johnson2023inferno}. Thus, using these approaches for algorithm evaluation is not ideal. 

Therefore, in our experiments, we use the Twin-2k-500 persona distribution of \citet{toubia2025twin} as the target population and LLM-generated preferences over this distribution as the evaluation module. The Twin-2K-500 dataset consists of ``digital twins" or ``digital persona" of $2,058$ representative people from the United States. The dataset was assembled by giving a comprehensive multi-wave survey (consisting of a total of 500 questions) to each person in the sample, spanning demographics, personality, cognitive ability, economic preferences, classic
heuristics-and-biases tasks, and a pricing study. Thus, the resulting personas are quite rich and detailed. This persona distribution forms our target population $\mathcal{D}_{\mathcal{X}}(\pi)$. Further, consistent with a randomized experiment design, we assume that the ad does not influence the personas who are shown the ad, i.e., $\mathcal{D}_{\mathcal{X}}(\pi)$ is the same for all $\pi$s. For each persona, we pass their survey answers to a multi-modal LLM (Gemini 2.5 Flash) as context, and then ask the LLM to assess the effectiveness of a given ad for that persona. Concretely, for a given persona and ad, we ask the LLM (see Figure \ref{fig:persona_prompt} for the detailed prompt in Web Appendix \ref{appssec:persona_prompt}) to rate the ad's ``effectiveness'' on a scale of
$1$--$5$, where $1$ is least effective and $5$ is most effective. Because the LLM returns
token-level log probabilities over the discrete scores,\footnote{See Vertex AI's logprobs description (\url{https://cloud.google.com/vertex-ai/generative-ai/docs/model-reference/inference\#logprobs}) for details.} we convert these into a mean score and treat this mean (a real number such as $3.41$) as the ad's evaluation score for that persona. Averaging over personas then yields the final evaluation score used by
$\mathrm{Eval}()$.

We split the Twin-2k-500 personas into 80\% (1{,}647 personas) for training and 20\% (411 personas) for testing. During algorithm training, each evaluation of an ad $\Phi(\pi)$ proceeds as follows: we randomly sample 200 personas from the training set and simulate ad effectiveness for each persona via the LLM. When we test a final ad (e.g., after running \textsc{TextBO}), we simulate its effectiveness for all personas in the test set.

We note that important advantage of our approach: using this persona distribution with LLM-generated preferences enables fast and controlled comparisons of different algorithms without large-scale field experiments or expensive surveys.\footnote{Naive field deployment cannot be used to test the relative performance of algorithms, since an algorithm’s measured advantage would then mix its own effect with distributional effects induced by the ad platform's targeting systems.} Indeed, our approach can serve as a template for algorithm comparisons in many other settings where researchers require evaluations that reflect the variance/heterogeneity of human preferences. Further, we note that while LLM-based preference simulation may not reflect true human preferences \citep{li2025llm, peng2025mega}, it still induces a well-defined preference distribution on which we can benchmark optimization methods and algorithms. Specifically, \citet{kang2025llmpersonassubstitutefield} proves that swapping humans for personas is equivalent to changing the evaluation population (e.g., from the New York population to the Jakarta population), under two benchmark hygiene conditions: (i) methods observe only the aggregate outcome
(called \textit{aggregate-only observation} condition) and (ii) evaluation
depends only on the submitted artifact and not on
the algorithm’s identity or provenance (called \textit{algorithm-blind evaluation} condition).
Furthermore, they show that simply increasing the size of the persona dataset is sufficient to guarantee that persona simulation becomes as useful a benchmark as field experimentation. Nevertheless, we also provide evaluations of our approach that do not leverage persona datasets; see $\S$\ref{ssec:ablation} for details.

\subsection{Application of \textsc{TextBO}}
\label{ssec:application_tbonbo_ad}

We now describe how the \textsc{TextBO} algorithm is instantiated for the ad-optimization problem. Here, we choose the hyperparameters as: iterations $T=10$; trajectories $J=5$; gradient steps $G=5$; and candidates per gradient step $N=5$.

\paragraph{Initialization.}  

As described in $\S$\ref{ssec:implementDetail}, for each scenario, we first generate 64 high-quality prompts from its creative brief using a meta-prompt. Among the 64 ads for each scenario, we identify the best initial ad, \textsc{Best-of-64} (which will be used as one of the baselines but is not directly used in our algorithm; see $\S$\ref{ssec:baselines} for details), and the worst five, \textsc{Worst5-of-64}, by running a best-arm
identification bandit algorithm \citep{russo2016simple}. Specifically, we alternate
between a best-arm identification objective and a worst 5-arm identification objective
over 5{,}000 sequential random samples from the training persona set. These
\textsc{Worst5-of-64} ads serve as starting points $\{\pi_0^j\}_{j=1}^{5}$ for
\textsc{TextBO} procedure. This starting point allows us to attribute observed gains to the algorithms rather than to favorable starting prompts.\footnote{To maintain a rigorous baseline, we make these 64 prompts detailed enough to yield strong creatives. This ensures that subsequent improvements are not merely due to trivial gains from starting with an under-specified
prompts; instead, improvements reflect better alignment with the underlying target audience distribution.} Let $\{s_0^j\}_{j=1}^J$ denote the evaluation scores for these five ads. Then, the tuples $\{(\pi_0^j,\Phi(\pi_0^j),s_0^j)\}_{j=1}^J$ form the initial history $H_0$ for the algorithm.

\paragraph{\textsc{Best-of-N} textual-gradient steps.}  
Within each optimization iteration $t$, \textsc{TextBO} refines each trajectory’s prompt
$\pi_{t-1}^j$ through five steps of \textsc{Best-of-N} textual-gradients before evaluation (see
Figure \ref{fig:tbon-summary}). We follow the generic procedure from
$\S$\ref{sec:TBoN}, but here specialize the critic’s context and edits to the ad
domain.

For a given trajectory $j$ and inner step $g$, the critic model $M_{\text{critic}}$ (Gemini~2.5 Flash) receives as context: (1) the current prompt $\pi_{t,g-1}^j$ and (2) the 
current meta-reflection $R_{t-1}$. Conditioned on this context, using operation $\nabla_{\text {text }}$, it independently samples $N$ textual gradients
$\{\delta_{t,g}^{j,(i)}\}_{i=1}^N$, applies them using operation \textsc{Apply} to obtain candidate prompts
$\{\pi_{t,g}^{j,(i)}\}_{i=1}^N$. These candidate prompts are then fed to the ad generation module $\Phi$ to generate the corresponding 5 ads
$\{c_{t,g}^{j,(i)}\}_{i=1}^N = \{\Phi(\pi_{t,g}^{j,(i)})\}_{i=1}^N$. For example, suppose that the current prompt $\pi_{t,g-1}^j$ is ``\texttt{A simple, photorealistic image of a plant-based burger patty.}'', if $R_{t-1}$ emphasizes that ``social
occasions and visible grill marks tend to perform better than plain pack shots,'' then we may query the operation $\nabla_{\text {text }}$ (see Figure \ref{fig:prompt_improvement_prompts} in Web Appendix \ref{appssec:critic_llm_prompts} for the meta-prompt) to propose textual-gradient edits such as:
\squishlist
    \item \emph{Emphasize social context and enjoyment:}  
    ``Show the burger being enjoyed at an outdoor barbecue with friends, not alone on a plate.''

    \item \emph{Highlight sensory appeal and indulgence:}  
    ``Make the patty look extra juicy with clear grill marks, a toasted bun, and melty toppings.''

    \item \emph{Make sustainability salient but secondary:}  
    ``Subtly include eco-friendly cues (like a small `plant-based' tag or greenery in the background) without overpowering the food.'' 
\squishend
We then query these edits for the operation \textsc{Apply} operation to an initial prompt (see Figure \ref{fig:apply_prompt} in Web Appendix \ref{appssec:critic_llm_prompts} for the meta-prompt) may yield candidates such as:
\begin{quote}
\small
\texttt{``A photorealistic image of a juicy plant-based burger with grill marks, served at a vibrant summer barbecue with friends, everyone smiling and reaching for food.''} 
\end{quote}
\begin{quote}
\small
\texttt{``A close-up, photorealistic shot of a plant-based burger with deep grill marks, a toasted brioche bun, melty vegan cheese, fresh toppings, and steam rising to suggest warmth and juiciness.''} 
\end{quote}
\begin{quote}
\small
\texttt{``A photorealistic image of a sizzling plant-based burger with rich grill marks, served on a rustic wooden table with subtle greenery in the background and a small plant-based tag, conveying indulgence without sacrificing sustainability.''}
\end{quote}

To select which candidate to keep for the next gradient step, \textsc{TextBO} applies the \textsc{Best-of-N} rule from $\S$\ref{sec:BoNgrad} to the ads
$\{\Phi(\pi_{t,g}^{j,(i)})\}_{i=1}^N$. Specifically, following \citet{liu2025pairjudgermperformbestofn}, we
implement \textsc{Best-of-N} via a pairwise tournament: in each match, the critic compares two
creatives side by side and predicts which will perform better using operator \textsc{PairwiseJudge}, given the campaign goals
and the current reflection $R_{t-1}$ (see Figure~\ref{fig:tornament_prompt} in Web
Appendix~\ref{appssec:critic_llm_prompts} for the meta-prompt). The tournament winner prompt and its corresponding ad, $\left(\pi_{t, g}^j, \Phi\left(\pi_{t, g}^j\right)\right)$, are carried forward to the next gradient step. Repeating this for $g=1,\dots,G$ yields the final refined
prompt $\pi_t^j$ and ad $\Phi(\pi_t^j)$ for trajectory $j$ at iteration $t$.

\paragraph{Evaluation as the costly step.}   
After the $G$ gradient steps, \textsc{TextBO} performs a single expensive evaluation per
trajectory (each of which consists of assessing the ad for 200 randomly sampled personas from the training persona set) by deploying  $\Phi(\pi_t^j)$ to get the evaluation result $s_t^j$ using the evaluation
operator $\mathrm{Eval}$ from Equation \eqref{eq:evaluation} (instantiated using the evaluation module discussed in 
$\S$\ref{ssec:implementDetail}). The new triple $(\pi_t^j,\Phi(\pi_t^j),s_t^j)$ is appended to $H_t$, and the usual
acceptance rule is applied: if $s_t^j \le s_{t-1}^j$, the trajectory reverts to
$(\pi_{t-1}^j,\Phi(\pi_{t-1}^j),s_{t-1}^j)$, ensuring non-decreasing scores along each
trajectory.

\paragraph{Meta-reflection shared across trajectories.}  
After all $J$ trajectories have been evaluated at iteration $t$, the algorithm updates the
global meta-reflection $R_t = \textsc{MetaReflect}(H_t;M_{\text{critic}})$. Concretely, $M_{\text{critic}}$ is prompted to review a subset of
the best- and worst-performing ads and prompts in $H_t$ and to summarize the key patterns
that distinguish successful creatives from unsuccessful ones (see
Figure~\ref{fig:reflection_prompt} in Web Appendix~\ref{appssec:critic_llm_prompts} for the meta-prompt). For example, it may
infer that scenes with social eating and visible grill marks outperform isolated pack
shots, or that emphasizing ``satisfying'' and ``juicy'' tends to improve effectiveness.
This reflection is then injected into the context of subsequent textual-gradient
generation and \textsc{Best-of-N} comparisons, allowing trajectories to share accumulated
knowledge even as they explore different regions of the prompt space.

\paragraph{Reported outcome.}  
After $T=10$ iterations, the system reports the sequence of best-performing ads
$\{\Phi\left(\pi_t^{j_t^*}\right)\}_{t=1}^T$, where $j_t^\ast = \arg\max_j s_t^j$ is the trajectory that
achieved the highest score at iteration $t$.

\section{Experiments for the Digital Ad optimization}
\label{sec:experiments}

In this section, we discuss the numerical performance of \textsc{TextBO} for the ad-optimization case (discussed in \S \ref{sec:case}).  In $\S$\ref{ssec:baselines}, we first discuss the baseline algorithms we compare against, and then in $\S$\ref{ssec:exp_design}, we describe the experiment design. We present the main results in $\S$\ref{ssec:results_analysis} and then finally end with an ablation study in $\S$\ref{ssec:ablation}.

\subsection{Baselines}
\label{ssec:baselines}

We consider two state-of-the-art algorithms in the AI literature as baselines -- (1) \textsc{Best-of-N} \citep{snell2024scaling} and (2) \textsc{GEPA} \citep{agrawal2025gepa}. 

\textsc{Best-of-N} is one of the most popular and well-performing test-time alignment methods that has been reported to outperform reinforcement learning fine-tuning methods under large enough $N$ \citep{gao2023scaling, mudgal2023controlled, eisenstein2023helping, beirami2024theoretical}. For any given scenario, this algorithm samples $N$ candidate outputs from a fixed generator and returns the single candidate with the highest predicted score under a judge/reward model, trading extra inference-time compute for higher expected quality \citep{snell2024scaling, beirami2024theoretical}. In our setting, recall that we generate a set of 64 initial ads (see $\S$\ref{ssec:implementDetail} on the details of the initial ad set generation). Then, we use the \textit{training} dataset of personas and the $\textrm{Eval()}$ operator in conjunction witha best-arm identification algorithm to identify the \textsc{Best-of-64} ad (as described in  $\S$\ref{ssec:application_tbonbo_ad}).  Finally, we evaluate \textsc{Best-of-64}'s performance on the \textit{testing} persona dataset; this score serves as our \textsc{Best-of-64} baseline for that scenario.

\textsc{GEPA} is a reflective and evolutionary prompt optimization algorithm that has demonstrated superior performance compared to state-of-the-art reinforcement fine-tuning techniques such as GRPO \citep{liu2024deepseek}. In our implementation of \textsc{GEPA}, we use the same set of initial prompts, critic models, ad generation modules, and evaluation operators we use for \textsc{TextBO} (see $\S$\ref{ssec:implementDetail} for details). Further, for each of the eight scenarios, we start \textsc{GEPA} with the same Worst5-of-64 prompts as \text{TextBO}, i.e., the worst five initial prompts we identified in $\S$\ref{ssec:implementDetail}. We also run \textsc{GEPA} for the same number of iterations ($T=10$). Here, keeping the components of the algorithms and the starting points the same across both algorithms ensures a fair comparison.  
Notably, when we adapt \textsc{GEPA} to the digital ad optimization problem setup (where we cannot store and retrieve problem instances (e.g., personas) freely anytime we want as in agentic AI benchmarks of \cite{agrawal2025gepa}), it effectively reduces to \textsc{TextBO} with $G=1$ and $N=1$\footnote{In $\S$\ref{sec:agenticExp}, we adapt the principles of \textsc{Best-of-N} gradient multiple gradient steps of \textsc{TextB} for \textsc{GEPA} and test it for agentic AI benchmark experiments \citep{agrawal2025gepa} for which \textsc{GEPA} was tested.}; see Web Appendix \ref{appssec:GEPA_implementation} for the implementation details of the adapted \textsc{GEPA}.

\subsection{Experiment Design}
\label{ssec:exp_design} 

Our experiment is designed to test whether we can start from ex-post least-aligned ads (\textsc{Worst5-of-64}) and apply \textsc{TextBO} to outperform the ex-post best-aligned ad (\textsc{Best-of-64}). Here, “alignment” is measured by the evaluation score induced by our persona-based evaluation environment (see $\S$\ref{ssec:implementDetail}). This design choice matters because the \textsc{Best-of-64} baseline is already a strong, ex-post selection benchmark: it is the best-performing ad within a reasonably rich, human-like initial pool of 64 creatives generated from the same creative brief \citep{snell2024scaling, beirami2024theoretical}. 
By initializing optimization from the opposite end of the same pool (\textsc{Worst5-of-64}), we attribute improvements to the algorithm’s performance rather than to a favorable starting prompt. Relatedly, we are also interested in comparing the performance of \textsc{TextBO} with \textsc{GEPA}. To do so, we also need to ensure that the comparison between \textsc{GEPA} and \textsc{TextBO} is fair and both use similar starting points and algorithm components. Therefore, for each scenario, we proceed as follows:
\squishlist
    \item \textit{Construct a common initial candidate pool.} Using the scenario’s creative brief, we generate 64 detailed prompts and their corresponding ad images. All algorithms start with this initial pool of prompts and ads. 
    \item \textit{Define ex-post baselines and initializations within the pool.} Using the training persona set and the fixed evaluation operator $\mathrm{Eval}(\cdot)$, we identify (i) the single best ad in the pool (\textsc{Best-of-64}) and (ii) the five worst ads (\textsc{Worst5-of-64}) via best-arm identification (see $\S$\ref{ssec:application_tbonbo_ad}). We use \textsc{Worst5-of-64} as the starting set $\{\pi_0^j\}_{j=1}^5$ for both \textsc{TextBO} and  \textsc{GEPA}, while \textsc{Best-of-64} serves as the ex-post “best-aligned” benchmark.
    \item \textit{Hold the algorithm components and evaluation environment fixed across methods.} All methods use the same (i) scenarios and initial pool, (ii) ad generator $\Phi$, (iii) critic model $M_{\text{critic}}$, and (iv) persona-based evaluation operator. This ensures that differences in outcomes are attributable to the optimization procedure, rather than the evaluation environment. 
    \item \textit{Evaluate under a common budget and report out-of-sample performance.} During optimization steps in \textsc{GEPA} and \textsc{TextBO}, each costly evaluation uses the same protocol (randomly sampled 200 personas from the training set; see $\S$\ref{ssec:implementDetail}). For reporting, we evaluate the resulting ads on the held-out 411 test persona set and compare progress trajectories and final performance across methods.
\squishend

\subsection{Results and Analysis}
\label{ssec:results_analysis}
We now present the main results from the experiments. Note that while the algorithm is designed to minimize regret, we cannot directly plot regret here since we don't have true $\pi^{\star}$ (or the the prompt that maximizes the score), a priori, for any of the scenarios. Therefore, throughout, we report experiment results in terms of empirical performance scores.



\begin{figure}[!ht]
    \centering
    \includegraphics[width=0.70\linewidth]{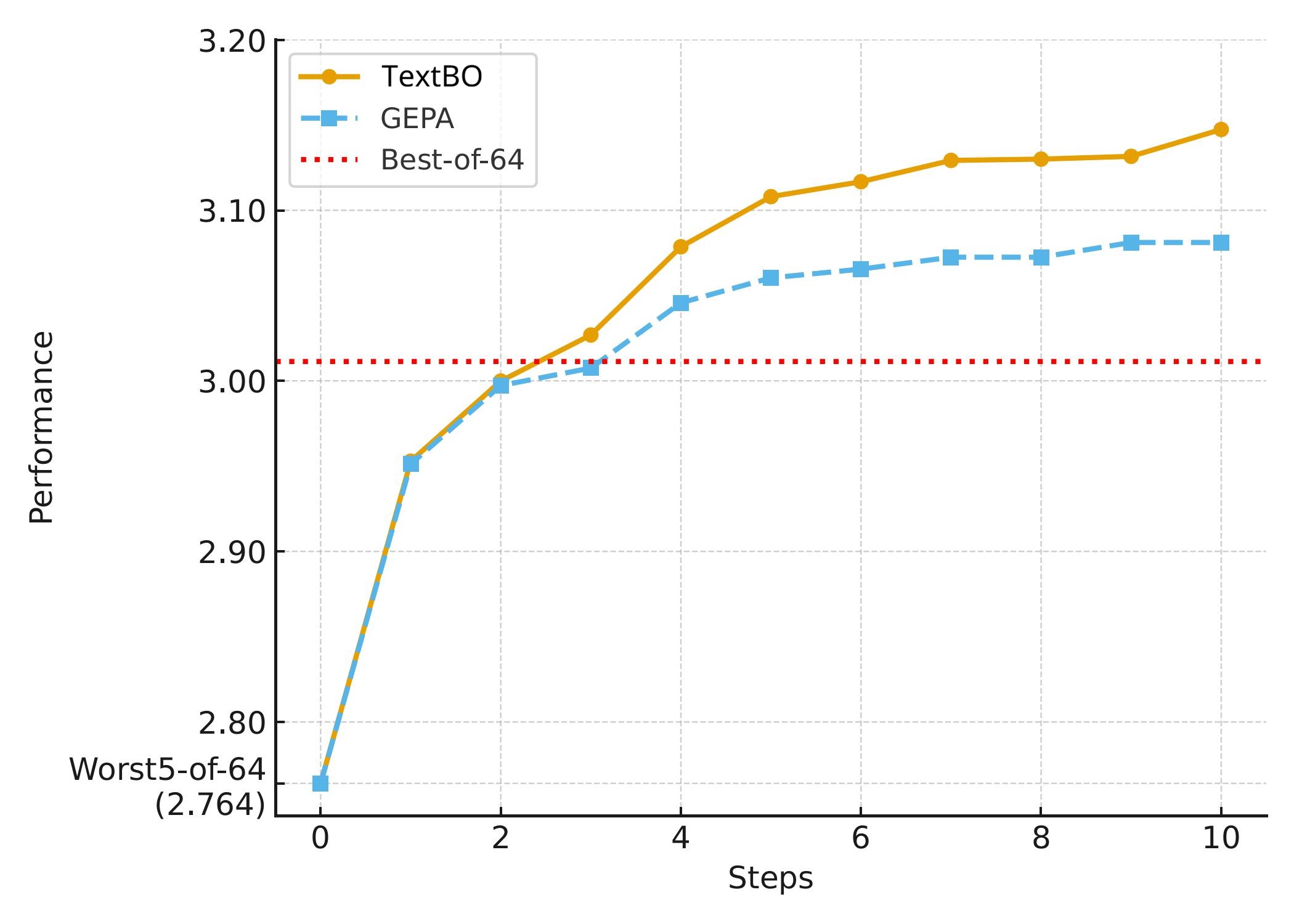}
    \caption{Progress comparison of \textsc{TextBO} and \textsc{GEPA} with \textsc{Best-of-64} baseline. \textsc{TextBO} implements parallel \textsc{TextBO} with $J=5$. The performance score (the $y$-axis) is the average of eight scenarios' mean evaluation score for the testing persona set. }
    \label{fig:performance}
\end{figure}


\paragraph{Main results and trend analysis:} Figure \ref{fig:performance} shows the results of our numerical experiments over 10 iterations: it plots how much \textsc{TextBO} and \textsc{GEPA} improved their performance from the initial \textsc{Worst5-of-N} score across eight scenarios. (For the detailed data used to generate this plot, see Tables \ref{tab:steps_cs} and \ref{tab:steps_cs_gepa} in Web Appendix \ref{appssec:detailedExpResult}.)  At each step (the $x$-axis), the performance score (the $y$-axis) is the average of eight scenarios' mean evaluation score for the test persona dataset.\footnote{\textsc{Best-of-N} remains constant because, as discussed earlier, it is simply the best among the initial pool of ads and this baseline does not involve generation of new ads.} For example, the step 2 score is the average of the eight step 2 mean evaluation scores.
Both \textsc{TextBO} and \textsc{GEPA} improve consistently across steps, with \textsc{TextBO} demonstrating faster early gains and achieving a significantly better final performance than \textsc{GEPA}. Notably, by the third step on average, which is equivalent to 600 persona evaluations, both \textsc{TextBO} and \textsc{GEPA} outperform the \textsc{Best-of-64} baseline. Further, by step 8, both algorithms have reached stable performance.


\begin{table}[!ht]
\centering
\begin{threeparttable}

\begin{tabular}{lccc}
\toprule
 \textit{Effect} & Estimate & Std. err & T-stat \\[2pt]

$\beta_1$: \textsc{GEPA} gain $(t=10)$ from \textsc{Worst5-of-64} 
  & 0.3168*** & 0.0325 & 9.752 \\

$\beta_2$: \textsc{TextBO} gain - \textsc{GEPA} gain $(t=10)$ from \textsc{Worst5-of-64}  
  & 0.0743*** & 0.0248 & 2.991 \\

$\beta_1+\beta_2$: \textsc{TextBO} gain $(t=10)$ from \textsc{Worst5-of-64}
  & 0.3911*** & 0.0355 & 11.011 \\

\midrule
\multicolumn{4}{l}{\textit{Summary statistics}} \\
Observations      & \multicolumn{3}{c}{13,152} \\
Personas  & \multicolumn{3}{c}{411} \\
Scenarios & \multicolumn{3}{c}{8} \\
\bottomrule
\end{tabular}

\begin{tablenotes}[flushleft]
\footnotesize
\item Notes: Two-way cluster-robust standard errors. 
\item * $p<0.10$, ** $p<0.05$, *** $p<0.01$.
\end{tablenotes}

\end{threeparttable}

\caption{Two-way fixed effect (TWFE) analysis result with persona and scenario fixed effects with two-way clustered SEs. Standard error of $\beta_1+\beta_2$ is computed using the Delta method.}
\label{tab:twfe_matrix}
\end{table}

\paragraph{Statistical analysis:} Figure \ref{fig:performance} shows the mean of means (i.e., scores are averaged over personas in test dataset and then over the eight scenarios). Since both the set of personas and scenarios are common across all algorithms, we cannot directly plot standard errors. Therefore, we now present a statistical analysis using the raw scores at the persona-scenario level and examine whether the performance of \textsc{TextBO} is statistically significant different from that of \textsc{GEPA}. We consider the Two-way Fixed Effect (TWFE) model:
\begin{equation}
    \text{score(persona, scenario, t)} = \alpha + \beta_1 \mathbf{1}\{t=10\} + \beta_2 \mathbf{1}\{\mathrm{TextBO}\}\,\mathbf{1}\{t=10\} + \text{persona FE} + \text{scenario FE} + \varepsilon, \nonumber
\end{equation}
where $\mathbf{1}\{t=10\}$ is the indicator variable of whether it is $t=0$ (\textsc{GEPA} and \textsc{TextBO}'s initial starting point, i.e., \textsc{Worst5-of-64} score) or $t=10$ (\textsc{GEPA} and \textsc{TextBO}'s final score), $\beta_1$ is the \textsc{GEPA}'s gain from \textsc{Worst5-of-64} after $t=10$, $\beta_2$ is the \textsc{TextBO}'s advantage over GEPA after $t=10$, and $\beta_1+\beta_2$ is the \textsc{TextBO}'s gain from \textsc{Worst5-of-64} after $t=10$.  Further, we control for persona and scenario fixed effects to account for persistent differences in personas and scenarios (across both algorithms). The standard errors are two-way clustered  \citep{cameron2011robust, cameron2015practitioner}. 

The results from this TWFE model are shown in Table \ref{tab:twfe_matrix}. We see that both \textsc{TextBO}'s and \textsc{GEPA} show statistically significant gains from the starting point (\textsc{Worst5-of-64}) after ten iterations. In addition, at $T=10$, \textsc{TextBO} shows a 23.5\% more gain in score compared to \textsc{GEPA}. Together, these findings provide two key takeaways: (1) both \textsc{GEPA} and \textsc{TextBO} are able to self-improve and generate ads better aligned with the target population, and (2) the performance of \textsc{TextBO} is significantly better than that of \textsc{GEPA} for the same number of evaluations, and this difference is both statistically significant and meaningful in magnitude.  



\paragraph{Qualitative analysis:} So far, we have shown that \textsc{TextBO} is able to self-improve and generate ads that align well with the persona distribution within a few iterations. However, one could argue that this can be simply due to a weak set of initial prompts, i.e., a poor set of \textsc{Worst5-of-N} prompts and corresponding ad images. We examine this alternative explanation by examining the prompts and ad creatives/images.

First, one natural way to rule out this explanation is to compare the quality of the prompts at the starting point and end point of the algorithm, and show that the prompts at the starting point are grammatically correct, meaningful, and sensible and that the differences in the prompts at $t=0$ and $t=10$ stem from differences in the message and focus of the prompt rather than grammar/structure. We confirm that this in indeed the case by comparing the prompts at the starting and end points of the algorithm for all the scenarios and trajectories. For example, see the prompt for the \textsc{Worst-of-64} prompt for Scenario 1 in Figure \ref{fig:best_ad_prompt} in Web Appendix \ref{appssec:initial_prompts} and the prompt for \textsc{TextBO} at the last step in Figure \ref{fig:TBoN_ad_prompt} in Web Appendix \ref{appssec:final_tbonbo_gepa_prompts}.  

A second way to rule this out this alternative explanation is to compare the generated ad images and verify that the differences among ads across algorithms and (over the iterations, within a algorithm) stem from the message and representation, and not from the image quality. To that end, in Figure \ref{fig:ad_images}, we show the ad images at the start and end points for both \textsc{GEPA} and \textsc{TextBO}, as well the \textsc{Best-of-N} baseline (for three scenarios). There is no clear visual quality differences among these ad images; the differences are mainly in the ads' message and how it is represented. This further suggests that  the improvements from \textsc{TextBO} (and \textsc{GEPA}) over the \textsc{Worst5-of-N} starting point can be attributed to alignment, and are not due to trivial quality gains from correcting poor or under-specified initial prompts.


\begin{figure*}[ht]
  \centering
  \setlength{\tabcolsep}{3pt}
  \renewcommand{\arraystretch}{1.05}
  \footnotesize
  \resizebox{0.92\textwidth}{!}{%
  \begin{tabular}{M{0.11\textwidth} M{0.26\textwidth} M{0.26\textwidth} M{0.26\textwidth}}
    & \shortstack{GreenBite\\(Plant-based patty)} & \shortstack{MindGarden\\(Meditation app)} & \shortstack{AuraSonics X1\\(Noise-canceling earbuds)}
    \\
    \shortstack{\textsc{Worst5-of-64}\\(starting point)}
     &
      \shortstack{\includegraphics[width=\linewidth]{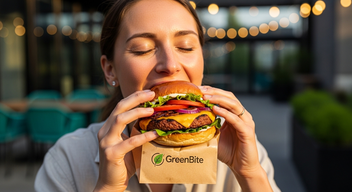}\\Score: 2.672} &
      \shortstack{\includegraphics[width=\linewidth]{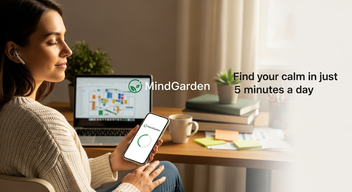}\\Score: 2.670} &
      \shortstack{\includegraphics[width=\linewidth]{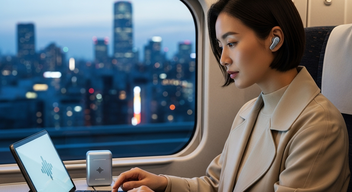}\\Score: 2.765} \\
    \textsc{Best-of-64} &
      \shortstack{\includegraphics[width=\linewidth]{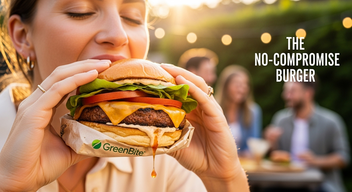}\\Score: 2.898} &
      \shortstack{\includegraphics[width=\linewidth]{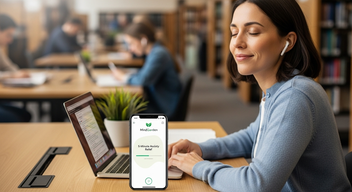}\\Score: 2.930} &
      \shortstack{\includegraphics[width=\linewidth]{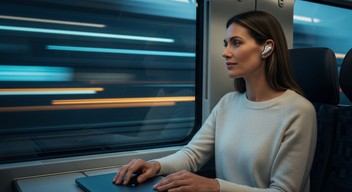}\\Score: 2.934} \\
    \shortstack{\textsc{GEPA}\\(after T=10)} &
      \shortstack{\includegraphics[width=\linewidth]{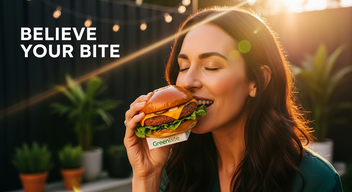}\\Score: 3.046} &
      \shortstack{\includegraphics[width=\linewidth]{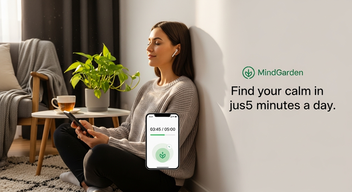}\\Score: 3.068} &
      \shortstack{\includegraphics[width=\linewidth]{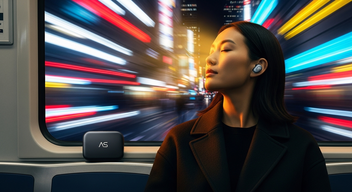}\\Score: 3.028} \\
     \shortstack{\textsc{TextBO}\\(after T=10)} &
      \shortstack{\includegraphics[width=\linewidth]{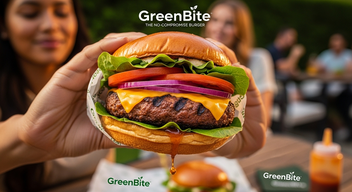}\\Score: 3.168} &
      \shortstack{\includegraphics[width=\linewidth]{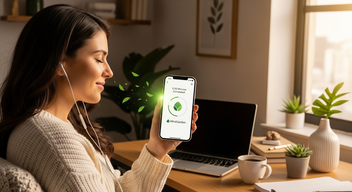}\\Score: 3.132} &
      \shortstack{\includegraphics[width=\linewidth]{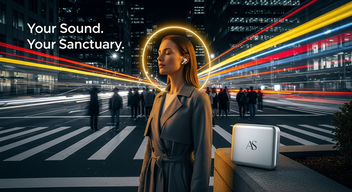}\\Score: 3.104} \\
  \end{tabular}}
  \caption{\textsc{TextBO} and baselines' generated ad images for three fictional brands: GreenBite burger patty (Scenario 1), MindGarden meditation app (Scenario 6), and AuraSonics noise-canceling earbuds (Scenario 2). For the ad images of the other five scenarios, see Figure \ref{fig:ad_images_more} in Web Appendix \ref{appssec:detailedExpResult}.}
  \label{fig:ad_images}
\end{figure*}

\begin{figure}[ht!]
    \centering
    \includegraphics[width=0.7\linewidth]{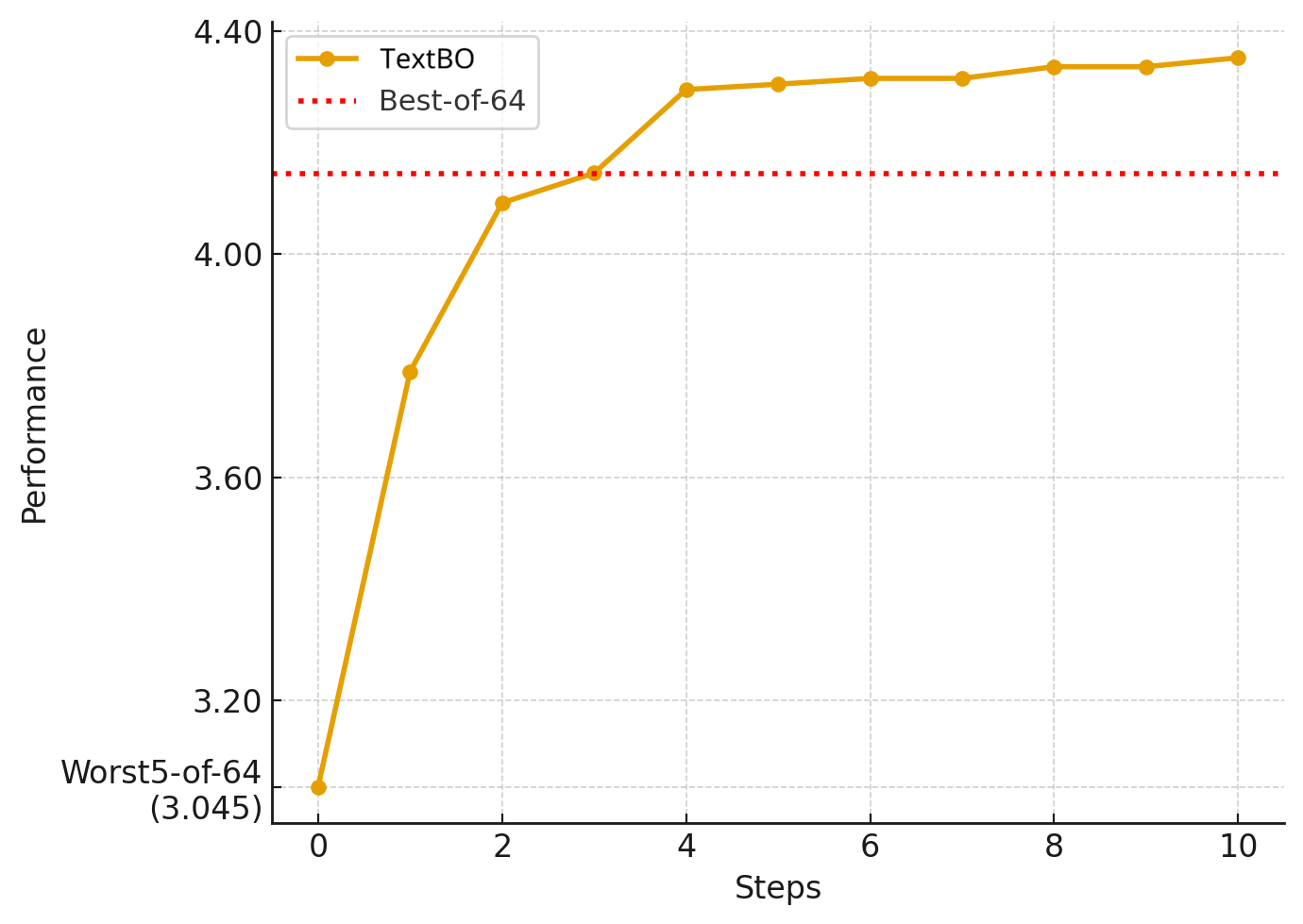}
    \caption{Ablation study with \textsc{TextBO} with $J=5$ and \textsc{Best-of-64} for Gemini 2.5 Flash as the evaluator, with temperature 0. The performance score is the average of eight scenarios' mean evaluation score at that step. }
    \label{fig:Flashperformance}
\end{figure}



\subsection{Ablation Study}
\label{ssec:ablation}

So far, in all our experiments, we used the persona dataset for evaluation and the goal was to optimize the ads serve to the target population captured by the persona dataset \citep{toubia2025twin}. However, a natural question here is whether \textsc{TextBO} can still perform well if we simply use a different target population/evaluator without access to extensive persona information during training. (Recall that each persona consists of a large amount of information, with 500 survey questions accounting for around $50{,}000$ tokens.) Therefore, we run a no-persona ablation experiment in which the evaluator is a single LLM (Gemini 2.5 Flash) with no auxiliary persona context. Concretely, at each iteration, we present only the ad image candidate to a LLM judge (Gemini 2.5 Flash)  together with the same 1–5 effectiveness rubric, set the temperature to 0, and compute the same scalar score as the log-probability–weighted expectation over $\{1,\dots,5\}$ as before.\footnote{With temperature $=0$, the judge’s output is deterministic for a fixed image and rubric; the logprob vector is still available, so the expected score is stable across repeated calls.} This collapses the evaluation to a deterministic mapping between ads and scores; it removes heterogeneity from personas and stochasticity from sampling, allowing us to test whether \textsc{TextBO}'s performance depends on specific persona information-related setup. To ensure that there is no leakage of preferences, we used GPT-5 instead of Gemini 2.5 Flash for critic model $M_{critic}$ (i.e., the operators used for generating meta-reflection, generating textual gradients, and making the \textsc{Best-of-N} choices.). 


We again compare \textsc{TextBO} ($J{=}5$ trajectories) against the \textsc{Best-of-64} initial baseline.  The results from this exercise are shown in Figure \ref{fig:Flashperformance} plots \textsc{TextBO} across 10 optimization steps starting from the \textsc{Worst5-of-64} prompts under the ablation setup.\footnote{Here, \textsc{Worst5-of-64} and \textsc{Best-of-64} are different from those of the main experiments, as the preference distribution is different.} 
As before, we see that \textsc{TextBO} outperforms \textsc{Best-of-64} baseline from optimization step 3 onwards, and the overall trend in Figure \ref{fig:Flashperformance} is similar to that in Figure \ref{fig:performance}. The main difference is that the evaluation scores without persona information are more optimistic than scores with persona information for both \textsc{TextBO} and \textsc{Best-of-64}. This is understandable, since \citet{peng2025mega} find that persona information does not significantly improve LLMs' overall ability to predict preferences and mainly adds variance to the predictions. Since variance makes the learning problem harder, it is natural that the main experiment, which uses personas, shows slower improvement than the pure-LLM case here. Overall, these findings suggest that when learning the target population's preferences is difficult, the method may require more iterations and larger training samples to optimize effectively.

\section{Experiments for Agentic AI Benchmarks}
\label{sec:agenticExp}

In $\S$\ref{sec:experiments}, we saw that \textsc{TextBO} significantly outperforms \textsc{GEPA} for ad optimization experiments. Recall that the adapted version of \textsc{GEPA} in the ads settings effectively reduced to \textsc{TextBO} with $N=1$ and $G=1$. As such, one of the main takeaways from those experiments is that \textsc{Best-of-N} gradient sampling and multiple gradient steps can effectively improve the performance of self-improving AI under limited evaluation budget. 

The natural follow-up question is whether \textsc{Best-of-N} gradient sampling and multiple gradient steps will also improve the performance of standard \textsc{GEPA} used in the agentic AI benchmarks considered in \citet{agrawal2025gepa}. Therefore, we now present a series of experiments, where we augment \textsc{GEPA} with the key theoretical and algorithmic innovations from our approach  -- \textsc{Best-of-N} gradient sampling and multiple gradient steps (as described in $\S$\ref{sec:BoNAndTheory}) . 
In our augmented version of \textsc{GEPA}, denoted as \textsc{TextBO-GEPA}, at each iteration, we take $G=5$ gradient steps, and at each step, we employ \textsc{Best-of-N} textual gradients with $N=5$. (In standard \textsc{GEPA}, each time we improve a prompt, we only take a single textual gradient step.) See Web Appendix $\S$\ref{appssec:agentic_GEPA_implementation} for details.



\subsection{Experiment Design}\label{ssec:agentExpDesign}

To rigorously evaluate \textsc{GEPA} and \textsc{TextBO-GEPA}, we closely replicate three agentic AI benchmarks and their corresponding experimental setups from \cite{agrawal2025gepa}.\footnote{\cite{agrawal2025gepa} also included IFBench \cite{ifbench_bench}, for which they report that none of the algorithm show improvement. Therefore we exclude it in our analysis.}  In all the experiments, we used Qwen3-8B \citep{qwen3_technical_report}, the same model and setup used in \cite{agrawal2025gepa} (decoding temperature of 0.6, top-p of 0.95, and top-k of 20 for training as well as inference). As reported in \cite{agrawal2025gepa}, \textsc{GEPA} saturates after $7,500$ evaluations (rollouts), so we conduct experiments up to $7,500$ evaluations. 

We now briefly describe the three benchmark tasks:\\
\noindent \textbf{HotpotQA}\; \citet{hotpotqa_bench} provides HotpotQA, a set of diverse, explainable multi-hop question answering tasks comprising approximately 113k Wikipedia-based examples. This task requires reasoning over multiple supporting documents and providing sentence-level supporting facts to enable intense supervision. We follow the implementation details described in \cite{agrawal2025gepa}, except for the retriever module, which was not specified in the paper. So we conservatively set it to BM25 \citep{robertson2009probabilistic}, one of the most widely used retrievers. We use 150 examples for training, 300 for validation, and 300 for testing.

\noindent \textbf{HoVer}\; \citet{hover_bench} provides HoVer, a set of many-hop fact-extraction and claim-verification tasks, constructed from the same 113k Wikipedia-based examples used in HotpotQA. It tests complex reasoning that requires aggregating evidence from multiple sentences across disparate documents, often involving diverse reasoning graphs and long-range dependencies. Again, we follow the implementation details described in \cite{agrawal2025gepa}, except for the retriever module, which was again not specified in the paper, and was set to BM25 \citep{robertson2009probabilistic}. Again, we use 150 examples for training, 300 for validation, and 300 for testing.

\noindent \textbf{PUPA}\; \citet{papillon_bench}  provides PUPA, a set of privacy-conscious delegation tasks which address real-world user queries by orchestrating an ensemble of trusted and untrusted models. The primary objective is to maintain a high response quality, comparable to that of proprietary models, while minimizing exposure of personally identifiable information (PII) to the untrusted backend. Again, we follow the implementation details described in \cite{agrawal2025gepa}, except that we use case-insensitive substring matching instead of LLM-as-a-judge to check whether each PII exists in the output. We use 111 training examples, 111 for validation, and 221 for testing.

\subsection{Results and Analysis}
\label{ssec:agentResultsAnalysis}

\begin{figure}[t]
  \centering
  \begin{subfigure}[t]{0.49\textwidth}
    \centering
    \includegraphics[width=\linewidth]{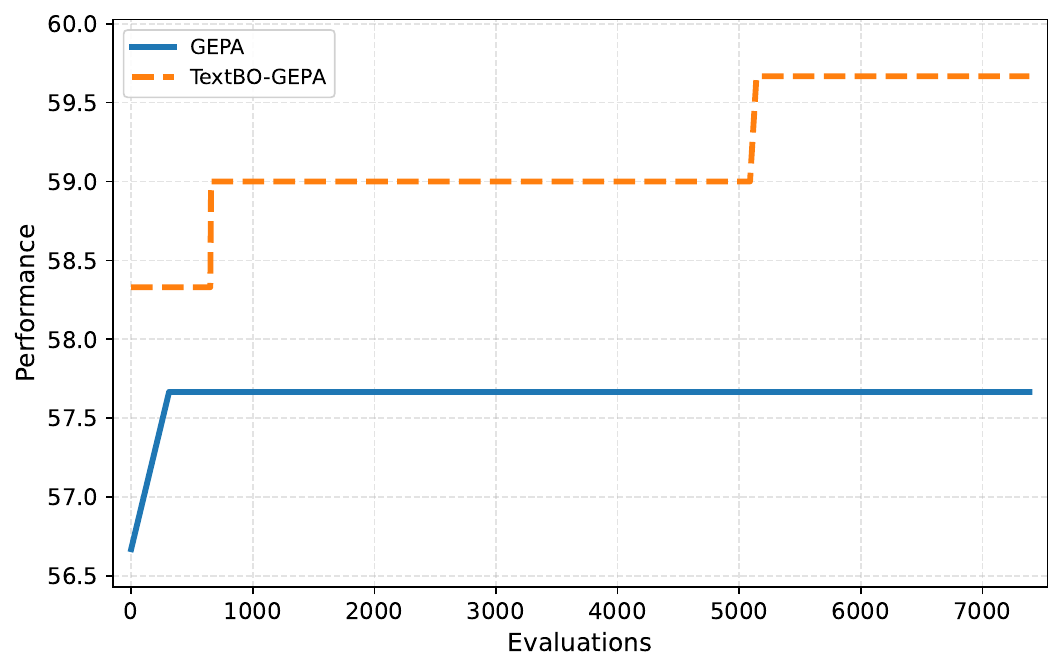}
    \caption{HotpotQA, Qwen3 8B}
    \label{fig:hotpot}
  \end{subfigure}\hfill
  \begin{subfigure}[t]{0.49\textwidth}
    \centering
    \includegraphics[width=\linewidth]{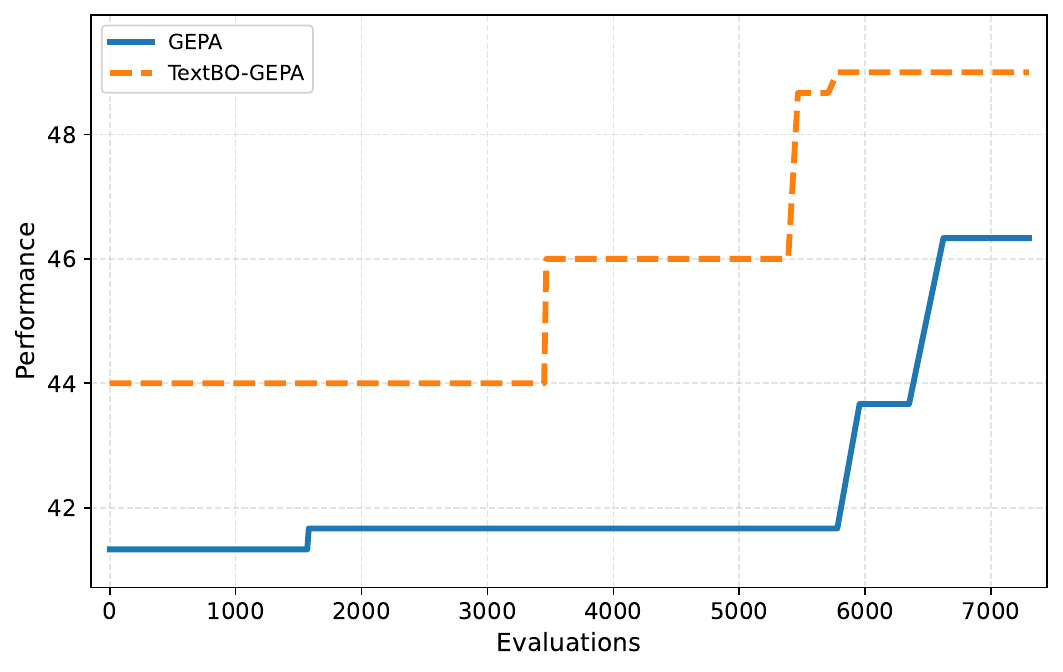}
    \caption{HoVer, Qwen3 8B}
    \label{fig:hover}
  \end{subfigure}

  \vspace{0.8em} 

  \begin{subfigure}[t]{0.49\textwidth}
    \centering
    \includegraphics[width=\linewidth]{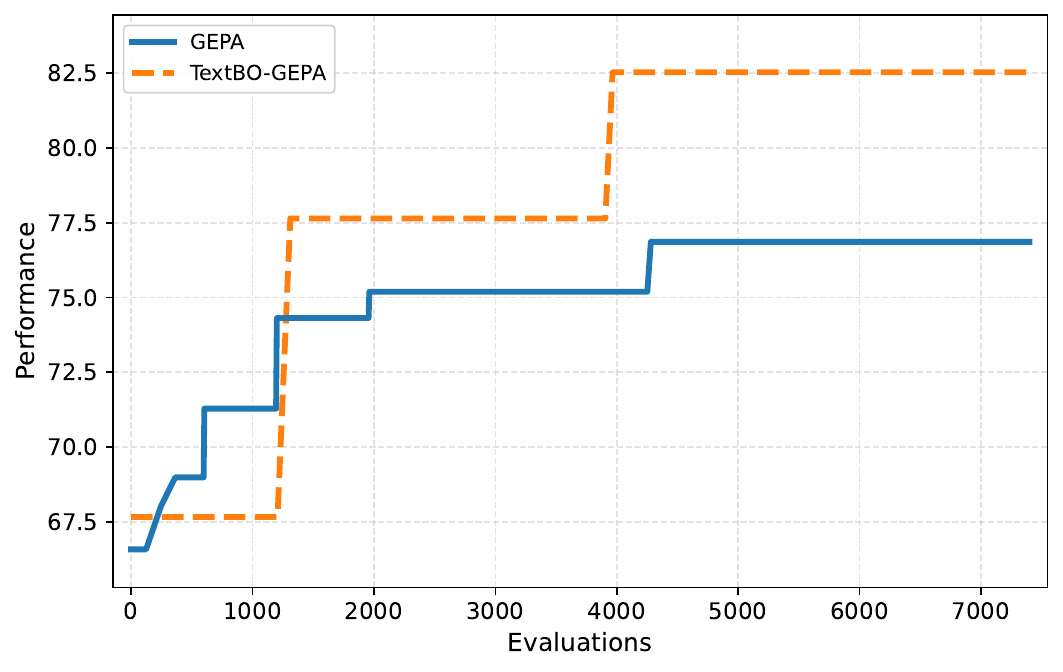}
    \caption{PUPA, Qwen3 8B}
    \label{fig:pupa}
  \end{subfigure}
  \caption{Experiment results for three agentic AI benchmarks considered in \cite{agrawal2025gepa}. As more evaluations are conducted, \textsc{TextBO-GEPA} (\textsc{GEPA} with \textsc{Best-of-N} gradient sampling and multiple gradient steps) consistently outperform vanilla \textsc{GEPA}.}
  \label{fig:agentic_results_graphs}
\end{figure}

Figure \ref{fig:agentic_results_graphs} and Table \ref{tab:agentic_result_tab} present the results of agentic AI benchmark experiments. We plot the improvement from their initial scores for both \textsc{TextBO-GPEA} and \textsc{GEPA}. Relative to the baseline, \textsc{TextBO-GEPA} delivers larger gains than \textsc{GEPA} across all datasets. On HotpotQA, the improvement rises from +1.00 to +3.00 points, which is a $200 \%$ increase in the baseline-relative gain. On Hover, the gain increases from +5.00 to +7.67 points $(+53.4 \%)$. On PUPA, it grows from +10.28 to +15.95 points $(+55.2 \%)$. Overall, the aggregate improvement increases from +5.43 to +8.87 points, corresponding to a $63.4 \%$ increase in baseline-relative improvement. Considering that language model evaluation benchmarks often have standard deviations well below 1 point, and 95\% CI sizes often around 1–2 points \citep{madaan2024quantifying},these results illustrate that \textsc{TextBO-GEPA} consistently significantly outperforms \textsc{GEPA} across the agentic AI benchmarks.

\begin{table}[htbp!]
  \centering
  \scalebox{0.9}{%
    \begin{tabular}{lccccc}
      \toprule
      \textbf{Model} & \textbf{HotpotQA} & \textbf{Hover} & \textbf{PUPA} & \textbf{Aggregate} & \textbf{Improvement} \\
      \midrule
      \multicolumn{6}{l}{\textbf{Qwen3-8B}} \\
      \midrule
      Baseline & 56.67 & 41.33 & 66.58 & 54.86 & --- \\
      \textsc{GEPA} & 57.67 & 46.33 & 76.86 & 60.29 & +5.43 \\
      \textsc{TextBO-GEPA} & \textbf{59.67} & \textbf{49.00} & \textbf{82.53} & \textbf{63.73} & \textbf{+8.87} \\
      \bottomrule
    \end{tabular}%
  }
  \caption{Results on HotpotQA, Hover, and PUPA for Qwen3-8B after ~7,500 evaluations. Aggregate is the mean across the three datasets; Improvement is vs.\ Baseline aggregate.}
  \label{tab:agentic_result_tab}
\end{table}

\section{Conclusion}
\label{sec:conclusion}

Evaluation-efficient self-improving AI is increasingly important in societal and business settings where real-world feedback (e.g., human responses, field experiments) is slow and costly, even as candidate generation and analysis have become cheap. In this paper, we address this gap by proposing \textsc{TextBO}, a simple evaluation-efficient prompt-optimization framework that combines TextGrad-style local textual updates with \textsc{Best-of-N} gradient selection. Our main theoretical result shows that \textsc{Best-of-N} over locally sampled textual edits induces, in probability, ascent directions of a UCB-style acquisition function. This provides a principled bridge between language-space self-improvement and classical UCB Bayesian optimization: \textsc{TextBO} emulates parallel gradient-based UCB-BO in the implicit embedding space induced by the LLM, while operating purely in language space and without constructing an explicit surrogate or calibrated uncertainty model. Empirically, we find that \textsc{TextBO} improves evaluation efficiency in ad-creative prompt optimization, and our ablations indicate that the gains are not driven by a particular persona dataset or by starting from unusually weak prompts. We additionally find that the same \textsc{Best‑of‑N} multi‑step mechanism strengthens \textsc{GEPA} in agentic benchmark settings: \textsc{TextBO‑GEPA} exhibits larger improvements than \textsc{GEPA} across evaluation steps on agentic AI benchmarks.

Overall, this work makes four contributions to literature on AI, optimization, and marketing. First, we elevate \emph{evaluation efficiency} as a first-order design objective for self-improving AI deployed in real-world decisions. Second, we introduce \textsc{TextBO}, a simple and practical language-space optimization procedure that combines textual gradients with \textsc{Best-of-N} selection. Third, we provide theory showing that \textsc{Best-of-N} textual-gradient steps implement optimism-driven UCB-style exploration in an implicit embedding space, yielding evaluation-efficiency guarantees. Fourth, we demonstrate empirical gains relative to strong baselines in both the digital advertising case study and the agentic AI benchmarks.

From a managerial perspective, evaluation-efficient self-improving AI can sharply reduce the marginal cost of iteration in data-driven decision making and enable even small businesses and resource constrained organizations to adopt AI tools (in ad optimization, promotion design, and other marketing activities). As generation and analysis become automated, our framework offers firms the tools and ability to sustain continuous optimization loops that adapt campaigns to evolving market conditions and consumer responses.

\singlespacing
\bibliographystyle{plainnat}
\bibliography{references} 

\newpage
\begin{appendices}

\setcounter{table}{0}
\setcounter{figure}{0}
\setcounter{equation}{0}
\setcounter{page}{0}
\renewcommand{\thetable}{A\arabic{table}}
\renewcommand{\thefigure}{A\arabic{figure}}
\renewcommand{\theequation}{A\arabic{equation}}
\renewcommand{\thepage}{\roman{page}}
\pagenumbering{roman}

\input{GP-UCB}

\clearpage
\section{Proofs}
\label{sec:proof}

\subsection{Auxiliary lemmas}

\begin{lemma}[Uniform first-order expansion]
\label{lem:taylor}
Let $g=\nabla\mu(e)$ and $h=\nabla\sigma(e)$. There exist $C>0$ and $\varepsilon_0>0$ such that for all $\varepsilon\in(0,\varepsilon_0]$ and all $u\in\mathbb{R}^{d}$,
\begin{align}
\mu(e+\varepsilon u) &= \mu(e) + \varepsilon\, g^\top u \;+\; r_\mu(\varepsilon,u), & |r_\mu(\varepsilon,u)| &\le C \varepsilon^2, \label{eq:taylor-mu}\\
\sigma(e+\varepsilon u) &= \sigma(e) + \varepsilon\, h^\top u \;+\; r_\sigma(\varepsilon,u), & |r_\sigma(\varepsilon,u)| &\le C \varepsilon^2. \label{eq:taylor-sig}
\end{align}
\end{lemma}

\begin{proof}
By the mean-value theorem with Lipschitz gradients, for some $\xi$ between $e$ and $e+\varepsilon u$,
\(
\mu(e+\varepsilon u)=\mu(e)+\nabla\mu(\xi)^\top(\varepsilon u).
\)
Hence
\(
\mu(e+\varepsilon u)=\mu(e)+\varepsilon g^\top u + \varepsilon (\nabla\mu(\xi)-\nabla\mu(e))^\top u,
\)
and $|\nabla\mu(\xi)-\nabla\mu(e)|\le L_\mu \|\xi-e\| \le L_\mu \varepsilon$, giving $|r_\mu|\le L_\mu \varepsilon^2$. The proof for $\sigma$ is identical with $L_\sigma$. Taking $C:=\max\{L_\mu,L_\sigma\}$ yields the stated bounds with the same $C$.
\end{proof}

\begin{lemma}[Max-stability under bounded perturbations]
\label{lem:argmax-stability}
Let $(f_i)_{i\le N}$ and $(g_i)_{i\le N}$ be real arrays. Let $i_f\in\arg\max_i f_i$ and $i_g\in\arg\max_i g_i$. If for some $\Delta\ge 0$,
\begin{align}
    g_{i_f} \ge f_{i_f} - \Delta
\quad\text{and}\quad
f_{i_g} \ge g_{i_g} - \Delta,
\end{align}
then $g_{i_f} \ge \max_i g_i - \Delta$ and $f_{i_g} \ge \max_i f_i - \Delta$. In particular, $g_{i_f} \ge \max_i g_i - \Delta$ implies that any maximizer of $f$ is a $\Delta$-approximate maximizer of $g$.
\end{lemma}

\begin{proof}
Since $g_{i_f} \ge f_{i_f} - \Delta = \max_i f_i - \Delta \ge g_{i_g} - \Delta = \max_i g_i - \Delta$, the first claim holds. The second is symmetric.
\end{proof}

\begin{lemma}[Spherical cap coverage]
\label{lem:spherical-cap}  For $v\in\mathbb{S}^{d-1}$ and $\eta\in(0,1)$,
define $C(v,\eta):=\{u\in\mathbb{S}^{d-1}: v^\top u \ge 1-\eta\}$. Let $U_1,\ldots,U_N\stackrel{\text{i.i.d.}}{\sim}\rho$ as above. Then $\max_{i\le N} v^\top U_i \to 1$ almost surely as $N\to\infty$. Moreover, for any deterministic sequence $\eta_N\downarrow 0$ with $N \cdot \rho\{u: v^\top u \ge 1-\eta_N\}\to\infty$, we have
\begin{align}
    \mathbb{P}\!\left(\max_{i\le N} v^\top U_i \ge 1-\eta_N\right)\to 1.
\end{align}
\end{lemma}

\begin{proof}
Let $C(\eta)=\{u\in\mathbb{S}^{d-1}: v^\top u \ge 1-\eta\}$ be a spherical cap. Since $\rho$ has a density bounded below, $\rho(C(\eta)) \asymp \eta^{\frac{d-1}{2}}$ as $\eta\downarrow 0$. Then
\(
\mathbb{P}(\max_i v^\top U_i < 1-\eta)
= (1-\rho(C(\eta)))^N \le \exp(-N\rho(C(\eta))).
\)
For any fixed $\eta\in(0,1)$, $\sum_{N\ge1}\mathbb{P}(\max_i v^\top U_i < 1-\eta)
\le \sum_{N\ge1} e^{-N\rho(C(\eta))} < \infty$,
so by Borel–Cantelli, $\mathbb{P}(\max_i v^\top U_i < 1-\eta\ \text{i.o.})=0$.
Intersecting over a countable sequence $\eta_m\downarrow 0$ yields
$\max_{i\le N} v^\top U_i \to 1$ almost surely. Choosing $\eta_N\downarrow 0$ so that $N\rho(C(\eta_N))\to\infty$ yields the result, and taking $\eta=\eta_N\to 0$ proves the first statement.
\end{proof}

\begin{lemma}[Near-maximum coupling in a thin band]
\label{lem:thin-band}
Let $M_N=\max_i \xi_i$ and define $v_N:=(g+M_N h)/\|g+M_N h\|$. Also, $S_i:=a_i+b_iM_N$, $Y_i^{(0)}:=a_i+b_i\xi_i$, $b_i:=\sigma(e)+\varepsilon h^\top U_i$. Fix any sequence $c_N\uparrow\infty$ with $c_N/q_N\to 0$ and set $\delta_N:=c_N/q_N$.
Then with probability $\to 1$ there exists an index $j$ such that
\begin{align}
    \xi_j \ge M_N-\delta_N
\quad\text{and}\quad
v_N^\top U_j \ge 1-\eta_N,
\end{align}
whenever $\eta_N\downarrow 0$ satisfies
\begin{align}
    N\,\mathbb{P}(\xi\ge M_N-\delta_N)\,\rho\{u: v_N^\top u\ge 1-\eta_N\}\to\infty.
\end{align}
Moreover, for this $j$,
\begin{align}
    S_j-Y_j^{(0)}=b_j(M_N-\xi_j)\le b_{\max}\,\delta_N.
\end{align}
\end{lemma}

\begin{proof}
Define $p_{\mathrm{cap},N}:=\rho\{u: v_N^\top u\ge 1-\eta_N\}$ and the event
\begin{align}
    E_N:=\{M_N \ge q_N-\delta_N\}.
\end{align}
By Lemma~\ref{lem:maxima}, $M_N-q_N\to 0$ in probability and $\delta_N\to 0$, hence $\mathbb{P}(E_N)\to 1$.

Set the deterministic threshold $t_N:=q_N-2\delta_N$. On $E_N$ we have $M_N-\delta_N\ge q_N-2\delta_N=t_N$, hence by monotonicity of the tail,
\begin{equation}
\label{eq:tail-monotone}
\mathbb{P}(\xi\ge t_N)\;\ge\;\mathbb{P}(\xi\ge M_N-\delta_N)\quad\text{on }E_N.
\end{equation}
Let
\begin{align}
    K_N:=\sum_{i=1}^N \mathbf{1}\{\xi_i\ge t_N\}\sim\mathrm{Binomial}\!\big(N,\,p_N\big),
\qquad p_N:=\mathbb{P}(\xi\ge t_N).
\end{align}
A Chernoff bound gives
\begin{equation}
\label{eq:Chernoff-K}
\mathbb{P}\!\left(K_N<\tfrac12 Np_N\right)\;\le\;\exp\!\left(-\tfrac18 Np_N\right).
\end{equation}
Consider
\begin{align}
    L_N\;:=\;\sum_{i=1}^N \mathbf{1}\{\xi_i\ge t_N\}\,\mathbf{1}\{v_N^\top U_i\ge 1-\eta_N\}.
\end{align}
Conditional on $\{\xi_i\}_{i=1}^N$ (and hence on $K_N$ and $v_N$), the variables $\{\mathbf{1}\{v_N^\top U_i\ge 1-\eta_N\}\}_{i=1}^N$ are i.i.d.\ Bernoulli$(p_{\mathrm{cap},N})$ and independent of $\{\xi_i\}$. Therefore,
\begin{equation}
\label{eq:LN-binomial}
L_N\,\big|\,\{\xi_i\}\;\sim\;\mathrm{Binomial}\!\big(K_N,\,p_{\mathrm{cap},N}\big),
\qquad
\mathbb{P}\big(L_N=0\,\big|\,\{\xi_i\}\big)\;=\;(1-p_{\mathrm{cap},N})^{K_N}\;\le\;\exp\!\big(-K_N p_{\mathrm{cap},N}\big).
\end{equation}

We now bound $\mathbb{P}(L_N=0)$ by splitting on $\{K_N\ge \tfrac12 Np_N\}$ and $E_N$:
\begin{align}
\mathbb{P}(L_N=0)
&\le \mathbb{E}\!\left[\exp\!\big(-K_N p_{\mathrm{cap},N}\big)\,\mathbf{1}_{\{K_N\ge \frac12 Np_N\}\cap E_N}\right]
\;+\;\mathbb{P}\!\left(K_N<\tfrac12 Np_N\right)\;+\;\mathbb{P}(E_N^c) \nonumber\\
&\le \mathbb{E}\!\left[\exp\!\big(-\tfrac12 Np_N\,p_{\mathrm{cap},N}\big)\,\mathbf{1}_{E_N}\right]
\;+\;\exp\!\big(-\tfrac18 Np_N\big)\;+\;o(1) \nonumber\\
&\le \mathbb{E}\!\left[\exp\!\big(-\tfrac12 Np_N\,p_{\mathrm{cap},N}\big)\right]
\;+\;\exp\!\big(-\tfrac18 Np_N\big)\;+\;o(1). \label{eq:three-terms-corrected}
\end{align}
On $E_N$, \eqref{eq:tail-monotone} yields $p_N \ge \mathbb{P}(\xi\ge M_N-\delta_N)$. By the lemma’s hypothesis,
\[
Z_N:=N p_N\,p_{\mathrm{cap},N} \;\ge\; N\,\mathbb{P}(\xi\ge M_N-\delta_N)\,p_{\mathrm{cap},N}\;\xrightarrow{p}\;\infty.
\]
Since $0\le e^{-\frac12 Z_N}\le 1$, for any fixed $T>0$,
\[
\mathbb{E}\!\left[e^{-\frac12 Z_N}\right]
\le e^{-\frac12 T}+\mathbb{P}(Z_N\le T)\;\longrightarrow\; e^{-\frac12 T}.
\]
Letting $T\to\infty$ gives $\mathbb{E}[e^{-\frac12 Z_N}]\to 0$. Moreover, $Z_N\le Np_N$ since $p_{\mathrm{cap},N}\le 1$. If $Np_N$ were bounded along a subsequence, then $Z_N$ would be bounded along that subsequence, contradicting $Z_N\xrightarrow{p}\infty$.
Hence $Np_N\to\infty$, so $\exp(-\tfrac18 Np_N)\to 0$. Therefore the second term in \eqref{eq:three-terms-corrected} vanishes, and $\mathbb{P}(L_N=0)\to 0$.

Therefore, with probability tending to one there exists $j$ such that $\xi_j\ge t_N$ and $v_N^\top U_j\ge 1-\eta_N$. Since $t_N\le M_N-\delta_N$ on $E_N$, this $j$ also satisfies $\xi_j\ge M_N-\delta_N$. Finally,
\(
S_j-Y_j^{(0)}=b_j(M_N-\xi_j)\le b_{\max}\delta_N
\)
holds deterministically by $M_N-\xi_j\le\delta_N$ and the definition of $b_{\max}$.
\end{proof}

\subsection{Proof of Theorem \ref{thm:main-ucb}}

\begin{proof}\;
\\
By Lemma~\ref{lem:taylor},
\begin{align}
Y_i
&= \mu(e)+\varepsilon g^\top u_i + r_\mu(\varepsilon,u_i)
+ \big(\sigma(e)+\varepsilon h^\top u_i + r_\sigma(\varepsilon,u_i)\big) \xi_i \nonumber\\
&= a_i + b_i \xi_i + R_i, \label{eq:Y-decomp}
\end{align}
where $a_i:=\mu(e)+\varepsilon g^\top u_i$, $b_i:=\sigma(e)+\varepsilon h^\top u_i$, and
\(
R_i:= r_\mu(\varepsilon,u_i) + \xi_i r_\sigma(\varepsilon,u_i).
\)
Then $|R_i|\le C\varepsilon^2 (1+|\xi_i|)$ for all $i$.
Define \(S_i:=a_i+b_i M_N\), \(b_{\max}:=\max_{1\le i\le N}|b_i|\), and \(U_i:=u_i\).

\noindent Now let $Y_i^{(0)}:=a_i+b_i \xi_i$ and $i_0\in\arg\max_i Y_i^{(0)}$. From \eqref{eq:Y-decomp},
\[
Y_{i^\star} \ge Y_{i_0} \implies
Y_{i^\star}^{(0)} \ge Y_{i_0}^{(0)} - |R_{i^\star}| - |R_{i_0}|.
\]
By the bound on $R_i$ and $|\xi_i|\le M_N:=\max_j \xi_j$,
\begin{equation}
\label{eq:delta-eps}
Y_{i^\star}^{(0)} \ge \max_i Y_i^{(0)} - 2C\varepsilon^2(1+M_N).
\end{equation}
Symmetrically, $\max_i Y_i^{(0)} \ge Y_{i^\star}^{(0)} \ge \max_i Y_i^{(0)} - 2C\varepsilon^2(1+M_N)$.
By Lemma~\ref{lem:argmax-stability}, $i^\star$ is a $2C\varepsilon^2(1+M_N)$-approximate maximizer of $Y^{(0)}$.

\noindent Now let $j$ be the index from Lemma~\ref{lem:thin-band} so that $\xi_j\ge M_N-\delta_N$ and $v_N^\top U_j\ge 1-\eta_N$.
By Lemma~\ref{lem:taylor},
\[
\max_{i\le N} S_i - S_j
\;\le\;
\varepsilon \|g+M_N h\|\,(1- v_N^\top U_j)\;+\;2C\varepsilon^2(1+M_N)
\;\le\;
\varepsilon \|g+M_N h\|\,\eta_N\;+\;2C\varepsilon^2(1+M_N).
\]
Also,
\[
S_j - S_{i^\star}
= (S_j-Y_j^{(0)}) + (Y_j^{(0)}-Y_{i^\star}^{(0)}) + (Y_{i^\star}^{(0)}-S_{i^\star})
\;\le\;
b_{\max}\delta_N \;+\; 2C\varepsilon^2(1+M_N)\;+\;\underbrace{(Y_{i^\star}^{(0)}-S_{i^\star})}_{\le 0},
\]
where the last term is $\le 0$ provided $b_i\ge 0$ for all $i$, which holds by Assumption~A4 for $\varepsilon\le \varepsilon_0:=\sigma(e)/(2\|h\|)$ when $\sigma(e)>0$ since
$b_i=\sigma(e)+\varepsilon h^\top U_i \ge \sigma(e)-\varepsilon\|h\|\ge \tfrac12\sigma(e)>0$.
Therefore,
\begin{equation}
\label{eq:maxS-gap}
\max_{i\le N} S_i - S_{i^\star}
\;\le\;
\varepsilon \|g+M_N h\|\,\eta_N \;+\; b_{\max}\delta_N \;+\; 4C\varepsilon^2(1+M_N).
\end{equation}

\noindent
Set:
\[
\Delta'_N(\varepsilon)\;:=\; b_{\max}\,\frac{c_N}{q_N} \;+\; 4C\varepsilon^2(1+M_N).
\]

\noindent 
By Lemma~\ref{lem:taylor},
\[
S_i \;=\; \mu(e) + M_N \sigma(e) + \varepsilon (g+M_N h)^\top u_i + \tilde r_i,
\qquad |\tilde r_i|\le C\varepsilon^2(1+M_N).
\]
Hence for any \(i\),
\[
S_i - S_{i^\star}
= \varepsilon \|g+M_N h\|\,(v_N^\top U_i - v_N^\top \widehat u_N) + (\tilde r_i - \tilde r_{i^\star}),
\]
so
\begin{equation}
\label{eq:two-sided}
\big|\, (S_i - S_{i^\star}) - \varepsilon \|g+M_N h\|\,(v_N^\top U_i - v_N^\top \widehat u_N)\,\big|
\;\le\; 2C\varepsilon^2(1+M_N).
\end{equation}
Taking \(i\) that attains \(\max_i S_i\) and using \eqref{eq:maxS-gap} yields
\begin{equation}
\label{eq:gap-bound-final}
\max_{i\le N} v_N^\top U_i - v_N^\top \widehat u_N
\;\le\;
\eta_N
\;+\;
\frac{\Delta'_N(\varepsilon)}{\varepsilon \|g+M_N h\|}.
\end{equation}
Thus
\begin{equation}
\label{eq:angle-split-final}
1 - v_N^\top \widehat u_N
\;\le\;
\underbrace{1-\max_{i\le N} v_N^\top U_i}_{\to 0\ \text{a.s.\ by Lemma \ref{lem:spherical-cap}}}
\;+\;
\eta_N
\;+\;
\frac{\Delta'_N(\varepsilon)}{\varepsilon \|g+M_N h\|}.
\end{equation}

Assume $\|h\|>0$ (the $\|h\|=0$ case is addressed in Remark~\ref{rem:edge}). Using
\[
\|g+M_N h\| \;\ge\; M_N\|h\| - \|g\|
\]
and Lemma~\ref{lem:maxima}, $M_N\to\infty$ and $M_N\asymp q_N$. Also $b_{\max} \le \sigma(e) + \varepsilon \|h\| + C \varepsilon^2$. Hence
\[
\frac{\Delta'_N(\varepsilon)}{\varepsilon\|g+M_N h\|}
= O_p\!\Big(\tfrac{c_N}{\varepsilon q_N^2\|h\|}\Big)+O\!\Big(\tfrac{\varepsilon}{\|h\|}\Big).
\]
Choosing, e.g., $c_N=\tfrac12\log\log N$ and any $\eta_N\downarrow 0$ with
\[
N\,e^{c_N}\,\rho\{u: v_N^\top u \ge 1-\eta_N\}\to\infty
\]
makes the RHS of \eqref{eq:angle-split-final} equal to $o_p(1)+O(\varepsilon)$ for fixed $\varepsilon>0$; then let $\varepsilon\downarrow 0$.

\noindent 
With $N\to\infty$ first and fixed $\varepsilon>0$, \eqref{eq:angle-split-final} implies $1 - v_N^\top \widehat u_N \xrightarrow{p} 0$. Then $\varepsilon\downarrow 0$ removes the $O(\varepsilon)$ residual, yielding $v_N^\top \widehat u_N \xrightarrow{p} 1$. Since $M_N-q_N\to 0$ in probability by Lemma~\ref{lem:maxima}, the unit vectors
\(
v_N=(g+M_N h)/\|g+M_N h\|
\)
and
\(
\tilde v_N:=(g+q_N h)/\|g+q_N h\|
\)
satisfy $\|v_N-\tilde v_N\|\to 0$ in probability. Hence $\widehat u_N \xrightarrow{p} \tilde v_N$, which is the claimed gradient direction of $A_{\beta_N}$ with $\beta_N=q_N$.

\noindent  
By Lemma~\ref{lem:taylor},
\[
A_{\beta_N}(e+\Delta e^{(N)}) = A_{\beta_N}(e) + \varepsilon \nabla A_{\beta_N}(e)^\top \widehat u_N + O(\varepsilon^2).
\]
Write $\nabla A_{\beta_N}(e) = g+q_N h$. Using $\tilde v_N^\top \widehat u_N \xrightarrow{p} 1$ gives
\(
\nabla A_{\beta_N}(e)^\top \widehat u_N
= \|\nabla A_{\beta_N}(e)\|\, \tilde v_N^\top \widehat u_N
= \|\nabla A_{\beta_N}(e)\|(1 - o_p(1)).
\)
Thus
\(
A_{\beta_N}(e+\Delta e^{(N)})
\ge A_{\beta_N}(e) + \varepsilon \|\nabla A_{\beta_N}(e)\| - o_p(\varepsilon),
\)
as claimed.
\end{proof}

\begin{remark}[Edge cases]\label{rem:edge} If $\|h\|=0$, then $\nabla A_{\beta_N}(e)=g$ and the direction reduces to $\nabla\mu(e)$. The Bayesian Optimization setup does not allow $\sigma(e)=0$. 
\end{remark}

\clearpage
\section{Appendix for Section~\ref{sec:case}}

\subsection{Creative Briefs for the Scenarios}
\label{ssec:scenarios}

Below are the creative briefs for the eight ad campaigns. Each creative brief outlines the strategic and creative direction for each scenario. The selected campaigns cover a diverse range of products across distinct categories. Each campaign addresses a unique consumer need—whether related to sustainability, luxury, technology, or lifestyle—ensuring comprehensive coverage of various market segments. From environmentally-conscious millennials to affluent, experience-driven travelers, the approach accounts for the full spectrum of consumer interests. By focusing on different consumer needs, the structure enables a holistic marketing strategy that captures both niche and broad demands, providing a well-rounded framework for the study.

\subsubsection{Scenario 1: GreenBite Plant-Based Burger}
\squishlist
    \item {\it Product:} "GreenBite," a new plant-based burger patty.
    \item {\it Background:} The market for plant-based food is growing, but many consumers remain skeptical about taste, believing they must sacrifice flavor for health.
    \item {\it The Challenge:} Convince flexitarians that GreenBite offers the juicy, satisfying experience of a traditional beef burger with no compromise.
    \item {\it Core Insight:} Consumers want to eat better for themselves and the planet, but they fear it means giving up the foods they love.
    \item {\it Single-Minded Proposition (SMP):} The delicious, no-compromise burger experience that's better for you and the planet.
    \item {\it Reasons to Believe (RTB):} Made with savory pea protein; sears and tastes like real beef; free of soy and GMOs.
    \item {\it Desired Response:} {\it Think}: "I can finally have a plant-based burger that tastes like the real thing.", {\it Feel:} Satisfied, vibrant, and guilt-free. {\it Do:} Purchase GreenBite for their next barbecue.
    \item {\it Brand Personality \& Tone:} Positive, confident, and mouth-watering.
\squishend

\subsubsection{Scenario 2: AuraSonics X1 Earbuds}
\squishlist
    \item {\it Product:} "AuraSonics X1," high-end, noise-canceling wireless earbuds.
    \item {\it Background:} The rise of hybrid work and urban density has increased daily auditory overload, making focus a precious commodity.
    \item {\it The Challenge:} Cut through the saturated audio market by positioning the X1 not as a gadget, but as an essential tool for mental clarity.
    \item {\it Core Insight:} In a world of constant noise, true luxury is the ability to control your own soundscape.
    \item {\it Single-Minded Proposition (SMP):} AuraSonics X1 creates a personal sanctuary for focus and immersion.
    \item {\it Reasons to Believe (RTB):} Adaptive Active Noise Cancellation; studio-quality audio; all-day battery life; minimalist aluminum design.
    \item {\it Desired Response:} {\it Think:} "This is the solution I need to escape the daily chaos." {\it Feel:} Calm, empowered, and sophisticated. {\it Do:} Visit the website to learn more.
    \item {\it Brand Personality \& Tone:} Minimalist, intelligent, and calm.
\squishend

\subsubsection{Scenario 3: Odyssey E-SUV}
\squishlist
    \item {\it Product:} The "Odyssey E-SUV," a new all-electric family SUV.
    \item {\it Background:} Many families want to switch to electric vehicles but are concerned about sacrificing space, safety, and range for sustainability.
    \item {\it The Challenge:} Position the Odyssey E-SUV as the first EV that meets all the practical and adventurous needs of a modern family, making the switch to electric feel like an upgrade, not a compromise.
    \item {\it Single-Minded Proposition (SMP):} The future of family adventure is here: all-electric, zero compromise.
    \item {\it Reasons to Believe (RTB):} 500km range; top-tier safety rating; spacious three-row seating; all-wheel drive capability.
    \item {\it Desired Response:} {\it Think:} "An electric car can actually handle my family's active lifestyle." {\it Feel:} Adventurous, optimistic, and secure. {\it Do:} Schedule a test drive.
    \item {\it Brand Personality \& Tone:} Inspiring, capable, and forward-thinking.
\squishend

\subsubsection{Scenario 4: Oasis Eco-Lodge}
\squishlist
    \item {\it Product:} "Oasis Eco-Lodge," a secluded, luxury resort operating in harmony with its natural surroundings.
    \item {\it Background:} The luxury travel market often equates opulence with excess. A growing segment of affluent travelers seeks experiences that are both exclusive and responsible.
    \item {\it The Challenge:} Redefine luxury as a seamless integration with nature, promising an experience that is restorative for both the guest and the environment.
    \item {\it Core Insight:} For those who have everything, true luxury is not more things, but a deeper connection to something pure and serene.
    \item {\it Single-Minded Proposition (SMP):} Rediscover tranquility in a luxury that respects nature.
    \item {\it Reasons to Believe (RTB):} Secluded private bungalows; farm-to-table dining with locally sourced ingredients; carbon-neutral operations.
    \item {\it Desired Response:} {\it Think:} "This is a truly special escape, not just another five-star hotel." {\it Feel:} Serene, exclusive, and rejuvenated. {\it Do:} Book a stay.
    \item {\it Brand Personality \& Tone:} Elegant, peaceful, and understated.
\squishend

\subsubsection{Scenario 5: Momentum Digital Banking}
\squishlist
    \item {\it Product:} "Momentum," a mobile-first banking app for freelancers and the gig economy.
    \item {\it Background:} Traditional banks are not built for the fluctuating incomes and unique business needs of self-employed professionals.
    \item {\it The Challenge:} Position Momentum as the essential financial co-pilot for the independent worker, simplifying complexity and providing stability.
    \item {\it Core Insight:} Freelancers love their freedom but feel anxious about their financial instability. They need a tool that brings order to their financial chaos.
    \item {\it Single-Minded Proposition (SMP):} Banking that is as flexible and entrepreneurial as you are.
    \item {\it Reasons to Believe (RTB):} Automated tax-saving tools; integrated invoicing and payment tracking; instant business expense categorization.
    \item {\it Desired Response:} {\it Think:} "Finally, a bank that gets the way I work." {\it Feel:} Empowered, organized, and financially confident. {\it Do:} Download the app and open an account.
    \item {\it Brand Personality \& Tone:} Modern, empowering, and simple.
\squishend

\subsubsection{Scenario 6: MindGarden Meditation App}
\squishlist
    \item {\it Product:} "MindGarden," a subscription-based meditation and mindfulness app.
    \item {\it Background:} While many are interested in mindfulness for stress-relief, they are often intimidated by the practice, seeing it as difficult or time-consuming.
    \item {\it The Challenge:} Make meditation feel accessible and achievable for absolute beginners, removing the barriers of time and perceived difficulty.
    \item {\it Core Insight:} People want the benefits of mindfulness but are convinced they don't have the time or ability to practice it correctly.
    \item {\it Single-Minded Proposition (SMP):} Find your calm in just 5 minutes a day.
    \item {\it Reasons to Believe (RTB):} Guided 5-minute sessions for specific needs (anxiety, focus); progress tracking to build a habit; simple, jargon-free instructions.
    \item {\it Desired Response:} {\it Think:} "I can do this. 5 minutes is easy." {\it Feel:} Calm, supported, and hopeful. {\it Do:} Start a free trial.
    \item {\it Brand Personality \& Tone:} Gentle, approachable, and encouraging.
\squishend

\subsubsection{Scenario 7: Aeterno Watch}
\squishlist
    \item {\it Product:} The "Aeterno," a classic, automatic Swiss-made wristwatch with a heritage design.
    \item {\it Background:} In an age of smartwatches that are obsolete in a few years, there is a renewed appreciation for objects with permanence and enduring value.
    \item {\it The Challenge:} Reassert the relevance of the classic mechanical watch as a symbol of taste, craftsmanship, and timelessness.
    \item {\it Core Insight:} In a disposable world, true status comes from owning something permanent that tells a story.
    \item {\it Single-Minded Proposition (SMP):} A legacy on your wrist. Craftsmanship that transcends time.
    \item {\it Reasons to Believe (RTB):} Swiss-made automatic movement; sapphire crystal glass; timeless design inspired by 1950s classics.
    \item {\it Desired Response:} {\it Think:} "This is a beautiful object that I will own forever." {\it Feel:} Prestigious, sophisticated, and discerning. {\it Do:} Locate an authorized dealer.
    \item {\it Brand Personality \& Tone:} Elegant, timeless, and confident.
\squishend

\subsubsection{Scenario 8: SyncFlow B2B Software}
\squishlist
    \item {\it Product:} "SyncFlow," a project management and collaboration software platform for remote teams.
    \item {\it Background:} Remote work has increased flexibility but also created challenges in team alignment, communication, and project visibility.
    \item {\it The Challenge:} Cut through a crowded SaaS market by focusing on the outcome of "effortless collaboration" rather than just listing features.
    \item {\it Core Insight:} Managers don't want another tool to manage; they want the feeling of clarity and momentum that comes from a team working in perfect sync.
    \item {\it Single-Minded Proposition (SMP):} Bring your remote team together for effortless collaboration and remarkable results.
    \item {\it Reasons to Believe (RTB):} Centralized dashboards with real-time project status; integrated communication channels; automated workflows and reporting.
    \item {\it Desired Response:} {\it Think:} "This could finally solve our remote work chaos and get everyone on the same page." {\it Feel:} Organized, in control, and successful. {\it Do:} Sign up for a team demo.
    \item {\it Brand Personality \& Tone:} Professional, efficient, and innovative.
\squishend

\clearpage

\subsection{Prompts for Initial Ad Generation}
\label{appssec:initial_prompts}
\input{initial_gen_prompt}

\input{image_gen_prompts}

\clearpage

\subsection{Prompts used for Critic LLM Operators}
\label{appssec:critic_llm_prompts}

\input{Pairwise_prompt}

\input{reflection_prompt}

\input{improvement_prompt}

\input{gradient_to_prompt}

\clearpage

\subsection{Prompt for Evaluation Based on Persona}
\label{appssec:persona_prompt}

\input{persona_prompt}

\clearpage

\clearpage

\section{Appendix for Section \ref{sec:experiments} }
\label{appsec:experiments}

\subsection{Implementation Details of \textsc{GEPA} for digital ad optimization}
\label{appssec:GEPA_implementation}

\paragraph{Implementation overview}
\textsc{GEPA} \citep{agrawal2025gepa} is an iterative, prompt-optimization-based, self-improving AI method. The key idea of \textsc{GEPA} is to ``combine reflection with evolution''. Here, an evolution is the same concept as taking a textual gradient in \cite{textgrad}, and reflection is the same reflection \textsc{TextBO} (\$\ref{sec:TBoN}) uses. \textsc{GEPA} iterates over three steps in each optimization iteration: Step 1) choosing a prompt and evaluating it; Step 2) updating the meta-reflection; and Step 3) improving the prompt. 

\textsc{GEPA} optimizes the prompt for a set of \textit{problem instances}. In standard agentic benchmarks considered in \cite{agrawal2025gepa}, a problem instance is a single task in a benchmark (together with whatever information is needed to evaluate it) that can be stored and revisited.
In our digital ad setting, a problem instance corresponds to an individual ad viewer (a persona): each evaluation asks how effective a given ad is for a specific ad viewer (i.e., an ad--ad viewer combination), and the overall score for an ad is computed by averaging effectiveness ratings over a random sample of ad viewers (personas).

Unlike the agentic benchmarks, in ad optimization, ad platforms \emph{do not} have the ability to store and retrieve a particular ad viewer at will. Instead, ad platforms (or managers) test an ad to a population of ad viewers and receive an aggregated evaluation signal (e.g., an average effectiveness rating across a sampled set of personas), which prevents us from maintaining instance-level records, such as which prompts win on which ad viewer (in our case, a persona). Therefore, to apply \textsc{GEPA} for digital ad optimization, we need to adapt it.

For Step 1), in agentic AI experiments considered in \S \ref{sec:agenticExp} and \cite{agrawal2025gepa}), where we can store, retrieve, or evaluate individual problem instances), it uses instance-level information to (i) maintain a Pareto archive of which prompts in the candidate prompt set currently win on which problem instances, (ii) stochastically sample a non-dominated candidate prompt (favoring candidates that win on more instances), and (iii) evaluate that sampled candidate by rolling it out on a minibatch of problem instances. On the other hand, in the digital ad optimization experiment (\S \ref{sec:case}), where algorithms have access only to aggregated evaluation results, this step's procedure simplifies to that of \textsc{TextBO}: evaluate the prompt across the entire test persona set.

In Step 2), we employ the same meta-reflection procedure as used for \textsc{TextBO}.

For Step 3), we take one step of evolution (or equivalently, one step of textual gradient) of the prompt we evaluated in Step 1) and accept it only if it has a better evaluation than the original prompt. For agentic AI experiments considered in $\S$\ref{sec:agenticExp} and \cite{agrawal2025gepa}) where we have access to individual problem instances, when the offspring prompt is accepted, we append it to the candidate prompt set and don't throw away the original prompt from the candidate set. On the other hand, in digital ad optimization in $\S$\ref{sec:case}, where we cannot maintain the instance-level Pareto archive and thus cannot construct the frequency-weighted, non-dominated sampling distribution, we discard the original prompt from the candidate set. This is because, if we were to keep appending accepted prompts, the candidate prompt set would grow without bound, and the evaluation budget per prompt would be uniformly diluted. Therefore, when the offspring prompt is accepted, we discard the original prompt, keeping the candidate prompt set size constant. 

Notably, after adapting \textsc{GEPA} to the digital ad optimization problem setup, where we cannot store and retrieve personas freely, it effectively reduces to \textsc{TextBO} with $G=1$ and $N=1$; we describe the pseudocode of the adapted version of \textsc{GEPA} in Algorithm \ref{algo:gepa-ad}.

\paragraph{Shared components with \textsc{TextBO}.}
To ensure a fair comparison, we keep the non-algorithmic components identical to those used for
\textsc{TextBO} (\S~\ref{ssec:implementDetail}): (i) the same scenario-specific initial pool of 64
prompts generated from the creative brief, (ii) the same ad-generation module $\Phi$ (Imagen~4), (iii) the
same critic model $M_{\text{critic}}$ (Gemini~2.5 Flash) and the same meta-prompts for reflection and
prompt rewriting, and (iv) the same persona-based evaluation operator $\mathrm{Eval}(\cdot)$ induced by Twin-2k-500 personas and LLM-scored effectiveness ratings.

\paragraph{Candidate prompt pool (``population'').}
In the ad setting, each prompt $\pi$ produces a single ad image $\Phi(\pi)$ and a single scalar score
$s=\mathrm{Eval}(\Phi(\pi),\mathcal{D}_{\mathcal{X}})$. As we described earlier, we implement \textsc{GEPA} with a fixed-size candidate pool (population) of size $J=5$.
We initialize this pool using the same \textsc{Worst5-of-64} prompts described in
$\S$\ref{ssec:application_tbonbo_ad}.
For each candidate $j\in\{1,\dots,J\}$ we maintain a triple
$(\pi_t^j, \Phi(\pi_t^j), s_t^j)$ where $s_t^j$ is the empirical evaluation score.

\paragraph{Evaluation}
Each time we evaluate an ad image $\Phi(\pi)$ during training, we follow the same protocol as in
\S~\ref{ssec:implementDetail}: we randomly sample 200 personas from the \emph{training} persona set and compute
the mean effectiveness score across those personas (using the same 1--5 rubric and the same
logprob-to-expectation conversion). This yields the scalar feedback used by \textsc{GEPA} for selection.
For reporting final performance (plots/tables), we evaluate selected ads on the held-out test personas.

\paragraph{Meta-reflection update.}
After collecting evaluations, \textsc{GEPA} constructs a global meta-reflection $R_t$ using the same
$\textsc{MetaReflect}$ operator as \textsc{TextBO} (\S~\ref{sec:TBoN} and
\S~\ref{ssec:implementDetail}). Concretely, the critic is prompted to compare a subset of the
best- and worst-performing creatives observed so far (subject to context-length limits) and to distill
recurring patterns into a small set of actionable rules (e.g., ``prefer warm lighting,'' ``show social
context,'' ``avoid clutter''). This reflection is then used to guide subsequent prompt evolution steps.

\paragraph{Evolution (mutation) step and acceptance rule.}
Given a current candidate prompt $\pi_t^j$ and reflection $R_t$, we generate a single offspring prompt
by (i) sampling a targeted textual improvement using the critic (a TextGrad-style local edit
\citep{textgrad}) and (ii) rewriting the prompt accordingly:
\[
\delta_t^j \sim \nabla_{\text{text}}(\pi_t^j;R_t,M_{\text{critic}}),
\qquad
\tilde{\pi}_t^j = \operatorname{Apply}(\pi_t^j,\delta_t^j).
\]
We then generate the corresponding offspring creative $\Phi(\tilde{\pi}_t^j)$ and evaluate it
to obtain $\tilde{s}_t^j=\mathrm{Eval}(\Phi(\tilde{\pi}_t^j),\mathcal{D}_{\mathcal{X}})$.
We use an elitist acceptance rule:
if $\tilde{s}_t^j > s_t^j$, the offspring replaces the parent in the pool; otherwise the parent is kept.
This keeps the pool size constant across iterations (preventing unbounded growth of candidates under a
fixed evaluation budget) while ensuring that each candidate lineage is non-decreasing in its accepted scores.

\paragraph{Full procedure and hyperparameters.}
As in \textsc{TextBO}, we run \textsc{GEPA} for $T=10$ iterations with pool size $J=5$ per scenario.
At each iteration, we evolve each of the $J$ candidates once (one offspring per candidate), evaluate the
offspring using 200 sampled training personas, apply the acceptance rule above, and then update the shared
meta-reflection. For comparison plots and tables, we track the best-performing candidate in the pool at
each iteration,
$j_t^\star := \arg\max_{j} s_t^j$, and report the corresponding ad $\Phi(\pi_t^{j_t^\star})$.

\begin{algorithm}[htp!]
\caption{\textsc{GEPA} (adapted for ad optimization)}
\label{algo:gepa-ad}
\SetAlgoLined
\KwIn{System $\Phi$; initial prompts $\{\pi_0^{j}\}_{j=1}^{J}$; critic $M_{\text{critic}}$; iterations $T$; pool size $J$}
\KwOut{Best ads $\{\Phi(\pi_t^{j_t^\star})\}_{t=0}^{T}$, where $j_t^\star=\arg\max_j s_t^j$}
\For{$j=1$ \KwTo $J$}{
  $c_0^{j} \leftarrow \Phi(\pi_0^{j})$\;
  $s_0^{j} \leftarrow \text{Eval}(c_0^{j}, \mathcal{D}_{\mathcal{X}})$ 
}
$H_0 \leftarrow \{(\pi_0^{j},c_0^{j},s_0^{j})\}_{j=1}^{J}$,\quad $R_0 \leftarrow \emptyset$\;

\For{$t = 0$ \KwTo $T-1$}{
  $R_t \leftarrow \textsc{MetaReflect}(H_t; M_{\text{critic}})$\;
  \For{$j=1$ \KwTo $J$}{
    $\delta_t^{j} \leftarrow \nabla_{\text{text}}(\pi_t^{j}; R_t, M_{\text{critic}})$\;
    $\tilde{\pi}_{t+1}^{j} \leftarrow \operatorname{Apply}(\pi_t^{j}, \delta_t^{j})$\;
    $\tilde{c}_{t+1}^{j} \leftarrow \Phi(\tilde{\pi}_{t+1}^{j})$\;
    $\tilde{s}_{t+1}^{j} \leftarrow \text{Eval}(\tilde{c}_{t+1}^{j}, \mathcal{D}_{\mathcal{X}})$\;
    \eIf{$\tilde{s}_{t+1}^{j} > s_t^{j}$}{
      $(\pi_{t+1}^{j},c_{t+1}^{j},s_{t+1}^{j}) \leftarrow (\tilde{\pi}_{t+1}^{j},\tilde{c}_{t+1}^{j},\tilde{s}_{t+1}^{j})$\;
    }{
      $(\pi_{t+1}^{j},c_{t+1}^{j},s_{t+1}^{j}) \leftarrow (\pi_t^{j},c_t^{j},s_t^{j})$\;
    }
  }
  $H_{t+1} \leftarrow H_t \cup \{(\pi_{t+1}^{j},c_{t+1}^{j},s_{t+1}^{j})\}_{j=1}^{J}$\;
}
\KwRet{$\{\Phi(\pi_t^{j_t^\star})\}_{t=0}^{T}$}
\end{algorithm}

\clearpage

\subsection{Detailed Experiment Results}
\label{appssec:detailedExpResult}

\begin{table}[!ht]
\centering
\small
\setlength\tabcolsep{4.5pt}

\scalebox{0.755}{
\begin{tabular}{cccccccccc|c}
\multicolumn{11}{c}{\large \textbf{TextBO}} \\
\toprule
 &  & \textbf{Aeterno} & \textbf{Aurasonics} & \textbf{Greenbite} & \textbf{Mindgarden} & \textbf{Momentum} & \textbf{Oasis} & \textbf{Odyssey} & \textbf{Syncflow} & \textbf{Average} \\
\midrule
\multicolumn{2}{l}{\textsc{Worst5-of-N}} 
& 2.942 (0.025) & 2.793 (0.023) & 2.706 (0.024) & 2.722 (0.024) 
& 2.644 (0.024) & 2.839 (0.025) & 2.828 (0.025) & 2.641 (0.023) & {2.739} \\
\midrule
\multirow{10}{*}{\textbf{Steps}} 
& 1  & 3.098 (0.028) & 2.813 (0.022) & 3.001 (0.027) & 2.980 (0.027) 
      & 2.799 (0.027) & 3.062 (0.028) & 2.943 (0.027) & 2.927 (0.023) & {2.932} \\
& 2  & 3.125 (0.026) & 2.969 (0.027) & 3.041 (0.026) & 2.980 (0.027) 
      & 2.907 (0.025) & 3.107 (0.026) & 2.943 (0.027) & 2.927 (0.023) & 2.982 
\\
& 3  & 3.125 (0.026) & 3.041 (0.025) & 3.041 (0.026) & 2.980 (0.027) 
      & 2.907 (0.025) & 3.131 (0.019) & 3.063 (0.019) & 2.927 (0.023) & 3.013 
\\
& 4  & 3.194 (0.029) & 3.041 (0.025) & 3.041 (0.026) & 3.132 (0.022) 
      & 2.907 (0.025) & 3.324 (0.028) & 3.063 (0.019) & 2.927 (0.023) & 3.062 
\\
& 5  & 3.194 (0.029) & 3.041 (0.025) & 3.076 (0.026) & 3.132 (0.022) 
      & 2.907 (0.025) & 3.324 (0.028) & 3.063 (0.019) & 3.127 (0.022) & 3.096 
\\
& 6  & 3.257 (0.023) & 3.041 (0.025) & 3.076 (0.026) & 3.132 (0.022) 
      & 2.907 (0.025) & 3.324 (0.028) & 3.070 (0.022) & 3.127 (0.022) & 3.097 
\\
& 7  & 3.257 (0.023) & 3.104 (0.028) & 3.076 (0.026) & 3.132 (0.022) 
      & 2.907 (0.025) & 3.324 (0.028) & 3.107 (0.027) & 3.127 (0.022) & 3.111 
\\
& 8  & 3.257 (0.023) & 3.104 (0.028) & 3.076 (0.026) & 3.132 (0.022) 
      & 2.913 (0.026) & 3.324 (0.028) & 3.107 (0.027) & 3.127 (0.022) & 3.112 
\\
& 9  & 3.257 (0.023) & 3.104 (0.028) & 3.076 (0.026) & 3.132 (0.022) 
      & 2.926 (0.026) & 3.324 (0.028) & 3.107 (0.027) & 3.127 (0.022) & 3.114 
\\
& 10 & 3.257 (0.023) & 3.104 (0.028) & 3.168 (0.025) & 3.132 (0.022) 
      & 3.024 (0.025) & 3.324 (0.028) & 3.107 (0.027) & 3.127 (0.022) & 3.155 
\\
\midrule
\multicolumn{2}{l}{\textsc{Best-of-N}} 
& 2.982 (0.025) & 2.934 (0.027) & 2.898 (0.025) & 2.930 (0.028) 
& 3.030 (0.027) & 3.265 (0.022) & 2.966 (0.027) & 3.086 (0.023) & 3.011 
\\
\bottomrule
\end{tabular}
}
\caption{Progress of \textsc{TextBO} across 10 optimization steps from its starting point (\textsc{Worst5-of-N}) and comparison with the \textsc{Best-of-N} baseline. Each data point, in the mean (standard error) format, indicates the mean and standard error across 411 personas in the test persona set. The last column, labeled ``Average'', reports the averages of the means across the eight scenarios. 
}
\label{tab:steps_cs}
\end{table}

\begin{table}[!ht]
\centering
\small
\setlength\tabcolsep{4.5pt}

\scalebox{0.755}{
\begin{tabular}{cccccccccc|c}
\multicolumn{11}{c}{\large \textbf{GEPA}} \\
\toprule
 &  & \textbf{Aeterno} & \textbf{Aurasonics} & \textbf{Greenbite} & \textbf{Mindgarden} & \textbf{Momentum} & \textbf{Oasis} & \textbf{Odyssey} & \textbf{Syncflow} & \textbf{Average} \\
\midrule
\multicolumn{2}{l}{\textsc{Worst5-of-N}}
& 2.942 (0.025) & 2.793 (0.023) & 2.706 (0.024) & 2.722 (0.024)
& 2.644 (0.024) & 2.839 (0.025) & 2.828 (0.025) & 2.641 (0.023) & 2.739 \\
\midrule
\multirow{10}{*}{\textbf{Steps}}
& 1  & 3.036 (0.025) & 2.942 (0.025) & 3.029 (0.025) & 2.841 (0.024)
      & 2.816 (0.024) & 3.151 (0.024) & 2.915 (0.025) & 2.880 (0.025) & 2.939 \\
& 2  & 3.036 (0.025) & 2.942 (0.025) & 3.046 (0.024) & 2.945 (0.023)
      & 2.907 (0.023) & 3.151 (0.024) & 3.071 (0.028) & 2.880 (0.025) & 2.992 \\
& 3  & 3.036 (0.025) & 2.942 (0.025) & 3.046 (0.024) & 2.945 (0.023)
      & 2.907 (0.023) & 3.182 (0.025) & 3.088 (0.026) & 2.914 (0.018) & 3.003 \\
& 4  & 3.138 (0.025) & 3.028 (0.022) & 3.046 (0.024) & 2.954 (0.024)
      & 2.976 (0.026) & 3.187 (0.023) & 3.088 (0.026) & 2.949 (0.028) & 3.033 \\
& 5  & 3.138 (0.025) & 3.028 (0.022) & 3.046 (0.024) & 3.012 (0.022)
      & 2.976 (0.026) & 3.187 (0.023) & 3.088 (0.026) & 3.008 (0.026) & 3.049 \\
& 6  & 3.157 (0.023) & 3.028 (0.022) & 3.046 (0.024) & 3.012 (0.022)
      & 2.976 (0.026) & 3.209 (0.020) & 3.088 (0.026) & 3.008 (0.026) & 3.052 \\
& 7  & 3.157 (0.023) & 3.028 (0.022) & 3.046 (0.024) & 3.068 (0.027)
      & 2.976 (0.026) & 3.209 (0.020) & 3.088 (0.026) & 3.008 (0.026) & 3.060 \\
& 8  & 3.157 (0.023) & 3.028 (0.022) & 3.046 (0.024) & 3.068 (0.027)
      & 2.976 (0.026) & 3.209 (0.020) & 3.088 (0.026) & 3.008 (0.026) & 3.060 \\
& 9  & 3.157 (0.023) & 3.028 (0.022) & 3.046 (0.024) & 3.068 (0.027)
      & 2.976 (0.026) & 3.278 (0.023) & 3.088 (0.026) & 3.008 (0.026) & 3.070 \\
& 10 & 3.157 (0.023) & 3.028 (0.022) & 3.046 (0.024) & 3.068 (0.027)
      & 2.976 (0.026) & 3.278 (0.023) & 3.088 (0.026) & 3.008 (0.026) & 3.070 \\
\midrule
\multicolumn{2}{l}{\textsc{Best-of-N}}
& 2.982 (0.025) & 2.934 (0.027) & 2.898 (0.025) & 2.930 (0.028)
& 3.030 (0.027) & 3.265 (0.022) & 2.966 (0.027) & 3.086 (0.023) & 3.011 \\
\bottomrule
\end{tabular}
}
\caption{Progress of \textsc{GEPA} across 10 optimization steps from its starting point (\textsc{Worst5-of-N}) and comparison with the \textsc{Best-of-N} baseline. Each data point, in the mean (standard error) format, indicates the mean and standard error across 411 personas in the test persona set. The last column, labeled ``Average'', reports the averages of the means across the eight scenarios.}
\label{tab:steps_cs_gepa}
\end{table}

\begin{figure*}[ht]
\centering
\setlength{\tabcolsep}{3pt}
\renewcommand{\arraystretch}{1.05}
\footnotesize
\resizebox{1.0\textwidth}{!}{%
\begin{tabular}{
    M{0.12\textwidth}
    M{0.17\textwidth}
    M{0.17\textwidth}
    M{0.17\textwidth}
    M{0.17\textwidth}
    M{0.17\textwidth}
}
&
\shortstack{Odyssey\\(Electric family SUV)} &
\shortstack{Oasis\\(Secluded resort)} &
\shortstack{Momentum\\(Mobile banking app)} &
\shortstack{Aeterno\\(Luxury watch)} &
\shortstack{Syncflow\\(Collaboration tool)}
\\

\textsc{Worst5-of-64} &
\shortstack{\includegraphics[width=\linewidth]{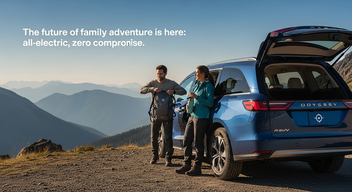}\\Score: 2.828} &
\shortstack{\includegraphics[width=\linewidth]{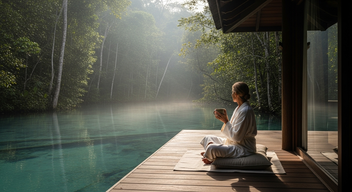}\\Score: 2.839} &
\shortstack{\includegraphics[width=\linewidth]{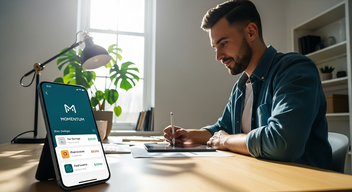}\\Score: 2.644} &
\shortstack{\includegraphics[width=\linewidth]{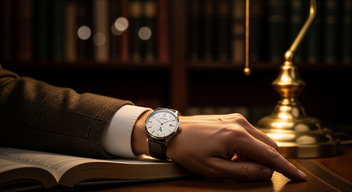}\\Score: 2.942} &
\shortstack{\includegraphics[width=\linewidth]{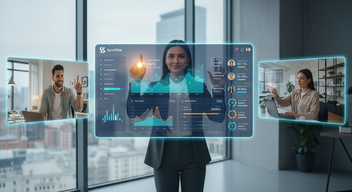}\\Score: 2.641}
\\

\textsc{Best-of-64} &
\shortstack{\includegraphics[width=\linewidth]{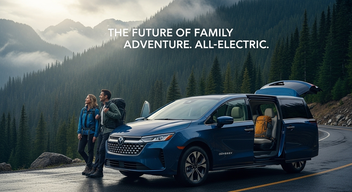}\\Score: 2.966} &
\shortstack{\includegraphics[width=\linewidth]{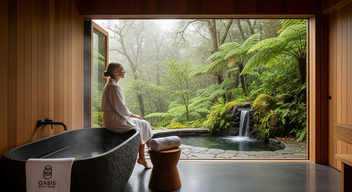}\\Score: 3.265} &
\shortstack{\includegraphics[width=\linewidth]{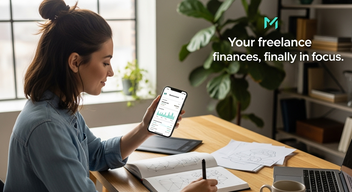}\\Score: 3.030} &
\shortstack{\includegraphics[width=\linewidth]{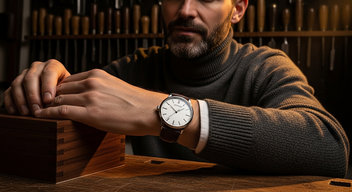}\\Score: 2.982} &
\shortstack{\includegraphics[width=\linewidth]{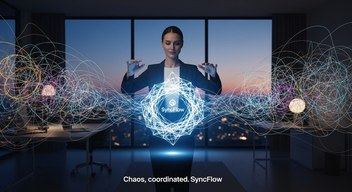}\\Score: 3.086}
\\

\textsc{GEPA (T=10)} &
\shortstack{\includegraphics[width=\linewidth]{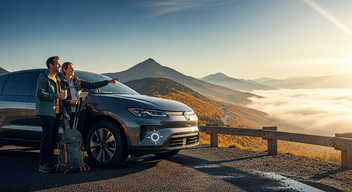}\\Score: 3.088} &
\shortstack{\includegraphics[width=\linewidth]{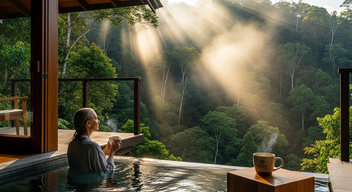}\\Score: 3.278} &
\shortstack{\includegraphics[width=\linewidth]{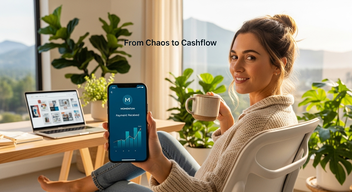}\\Score: 2.976} &
\shortstack{\includegraphics[width=\linewidth]{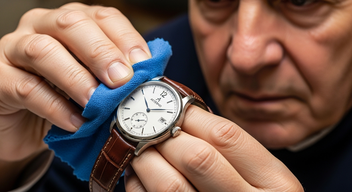}\\Score: 3.157} &
\shortstack{\includegraphics[width=\linewidth]{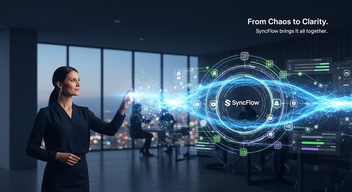}\\Score: 3.008}
\\

\textsc{TextBO (T=10)} &
\shortstack{\includegraphics[width=\linewidth]{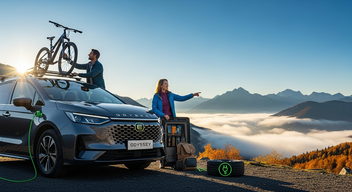}\\Score: 3.107} &
\shortstack{\includegraphics[width=\linewidth]{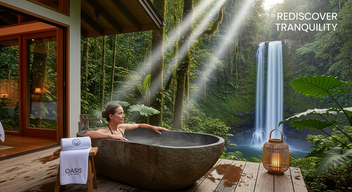}\\Score: 3.324} &
\shortstack{\includegraphics[width=\linewidth]{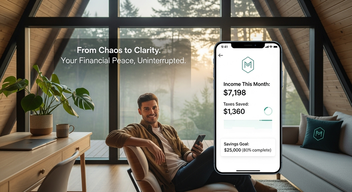}\\Score: 3.024} &
\shortstack{\includegraphics[width=\linewidth]{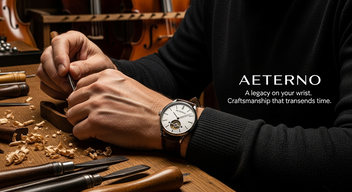}\\Score: 3.257} &
\shortstack{\includegraphics[width=\linewidth]{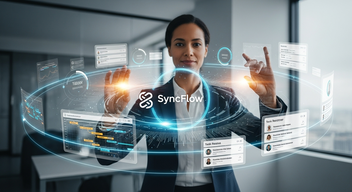}\\Score: 3.127}
\\

\end{tabular}}
\caption{\textsc{TextBO} and baseline's generated Ad images for five fictional brands: Odyssey electric family SUV (Scenario 3), Oasis eco-lodge (Scenario 4), Momentum Digital Banking (Scenario 5), Aeterno luxury watch (Scenario 7), and SyncFlow B2B software (Scenario 8)}
\label{fig:ad_images_more}
\end{figure*}
\clearpage
\subsection{Final Prompts in \textsc{TextBO} and \textsc{GEPA}}
\label{appssec:final_tbonbo_gepa_prompts}

\begin{figure}[ht]
\centering
\begin{tcolorbox}[
    colback=promptbg!100!white,
    sharp corners=south, 
    boxrule=0.25mm, 
    fonttitle=\bfseries, 
    title={\textbf{Prompt for generating the \textsc{TextBO} $(T=10)$ ad for GreenBite (plant-based burger patty)}}, 
    width=\textwidth, 
    enhanced,
    drop shadow
    ]
\scriptsize
\setlength{\parskip}{0pt}
\setlength{\itemsep}{0pt}

"**Key Message:** 
\\
"The No-Compromise Burger"—to visually prove that a plant-based burger can be just as juicy, delicious, and satisfying as a traditional beef burger, making it the hero of the ultimate backyard BBQ. 
\\**Core Components:** 
\\
* **Scene \& Environment:** A vibrant, meticulously designed modern backyard patio, perfect for a relaxed, aspirational late-afternoon get-together. The background is a soft-focus composition of stylish outdoor furniture, lush green garden plants, and subtly visible, engaged friends enjoying themselves and partaking in the shared experience. Occasional hints of other complementary gourmet side dishes or refreshing beverages on the blurred patio table enhance the sense of a complete, desirable BBQ experience, creating a rich, warm, and inviting social atmosphere. 
\\
* **Action \& Narrative:** A pair of adult hands are holding up a perfectly constructed, gourmet-style cheeseburger. The action captures either: (A) a moment of serene, authentic anticipation just before the first bite, with the subject's face (in soft focus) conveying genuine contentment or subtle bliss, possibly with eyes gently closed, savoring the aroma; OR (B) the burger held out slightly towards the camera, inviting the viewer into the experience, with the subject's soft-focus expression conveying a shared, authentic delight or playful 'aha!' moment. The emphasis is on genuine human emotion and connection, avoiding any forced, exaggerated, or unnatural expressions (especially if a bite is implied). A single, glistening drop of rich, savory juice is visible, about to fall from the patty, powerfully emphasizing its irresistible succulence and indulgent quality without appearing messy. * **Central Subject:** The star of the image is the GreenBite burger. It's a thick, plant-based patty with a deeply browned, flawlessly seared crust and visible grill marks, looking indistinguishable from premium ground beef. It's topped with a layer of glistening, perfectly melted cheddar cheese, a crisp piece of bright green leaf lettuce, a thick, ruby-red slice of a heirloom tomato, and a few rings of sharp red onion, all between a fluffy, toasted brioche bun. * **Composition \& Framing:** An enticing, slightly low-angle medium close-up shot where the GreenBite burger is the undeniable hero, prominently filling a significant portion of the frame and drawing the viewer's eye directly to its appetizing qualities. The focus is tack-sharp on the GreenBite patty's seared texture and the glistening juice. The hands holding the burger, and the subject's face (if present), along with the backyard scene, are rendered in a beautiful, soft-focus bokeh, creating a strong emotional connection and aspirational appeal while ensuring the burger remains supremely appetizing. There is significant negative space with soft, blurred greenery in the upper portion of the frame for a text overlay. 
\\
**Stylistic Qualities:** 
\\
* **Photography Style:** High-end, photorealistic commercial food photography with a cinematic lifestyle feel. The image should be incredibly crisp, detailed, and textured. * **Lighting:** Bathed in the warm, dreamy glow of golden hour sunlight. The light source comes from the side, casting soft shadows that accentuate the burger's shape and highlight the glistening textures of the melted cheese and juicy patty. A slight, elegant lens flare bleeds into the frame, adding to the warm, aspirational feel. * **Color Palette \& Tone:** Rich, saturated, and mouth-watering. A color palette dominated by the warm browns of the seared patty, the vibrant orange of the cheese, the fresh greens of the lettuce, and the deep reds of the tomato, all set against the warm, earthy tones of the background. 
\\
* **Atmosphere \& Feeling:** Evokes a powerful feeling of shared satisfaction, vibrant well-being, and guilt-free indulgence. The atmosphere is confidently positive, supremely delicious, and distinctly aspirational, conveying a desirable lifestyle of effortless enjoyment. **Brand Elements:** * 
\\
**Logo:** The simple, modern "GreenBite" logo (clean sans-serif font with a subtle leaf element) is integrated prominently yet naturally within the scene. It is elegantly printed on the branded wax paper partially wrapping the bottom of the burger, and also subtly visible on a branded napkin or a stylish sauce bottle on the blurred patio table in the background, reinforcing brand presence contextually within the aspirational setting. * In addition to the physical logo placements, a clean, prominent text overlay of the 'GreenBite' brand name, and optionally a relevant tagline (e.g., 'Believe Your Bite'), should be integrated into the significant negative space in the upper portion of the frame. This overlay should be highly legible and reinforce clear brand recognition without detracting from the visual appeal."

\end{tcolorbox}

\caption{Prompts for \textsc{TextBO} ($T=10$)}
\label{fig:TBoN_ad_prompt}
\end{figure}

\begin{figure}[ht]
\centering
\begin{tcolorbox}[
    colback=promptbg!100!white,
    sharp corners=south, 
    boxrule=0.25mm, 
    fonttitle=\bfseries, 
    title={\textbf{Prompt for generating the \textsc{GEPA} ($T=10$) ad for GreenBite (plant-based burger patty)}}, 
    width=\textwidth, 
    enhanced,
    drop shadow
    ]
\scriptsize
\setlength{\parskip}{0pt}
\setlength{\itemsep}{0pt}

"**Key Message:** 
\\
Visually communicate the "no-compromise" promise of GreenBite by capturing the exact moment of pure, blissful satisfaction of a consumer realizing a plant-based burger can be just as juicy and delicious as its beef counterpart. 
\\
**Core Components:** 
\\
* **Scene \& Environment:** A vibrant, sun-drenched modern backyard patio during a relaxed get-together of friends in their late 20s and early 30s. The setting is stylish and aspirational, featuring a natural wood deck, comfortable outdoor furniture, festoon string lights, and lush, green potted plants in the background. 
\\* **Action \& Narrative:** The central focus is a woman with her eyes open, full of genuine, surprised satisfaction, capturing the peak moment of consumption. Her mouth is wide open, actively biting into the fully-loaded GreenBite burger. The patty is clearly visible within her mouth, emphasizing the powerful 'A-ha!' moment of discovery and the sheer indulgence of the bite, directly communicating the 'no-compromise' promise. 
\\
* **Composition \& Framing:** Dynamic, slightly low-angle medium shot focusing on the woman and the burger. The burger is the hero; the depth of field is shallow, rendering the juicy, perfectly seared patty, melted plant-based cheddar, crisp lettuce, and glossy brioche bun in hyper-detailed, sharp focus. The background is beautifully out of focus with a pleasing bokeh effect. There is negative space in the upper left corner suitable for a text overlay. 
\\
**Stylistic Qualities:** * **Photography Style:** High-end, commercial lifestyle food photography. Polished, authentic, and incredibly photorealistic. Looks like an image from a top-tier food publication. * **Lighting:** Bathed in warm, cinematic golden hour sunlight. The light creates long, soft shadows and a beautiful, inviting glow, highlighting the textures of the food. Crucially, integrate a strong, warm sun flare or haze into the scene, creating a dramatic, aspirational glow that enhances the sense of satisfaction and makes the overall atmosphere feel more vibrant and special. * **Color Palette \& Tone:** Rich, warm, and saturated colors. Deep browns of the seared patty, vibrant greens of the lettuce and surrounding plants, warm yellows and oranges from the cheese and sunlight. The overall tone is confident, energetic, and mouth-watering. 
\\
* **Atmosphere \& Feeling:** A powerful feeling of guilt-free indulgence and vibrant satisfaction. The atmosphere is warm, happy, and effortlessly cool, capturing the joy of sharing great food with friends. **Inclusions:** 
\\
* **Logo:** Ensure the 'GreenBite' logo is subtly yet clearly visible on the burger wrapper, or on a small, non-distracting element within the scene, to provide product context without pulling focus from the main action. 
\\
* **Text Overlay:** A prominent, clean white text overlay, featuring a strong, benefit-driven headline such as 'THE NO-COMPROMISE BURGER' or 'BEYOND EXPECTATION.', is strategically placed in the negative space (e.g., upper left corner). This headline should be clearly legible, visually impactful, and act as a primary visual element, prioritizing the core message and value proposition over just generic brand reinforcement."

\end{tcolorbox}

\caption{Prompts for \textsc{GEPA} ($T=10$)}
\label{fig:GEPA_ad_prompt}
\end{figure}

\clearpage

\section{Appendix for Section \ref{sec:agenticExp}}

\subsection{Implementation Details of \textsc{GEPA} and \textsc{TextBO-GEPA} for agentic AI experiments}
\label{appssec:agentic_GEPA_implementation}

\paragraph{\textsc{GEPA}}
\textsc{GEPA} \citep{agrawal2025gepa} is an iterative, prompt-optimization-based, self-improving AI method. The key idea of \textsc{GEPA} is to ``combine reflection with evolution''. Here, an evolution is the same concept as taking a textual gradient in \cite{textgrad}. \textsc{GEPA} iterates over three steps in each optimization iteration: Step 1) choosing a prompt and evaluating it; Step 2) updating the meta-reflection; and Step 3) improving the prompt. 

For Step 1), it uses instance-level information to (i) maintain a Pareto archive of which prompts in the candidate prompt set currently win on which problem instances, (ii) stochastically sample a non-dominated candidate prompt (favoring candidates that win on more instances), and (iii) evaluate that sampled candidate by rolling it out on a minibatch of problem instances. 

In Step 2), \textsc{GEPA} updates a global meta-reflection $R_t$ that summarizes what has tended to work well or poorly so far.
Concretely, after collecting rollouts and their scores, it applies a reflection operator to the accumulated history,
$R_t := \texttt{MetaReflect}(H_t; M_{\text{critic}})$, where the critic is shown a context-limited subset of representative
high-performing and low-performing rollouts (including the prompts and outcomes) and asked to (i) identify recurring
success patterns and failure modes and (ii) distill them into a short list of actionable, prompt-level guidelines.
This shared reflection is carried forward and provided as context to subsequent prompt-evolution (textual-gradient) steps. This step is fundamentally the same as \textsc{TextBO}'s meta-reflection steps.

For Step 3), we take one step of evolution (or equivalently, one step of textual gradient) of the prompt we evaluated in Step 1) and accept it only if it has a better evaluation than the original prompt. When the offspring prompt is accepted, we append it to the candidate prompt set without removing the original prompt from the candidate set.

\paragraph{\textsc{TextBO-GEPA}}
\textsc{TextBO-GEPA} is identical to \textsc{GEPA} except for the prompt-improvement step (Step 3): whenever \textsc{GEPA} would take a single
evolution step (one textual-gradient update) to propose an offspring prompt, \textsc{TextBO-GEPA} instead performs multiple
successive \textsc{Best-of-N} textual-gradient steps (as in \S \ref{sec:BoNAndTheory}) before evaluating and applying the same acceptance rule as \textsc{GEPA}.
That is, at each inner step, it samples $N$ local textual edits, applies them to form $N$ candidate prompts, generates the
corresponding candidate outcomes, and uses the critic (conditioned on the current reflection) to select the most promising
candidate; repeating this for several steps yields a refined offspring prompt, which is then accepted and appended to the candidate prompt set only when it improves upon its parent under the evaluation procedure.

\clearpage

\end{appendices}
\end{document}

%% file: GP-UCB.tex
\section{GP-UCB and Its Evaluation Efficiency}
\label{appsec:GP-UCB}

\paragraph{Setup.}

Let $f$ be a unknown, black-box function that maps a compact and convex feasible region $\mathcal{D} \subset \mathbb{R}^d$ to $\mathbb{R}$. At each time $t$, a sampling algorithm $\mathcal{A}$ selects a point $x_t \in \mathcal{D}$ to evaluate. This selection is adaptive, based on the history of previously chosen points and their observed outcomes. The evaluation at $x_t$ yields a noisy observation $y_t=f\left(x_t\right)+\epsilon_t$, where $\epsilon_t$ is a random noise term, where $\left\{\epsilon_t: t=1,2, \ldots\right\}$ are independent, sub-Gaussian random variables. The unknown function $f$ is assumed to be an element of a Reproducing Kernel Hilbert Space (RKHS), $\mathcal{N}_{\Psi}(\mathcal{D})$, induced by a stationary kernel $\Psi$. Functions within this RKHS possess a specific smoothness property, the nature of which is determined by the kernel used.  The space $\mathcal{N}_{\Psi}(\mathcal{D})$ is endowed with an inner product and is formally defined as the set of functions 
$$
\mathcal{N}_{\Psi}(\mathcal{D}):=\left\{g: \mathcal{D} \mapsto \mathbb{R} \mid g(\mathbf{x})=\sum_{j=1}^{\infty} c_j \Psi\left(\mathbf{x}-\mathbf{x}_j\right), \text { for }\left\{c_j\right\} \subset \mathbb{R} \text { and }\left\{\mathbf{x}_j\right\} \subset \mathcal{D} \text { such that }\|g\|_{\mathcal{N}_{\Psi}(\mathcal{D})}<\infty\right\}
$$
where the norm is given by:
$
\|g\|_{\mathcal{N}_{\Psi}(\mathcal{D})}:=\left(\sum_{j, l=1}^{\infty} c_j c_l \Psi\left(\mathbf{x}_j-\mathbf{x}_l\right)\right)^{\frac{1}{2}}
$. 
For any two functions $g_1(\mathbf{x})=\sum_{j=1}^{\infty} a_j \Psi\left(\mathbf{x}-\mathbf{x}_j\right)$ and $g_2(\mathbf{x})=\sum_{l=1}^{\infty} b_l \Psi\left(\mathbf{x}-\mathbf{x}_l\right)$ in the space, their inner product is
$
\left\langle g_1, g_2\right\rangle_{\mathcal{N}_{\Psi}(\mathcal{D})}:=\sum_{j, l=1}^{\infty} a_j b_l \Psi\left(\mathbf{x}_j-\mathbf{x}_l\right)
$.
A key feature of RKHS space is the reproducing property, $\langle g, \Psi(\mathbf{x}-\cdot)\rangle_{\mathcal{N}_{\Psi}(\mathcal{D})}=g(\mathbf{x})$, which holds for all $\mathbf{x} \in \mathcal{D}$. The norm of the function $f$ in this space is assumed to be bounded such that $\|f\|_{\mathcal{N}_{\Psi}(\mathcal{D})} \leq B$ for some constant $B>0$.

\paragraph{The objectives.}
 The performance of $\mathcal{A}$ is often evaluated through two types of regret over a time horizon $T$. One is \textit{cumulative regret}, defined as
 \begin{align}
     \mathcal{R}_C(T ; f, \mathcal{A}):=\sum_{t=1}^T\left[\max _{\mathbf{x} \in \mathcal{D}} f(\mathbf{x})-f\left(\mathbf{x}_t\right)\right]
 \end{align}
The other is \textit{simple regret}, defined as
\begin{align}
    \mathcal{R}_S(T ; f, \mathcal{A}):=\max _{\mathbf{x} \in \mathcal{D}} f(\mathbf{x})-f\left(\mathbf{x}^{(T)}\right)
\end{align}
where $\mathbf{x}^{(T)}$ is the output of the algorithm $\mathcal{A}$ after taking $T$ function evaluations.

\paragraph{The lower bound.}
\cite{scarlett2017lower} showed that a lower bound exists on the regret of \textit{any} algorithm for $f \in \mathcal{N}_{\Psi}(\mathcal{D})$.  Specifically, for any constant $B>0$,
\begin{align}
    \inf _{\mathcal{A}} \sup _{\|f\|_{\mathcal{N}_{\Psi}(\mathcal{D})} \leq B} \mathbb{E}\left[\mathcal{R}_C(T ; f, \mathcal{A})\right] = \begin{cases} \Omega (T^{\frac{\nu+d}{2 \nu+d}}), & \text { for Matérn kernels, } \\ \Omega ( T^{\frac{1}{2}} \ln ^{\frac{d}{2}}(T)), & \text { for SE kernels. }\end{cases}
\end{align}
\begin{align}
    \inf _{\mathcal{A}} \sup _{\|f\|_{\mathcal{N}_{\Psi}(\mathcal{D})} \leq B} T\left(\epsilon, f, \mathcal{A}\right) = \begin{cases} \Omega\left(\left(\frac{1}{\epsilon}\right)^{2+d / \nu}\right), & \text { for Matérn kernels, } \\ \Omega\left(\frac{1}{\epsilon^2}\left(\log \frac{1}{\epsilon}\right)^{d / 2}\right), & \text { for SE kernels. }\end{cases}
\end{align}
where Matérn kernels are defined as
\begin{align}
\Psi_{\mathrm{M}}\left(\mathbf{x}-\mathbf{x}^{\prime}\right):=\frac{1}{\Gamma(\nu) 2^{\nu-1}}\left(\frac{2 \sqrt{\nu}\left\|\mathbf{x}-\mathbf{x}^{\prime}\right\|_2}{\ell}\right)^\nu K_\nu\left(\frac{2 \sqrt{\nu}\left\|\mathbf{x}-\mathbf{x}^{\prime}\right\|_2}{\ell}\right), \quad \mathbf{x}, \mathbf{x}^{\prime} \in \mathcal{D}
\end{align}
and SE kernels are defined as
\begin{align}
\Psi_{\mathrm{SE}}\left(\mathbf{x}-\mathbf{x}^{\prime}\right)=\exp \left(-\frac{\left\|\mathbf{x}-\mathbf{x}^{\prime}\right\|_2^2}{2 \ell^2}\right), \quad \mathbf{x}, \mathbf{x}^{\prime} \in \mathcal{D}
\end{align}
where $\nu>0$ is the smoothness parameter, $\ell>0$ is the length-scale parameter, $\Gamma(\cdot)$ is the gamma function, $K_\nu(\cdot)$ is the modified Bessel function of the second kind of order $\nu$, and $\|\cdot\|_2$ denotes the Euclidean norm. 

\paragraph{GP-UCB algorithm.} The Gaussian Process Upper Confidence Bound (GP-UCB) algorithm, introduced in the seminal work of \cite{srinivas2012information}, operates from a Bayesian perspective by placing a Gaussian Process (GP) prior on the unknown objective function~$f$. This GP is characterized by a zero-mean function and a symmetric, positive-definite covariance function, or kernel, $k: \mathcal{D} \times \mathcal{D} \mapsto \mathbb{R}$, where $k\left(\mathbf{x}, \mathbf{x}^{\prime}\right)=\Psi\left(\mathbf{x}-\mathbf{x}^{\prime}\right)$ for all $\mathbf{x}, \mathbf{x}^{\prime} \in \mathcal{D}$.  For any finite set of points $\{\mathbf{x}_1, \ldots, \mathbf{x}_t\} \subset \mathcal{D}$, the corresponding function values $(f(\mathbf{x}_1), \ldots, f(\mathbf{x}_t))^{\top}$ are assumed to follow a multivariate normal distribution with a zero mean vector and a covariance matrix with entries $[\mathbf{K}_t]_{jl} = k(\mathbf{x}_j, \mathbf{x}_l)$. Under the assumption of independent, zero-mean Gaussian observation noise with variance $\sigma^2$, conditioning the GP prior on a set of observations $\mathbf{D}_t:=\{(\mathbf{x}_j, y_j) : j=1, \ldots, t\}$ yields a closed-form posterior distribution. The posterior mean and variance at any point $\mathbf{x} \in \mathcal{D}$ are given by:
\begin{align}
\mathbb{E}\left[f(\mathbf{x}) \mid \mathbf{D}_t\right]  =\mathbf{k}_t^{\top}(\mathbf{x})\left(\mathbf{K}_t+\sigma^2 \mathbf{I}_t\right)^{-1} \mathbf{y}_t,
    \operatorname{Var}\left[f(\mathbf{x}) \mid \mathbf{D}_t\right]  =k(\mathbf{x}, \mathbf{x})-\mathbf{k}_t(\mathbf{x})^{\top}\left(\mathbf{K}_t+\sigma^2 \mathbf{I}_t\right)^{-1} \mathbf{k}_t(\mathbf{x}),
\end{align}
where $\mathbf{k}_t(\mathbf{x})$ is the vector of covariances between~$\mathbf{x}$ and the observed points, $\mathbf{K}_t$ is the covariance matrix of the observed points, $\mathbf{y}_t$ is the vector of observed values, and~$\mathbf{I}_t$ is the identity matrix. Inspired by upper confidence bound methods in the multi-armed bandit literature (Auer et al., 2002), GP-UCB employs an optimistic acquisition strategy to select subsequent evaluation points. At each step $t$, the next point $\mathbf{x}_{t+1}$ is chosen by maximizing an upper confidence bound on the function's value
$
    \mathbf{x}_{t+1}=\underset{\mathbf{x} \in \mathcal{D}}{\arg \max }\left(\mathbb{E}\left[f(\mathbf{x}) \mid \mathbf{D}_t\right]+\sqrt{\beta_t \operatorname{Var}\left[f(\mathbf{x}) \mid \mathbf{D}_t\right]}\right).
$
This acquisition function naturally balances exploitation, driven by the posterior mean $\mathbb{E}\left[f(\mathbf{x}) \mid \mathbf{D}_t\right]$, with exploration, driven by the posterior standard deviation $\sqrt{\operatorname{Var}\left[f(\mathbf{x}) \mid \mathbf{D}_t\right]}$. The tunable parameter $\beta_t > 0$ explicitly governs this trade-off, making its specification critical to the algorithm's performance.

\paragraph{The evaluation efficiency of GP-UCB.} 
\cite{whitehouse2023sublinear, wang2023regret} showed that GP-UCB \cite{srinivas2012information} regret upper bound for both expected cumulative regret and simple regret almost tightly matches the lower bound provided by \cite{scarlett2017lower}.  Specifically, for any constant $B>0$, the GP-UCB algorithm $\mathcal{A}_{\text{GPUCB}}$ achieves
\begin{align}
    \sup _{\|f\|_{\mathcal{N}_{\Psi}(\mathcal{D})} \leq B} \mathbb{E}\left[\mathcal{R}_C(T ; f, \mathcal{A}_{\text{GPUCB}})\right] = \begin{cases} O( T^{\frac{\nu+d}{2 \nu+d}}), & \text { for Matérn kernels, } \\ O( T^{\frac{1}{2}} \ln ^{\frac{d}{2}}(T)) & \text { for SE kernels. }\end{cases}
\end{align}
\begin{align}
    \sup_{\|f\|_{\mathcal{N}_{\Psi}(\mathcal{D})} \leq B} T(\epsilon, f, \mathcal{A}_{\mathrm{GPUCB}}) = 
\begin{cases}
O\left(\left(\frac{1}{\epsilon}\right)^{2+d/\nu}\right), & \text{for Matérn kernels,} \\
O\left(\frac{1}{\epsilon^2} \cdot \left(\log\frac{1}{\epsilon}\right)^{d+3}\right), & \text{for SE kernels.}
\end{cases}
\end{align}
\cite{wang2023regret}'s approach to proving the regret optimality of GP-UCB follows the two-step framework. The first critical component is to construct a high-probability uniform error bound that quantifies the difference between the true objective function $f(x)$ and its posterior mean estimate $\mu_t(x)$ at any given step $t$. This bound takes the form $|f(x) - \mu_t(x)| \le \sqrt{\beta_t} \sigma_t(x)$ for all $x \in \mathcal{D}$, where $\sigma_t(x)$ is the posterior standard deviation and $\beta_t$ is a carefully chosen exploration parameter. The second component involves bounding the cumulative sum of the instantaneous regrets, which, under the uniform error bound, can be related to the sum of the posterior variances $\sum_{t=1}^T \sigma_{t-1}^2(x_t)$. This sum is, in turn, bounded by the maximal information gain, $\gamma_T$. The final regret bound is thus a function of $T$, $\beta_T$, and an upper bound on $\gamma_T$. \cite{wang2023regret} employs tools from empirical process theory and decomposes the estimation error, $f(x) - \mu_t(x)$, into a bias term and a random error term. By leveraging the connection between Gaussian process regression and kernel ridge regression, along with properties of the Fourier transform for stationary kernels, they show that the bias term is bounded by $||f||_{\mathcal{N}_{\Psi}(\mathcal{D})} \sigma_t(x)$. The more challenging random error term is handled by viewing it as an empirical process indexed by a class of functions and bounding its supremum. They bound the $\epsilon$-entropy of this function class, which allows for a high-probability bound on the random error. With the new uniform error bound, we can select a much smaller, dimension-independent exploration parameter $\beta_t$ that grows only logarithmically with $T$. Combining this sharper $\beta_t$ with the tightest known bounds on the maximal information gain $\gamma_T$ for Matérn and SE kernels, we can achieve the tight cumulative regret of GP-UCB.

%% file: initial_gen_prompt.tex
\begin{figure}[ht]
\centering
\begin{tcolorbox}[
    colback=promptbg!100!white,
    sharp corners=south, 
    boxrule=0.25mm, 
    fonttitle=\bfseries, 
    title={\textbf{Meta-prompt used for generating prompts for initial 64 ads generation}}, 
    width=\textwidth, 
    enhanced,
    drop shadow
    ]
\scriptsize
\setlength{\parskip}{0pt}
\setlength{\itemsep}{0pt}
\textbf{CREATIVE BRIEF:}\\
\noindent \texttt{\{creative$\_$brief\}}
\medskip
\\
Based on this creative brief, generate a creative, structured, and descriptive prompt for a generative AI model (e.g., Imagen) that will produce a brand-aligned and scroll-stopping advertisement image suitable for an Instagram feed.\\
A successful prompt must be constructed using the following components and principles:\\
0. (Most important!) Prompt NEVER includes any description of kids. Kids are not allowed to generate in Imagen 4. Even if creative\_brief includes description of "family" or "kids", never include kids in the prompt. Only create adults if humans are included.\\
1. Key Message (The "Why"): This is the foundational element that guides all other components. It defines the core idea or feeling the ad must communicate. Before writing the rest of the prompt, clearly articulate the message the ad will deliver. This message will act as the "North Star" for all subsequent creative choices. It may be desirable to put the message as the text overlay.\\
2. Core Components (The "What"):\\
- This should completely depend on the key message.\\
- Scene \& Environment: Based on the key message, establish a clear, relatable setting that aligns with the brand's lifestyle appeal.\\
- Action \& Narrative: Based on the key message and the scene, describe a dynamic but clear action or interaction to create a sense of a captured moment and tell a micro-story.\\
- Composition \& Framing: Specify the camera shot, angle, and framing (e.g., "low-angle shot," "dynamic medium shot," "close-up on the shoe").\\
3. Stylistic Qualities (The "How"):\\
- This should completely depend on the key message and the scene.\\
- Photography Style: Define the overall aesthetic (e.g., "photorealistic," "cinematic," "professional product photography," "lifestyle action shot").\\
- Lighting: Be specific about the lighting to set the mood (e.g., "warm golden hour light," "bright morning sun," "dramatic side-lighting").\\
- Color Palette \& Tone: Guide the color scheme and emotional feel (e.g., "vibrant and energetic colors," "empowering and motivational tones," "clean and modern palette").\\
- Atmosphere \& Feeling: Aim to evoke a specific feeling aligned with the brand (e.g., "a feeling of effortless performance," "an atmosphere of vibrant energy," "a sense of supreme comfort").\\
4. What to Avoid:\\
- Vagueness: Use specific, descriptive terms instead of "nice" or "good."\\
- Contradictory Elements: Ensure all elements work together harmoniously.\\
- Over-stuffing: Focus on a single, clear message without too many competing objects.\\
5. What to include:\\
- Logo: create a logo based on creative brief, and naturally place it.\\
- Text: If you are including a text message, prompt for "negative space for text overlay".
\end{tcolorbox}

\caption{Meta-prompt for generating the prompts for the initial 64 ads.}
\label{fig:init_gen_prompt}
\end{figure}

%% file: image_gen_prompts.tex
\begin{figure}[ht]
\centering
\begin{tcolorbox}[
    colback=promptbg!100!white,
    sharp corners=south, 
    boxrule=0.25mm, 
    fonttitle=\bfseries, 
    title={\textbf{Worst-of-64 prompt for GreenBite (plant-based burger patty) ad scenario.}}, 
    width=\textwidth, 
    enhanced,
    drop shadow
    ]
\scriptsize
\setlength{\parskip}{0pt}
\setlength{\itemsep}{0pt}
**Key Message:** To visually communicate the core idea: "Experience the juicy, no-compromise satisfaction of a real burger, made better for you and the planet."\\
**Image Generation Prompt:** A dynamic, commercial lifestyle photograph for an Instagram feed ad.\\
**Core Components:**\\
* **Scene \& Environment:** A stylish, modern outdoor patio during a late afternoon barbecue. The background is softly blurred, showing contemporary patio furniture, lush green potted plants, and the warm glow of string lights just beginning to turn on.\\
* **Subject \& Action:** The focal point is a vibrant, stylish woman in her early thirties, captured in a candid moment of pure satisfaction. She is taking a bite of a huge, delicious-looking plant-based burger. Her eyes are closed in bliss, and a small, genuine smile plays on her lips, conveying an authentic "this is amazing" reaction. Juice from the burger is subtly visible on her hand to emphasize its succulence. She is dressed in a casual but chic linen shirt.\\
* **Product:** The GreenBite burger is the hero. It's a thick, perfectly seared patty with grill marks, nestled in a glossy brioche bun. It’s layered with glistening melted cheddar cheese, crisp green lettuce, a thick slice of ripe red tomato, and caramelized onions spilling out slightly. The burger is partially wrapped in a simple, brown kraft paper sheet which has a small, clean logo on it: "GreenBite" in a modern sans-serif font next to a stylized green leaf.\\
* **Composition \& Framing:** A medium close-up shot, angled slightly upwards to make the moment feel triumphant and empowering. The camera uses a shallow depth of field, keeping the woman’s expression and the burger in sharp focus while the background is beautifully out of focus (bokeh). The composition is balanced, with a clear area of soft background on the upper right side, creating negative space for a text overlay.\\
**Stylistic Qualities:**\\
* **Photography Style:** High-end, photorealistic commercial food photography with a warm, authentic lifestyle feel. The image should be incredibly sharp and detailed, especially the texture of the burger patty and the fresh ingredients.\\
* **Lighting:** Bathed in the warm, dreamy glow of the golden hour. The sunlight comes from the side, creating soft highlights and long, gentle shadows, emphasizing the textures of the food and giving her skin a radiant look.\\
* **Color Palette \& Tone:** A rich and mouth-watering color palette. Saturated and vibrant, but natural. Warm tones of orange, gold, and red from the sun and the food are balanced by the deep greens of the background foliage and the crispness of the lettuce.\\
* **Atmosphere \& Feeling:** The image radiates a feeling of guilt-free indulgence, vibrant well-being, and ultimate satisfaction. The atmosphere is confident, positive, and irresistibly appetizing, making the viewer crave that exact experience.\\

\end{tcolorbox}

\begin{tcolorbox}[
    colback=promptbg!100!white,
    sharp corners=south, 
    boxrule=0.25mm, 
    fonttitle=\bfseries, 
    title={\textbf{Best-of-64 prompt for GreenBite (plant-based burger patty) ad scenario.}}, 
    width=\textwidth, 
    enhanced,
    drop shadow
    ]
\scriptsize
\setlength{\parskip}{0pt}
\setlength{\itemsep}{0pt}
**1. Key Message (The "Why"):** To visually capture the peak moment of guilt-free indulgence, proving that a plant-based burger can deliver an unbelievably juicy, satisfying, and no-compromise flavor experience identical to a classic beef burger. The text overlay will read: "The No-Compromise Burger."\\
**2. Core Components (The "What"):**\\
* **Scene \& Environment:** A vibrant, sun-drenched backyard barbecue scene, shot with a very shallow depth of field. The background is a beautiful, soft-focus bokeh of friends laughing (adults only, 25-40 age range), string lights, and lush greenery, suggesting a warm, sociable atmosphere without distracting from the main subject.\\
* **Action \& Narrative:** A stylish woman in her early 30s, with a look of pure, blissful satisfaction, is taking the first, epic bite of a burger. Her eyes are closed in enjoyment. Her hands grip the burger firmly but gently. A single, perfect drop of sauce is starting to fall from the side of the burger, caught in mid-air, emphasizing its juiciness.\\
* **Composition \& Framing:** A dramatic, macro close-up, food photography style. The shot is framed tightly on the lower half of the woman's face and her hands holding the burger, making the burger the undeniable hero. The angle is slightly low to give the burger an epic, monumental feel. The right side of the frame has clean, out-of-focus background elements, creating intentional negative space for a text overlay.\\
**3. Stylistic Qualities (The "How"):**\\
* **Photography Style:** Hyper-realistic, professional commercial food photography. Every detail is rendered with extreme clarity: the char marks on the patty, the condensation on the lettuce, the glistening melt of the plant-based cheese, and the texture of the toasted brioche bun. Shot with a high-end DSLR camera and a macro lens.\\
* **Lighting:** Warm, radiant golden hour lighting. The sun acts as a dramatic backlight, creating a halo effect around the burger and the woman’s hands, and creating glistening specular highlights on the juicy patty and sauce. The lighting is bright, positive, and makes the food look incredibly appetizing.\\
* **Color Palette \& Tone:** A mouth-watering and vibrant palette. Rich, savory browns of the seared "GreenBite" patty, deep reds of a ripe tomato slice, vibrant green of crisp lettuce, and a warm golden-yellow of the melted cheese and toasted bun. The overall tone is confident, delicious, and energetic.\\
* **Atmosphere \& Feeling:** The atmosphere is one of ultimate satisfaction and pure, unadulterated foodie bliss. It feels warm, authentic, and completely satisfying, eliminating any doubt about the quality and taste of the plant-based option.\\
**4. Inclusions \& Details:**\\
* **The Burger Details:** The GreenBite patty is thick, with a perfect, glistening sear and visible char marks, looking indistinguishable from high-quality ground beef. It's topped with melting vegan cheddar cheese, a leaf of crisp butter lettuce, a thick slice of heirloom tomato, and a creamy aioli sauce.\\
* **Logo Placement:** The burger is partially wrapped at the base with a small square of branded, unbleached butcher paper. The "GreenBite" logo—a simple, modern green leaf icon next to the word "GreenBite" in a clean sans-serif font—is subtly visible on the paper.\\
* **Negative Space:** Ensure the composition leaves a clean, uncluttered area of soft-focus background on the right third of the image for text overlay.\\

\end{tcolorbox}

\caption{Prompts for the best and worst ad among the initial 64 ads generated using the meta-prompt in Figure \ref{fig:init_gen_prompt}.}
\label{fig:best_ad_prompt}
\end{figure}

%% file: Pairwise_prompt.tex
\begin{figure}[ht]
\centering
\begin{tcolorbox}[
    colback=promptbg!100!white,
    sharp corners=south, 
    boxrule=0.25mm, 
    fonttitle=\bfseries, 
    title={\textbf{Best-of-N Gradient steps' pairwise tournament prompt for ads.}}, 
    width=\textwidth, 
    enhanced,
    drop shadow
    ]
\scriptsize
\setlength{\parskip}{0pt}
\setlength{\itemsep}{0pt}
contents = [image1, image2, comparison$\_$prompt]
\\
comparison$\_$prompt = 
\\``You are evaluating two advertisement images for mobile Instagram ads.
Which image would be more effective at engaging users and driving clicks?

RESPOND WITH ONLY THE NUMBER 1 OR 2. NOTHING ELSE.
\\
1 = first image is better
\\
2 = second image is better
''
\end{tcolorbox}

\caption{$M_{\text{critic}}$'s pairwise tournament prompt for digital marketing example.}
\label{fig:tornament_prompt}
\end{figure}

%% file: reflection_prompt.tex
\begin{figure}[ht]
\centering
\begin{tcolorbox}[
    colback=promptbg!100!white,
    sharp corners=south, 
    boxrule=0.25mm, 
    fonttitle=\bfseries, 
    title={\textbf{Meta reflection prompt for ads}}, 
    width=\textwidth, 
    enhanced,
    drop shadow
    ]
\scriptsize
\setlength{\parskip}{0pt}
\setlength{\itemsep}{0pt}

You are an expert at analyzing visual patterns in advertising performance.

\medskip

{VISUAL ANALYSIS TASK:}\\
I will show you images from the lowest-scoring and highest-scoring ad iterations. Your task is to identify specific visual patterns that distinguish effective from ineffective ads.
\medskip

{VISUAL EXAMPLES -- BEST VS WORST PERFORMING:}

\medskip

\noindent {RANK 1/10 WORST PERFORMING (Score: 2.14/5.0):}\\
\noindent {Prompt excerpt: <prompt>...}\\
\noindent [Image]

\medskip

\noindent {RANK 2/10 LOWER HALF (Score: 2.52/5.0):}\\
\noindent {Prompt excerpt: <prompt>...}\\
\noindent [Image]

\medskip
$\ldots$
\medskip

\noindent {RANK 9/10 UPPER HALF (Score: 3.81/5.0):}\\
\noindent {Prompt excerpt: <prompt>...}\\
\noindent [Image]

\medskip

\noindent {RANK 10/10 BEST PERFORMING (Score: 4.15/5.0):}\\
\noindent {Prompt excerpt: <prompt>...}\\
\noindent [Image]

\medskip

Based on your visual analysis, identify patterns that correlate with higher effectiveness scores:
\begin{enumerate}[leftmargin=*, itemsep=0pt, topsep=2pt]
\item Visual composition and framing differences
\item Lighting conditions and mood variations
\item Color palettes and visual tone patterns
\item Subject positioning and action effectiveness
\item Brand integration approaches
\item Environmental and atmospheric elements
\end{enumerate}

\medskip

RESPONSE FORMAT:\\
Provide a structured analysis of visual patterns observed, focusing on what distinguishes high-performing from low-performing ads.

\end{tcolorbox}

\caption{Meta reflection prompt for digital marketing example.}
\label{fig:reflection_prompt}
\end{figure}

%% file: improvement_prompt.tex

\begin{figure}[ht]
\centering

\begin{tcolorbox}[
    colback=promptbg!100!white,
    sharp corners=south,
    boxrule=0.25mm,
    fonttitle=\bfseries,
    title={\textbf{Textual gradient generating meta-prompt for ads (at $t>0$)}},
    width=\textwidth,
    enhanced,
    drop shadow
]
\scriptsize
\setlength{\parskip}{0pt}
\setlength{\itemsep}{0pt}

You are an expert at optimizing image generation prompts for advertising effectiveness.

\medskip

\textbf{CURRENT PROMPT TO IMPROVE:}\\
\noindent \texttt{\{current\_prompt\}}

\medskip

\textbf{PERFORMANCE ANALYSIS FROM PREVIOUS ITERATIONS:}\\
\noindent \texttt{\{performance\_analysis\}}

\medskip

\textbf{TASK:} Based on the performance analysis above, generate specific, actionable improvements to make the current prompt more effective.

Focus on implementing the successful visual patterns identified in the analysis while avoiding the ineffective elements.

\medskip

Provide \textbf{3--5} specific, implementable suggestions for improvement. (It can be addition of a new prompt part, deletion of existing prompt part, or rewriting a prompt part).
Each suggestion should reference insights from the performance analysis.

\medskip

\textbf{RESPONSE FORMAT:}
\begin{enumerate}[leftmargin=*, itemsep=0pt, topsep=2pt]
\item \text{[Specific improvement suggestion based on performance analysis]}
\item \text{[Specific improvement suggestion based on performance analysis]}
\item $\ldots$
\end{enumerate}

Be concrete and actionable. Focus on changes that will meaningfully improve ad effectiveness based on the analysis.

\end{tcolorbox}

\begin{tcolorbox}[
    colback=promptbg!100!white,
    sharp corners=south,
    boxrule=0.25mm,
    fonttitle=\bfseries,
    title={\textbf{Textual gradient generating meta-prompt for ads (at $t=0$)}},
    width=\textwidth,
    enhanced,
    drop shadow
]
\scriptsize
\setlength{\parskip}{0pt}
\setlength{\itemsep}{0pt}

You are an expert at optimizing image generation prompts for advertising effectiveness.

\medskip

\textbf{CURRENT PROMPT TO IMPROVE:}\\
\noindent \texttt{\{current\_prompt\}}

\medskip

\textbf{TASK:} Generate specific, actionable improvements to make this ad prompt more effective.
Focus on elements that will increase engagement and appeal to the target audience.

\medskip

\textbf{Consider improvements in:}
\begin{enumerate}[leftmargin=*, itemsep=0pt, topsep=2pt]
\item Visual composition and framing
\item Emotional appeal and messaging
\item Color palette and lighting
\item Subject positioning and action
\item Brand integration and logo placement
\item Text overlay space and readability
\end{enumerate}

\medskip

Provide \textbf{3--5} specific, implementable suggestions for improvement.
Each suggestion should be concrete and actionable.

\medskip

\textbf{RESPONSE FORMAT:}
\begin{enumerate}[leftmargin=*, itemsep=0pt, topsep=2pt]
\item \text{[Specific improvement suggestion based on performance analysis]}
\item \text{[Specific improvement suggestion based on performance analysis]}
\item $\ldots$
\end{enumerate}

Be concise but specific. Focus on changes that will meaningfully improve ad effectiveness.

\end{tcolorbox}

\caption{Meta prompts for improving ad image-generation prompts (with and without performance analysis).}
\label{fig:prompt_improvement_prompts}
\end{figure}

%% file: gradient_to_prompt.tex

\begin{figure}[ht]
\centering

\begin{tcolorbox}[
    colback=promptbg!100!white,
    sharp corners=south,
    boxrule=0.25mm,
    fonttitle=\bfseries,
    title={\textbf{Meta-prompt for applying textual gradients}},
    width=\textwidth,
    enhanced,
    drop shadow
]
\scriptsize
\setlength{\parskip}{0pt}
\setlength{\itemsep}{0pt}

You are an expert at revising image generation prompts based on improvement suggestions.

\medskip

\textbf{ORIGINAL PROMPT:}\\
\noindent \texttt{\{current\_prompt\}}

\medskip

\textbf{IMPROVEMENT SUGGESTIONS:}\\
\noindent \texttt{\{gradient\}}

\medskip

\textbf{TASK:} Rewrite the prompt incorporating the improvement suggestions while maintaining the core message and structure.

\medskip

\textbf{REQUIREMENTS:}
\begin{itemize}[leftmargin=*, itemsep=0pt, topsep=2pt]
\item Keep the prompt structure and format similar to the original
\item Integrate the improvement suggestions naturally
\item Maintain coherence and readability
\item Ensure the prompt is optimized for image generation
\item Keep the prompt length reasonable (not too long)
\end{itemize}

\medskip

\textbf{OUTPUT CONSTRAINT:}\\
Return ONLY the revised prompt, no explanations or additional text.

\end{tcolorbox}

\caption{Meta-prompt for applying improvement suggestions to revise an image-generation ad prompt.}
\label{fig:apply_prompt}
\end{figure}

%% file: persona_prompt.tex
\begin{figure}[!ht]
\centering
\begin{tcolorbox}[
    colback=promptbg!100!white,
    sharp corners=south, 
    boxrule=0.25mm, 
    fonttitle=\bfseries, 
    title={\textbf{Persona prompt for simulating ad effectiveness.}}, 
    width=\textwidth, 
    enhanced,
    drop shadow
    ]
\scriptsize
\setlength{\parskip}{0pt}
\setlength{\itemsep}{0pt}

SYSTEM: $\{$
\\
You are an AI assistant. Your task is to answer the TASK as if you are the individual described in the `Persona Profile' (which contains their past survey responses). Remain consistent with the persona's past survey responses and stated characteristics. Carefully follow any instructions provided for the new question, including formatting requirements.
\\
\}

PERSONA DATA: $\{$
\\
Which part of the United States do you currently live in?
\\
Question Type: Single Choice
\\
Options:
\\
1 - Northeast (PA, NY, NJ, RI, CT, MA, VT, NH, ME)
\\
2 - Midwest (ND, SD, NE, KS, MN, IA, MO, WI, IL, MI, IN, OH)
\\
3 - South (TX, OK, AR, LA, KY, TN, MS, AL, WV, DC, MD, DE, VA, NC, SC, GA, FL)
\\
4 - West (WA, OR, ID, MT, WY, CA, NV, UT, CO, AZ, NM)
\\
5 - Pacific (HI, AK)
\\
Answer: 2 - Midwest (ND, SD, NE, KS, MN, IA, MO, WI, IL, MI, IN, OH)
\\
What is the highest level of schooling or degree that you have completed?
\\
Question Type: Single Choice
\\
Options:
\\
1 - Less than high school
\\
2 - High school graduate
\\
3 - Some college, no degree
\\
4 - Associate's degree
\\
5 - College graduate/some postgrad
\\
6 - Postgraduate
\\
Answer: 3 - Some college, no degree
\\
$\ldots$ (Many other survey questions and answers) $\ldots$
\\
Suppose you were given \$5 and had to offer to another (anonymous) person a way to split the money. The other person can either accept or reject your offer. If the
other person accepts your offer, you would each receive the amount you proposed. If the other person rejects your offer, you would both receive 
\$0. How much would you
 offer to the other person?
\\
Question Type: Single Choice
\\
Options:
\\
1 - \$0
\\
2 - \$1
\\
3 - \$2
\\
4 - \$3
\\
5 - \$4
\\
6 - \$5
\\
Answer: 3 - \$2
\\
$\ldots$ (Many other survey questions and answers) $\ldots$
\\
$\}$
\medskip
\\
AD IMAGE: [image]
\medskip
\\
TASK:

Return only one item from ["1","2","3","4","5"] for ad effectiveness.\\
Effective Score Scale Definition:\\
1: Extremely Unlikely. The persona would actively ignore or be annoyed by this ad.\\
2: Unlikely. The persona would likely scroll past without a second thought.
\\
3: Mediocre. It is hard to decide whether the personal would click or don't click.\\
4: Likely. The persona is intrigued and has a good chance of clicking to learn more.\\
5: Extremely Likely. The persona is the ideal target; a click is almost certain.\\
No explanation. Just the score.
\end{tcolorbox}

\caption{Prompt for simulating the effectiveness of a given ad-persona combination.}
\label{fig:persona_prompt}
\end{figure}

%% file: references.bib
@article{jiang2024many,
  title={Many-shot in-context learning in multimodal foundation models},
  author={Jiang, Yixing and Irvin, Jeremy and Wang, Ji Hun and Chaudhry, Muhammad Ahmed and Chen, Jonathan H and Ng, Andrew Y},
  journal={arXiv preprint arXiv:2405.09798},
  year={2024}
}

@misc{kang2025llmpersonassubstitutefield,
      title={LLM Personas as a Substitute for Field Experiments in Method Benchmarking}, 
      author={Enoch Hyunwook Kang},
      year={2025},
      eprint={2512.21080},
      archivePrefix={arXiv},
      primaryClass={cs.AI},
      url={https://arxiv.org/abs/2512.21080}, 
}

@article{madaan2024quantifying,
  title={Quantifying variance in evaluation benchmarks},
  author={Madaan, Lovish and Singh, Aaditya K and Schaeffer, Rylan and Poulton, Andrew and Koyejo, Sanmi and Stenetorp, Pontus and Narang, Sharan and Hupkes, Dieuwke},
  journal={arXiv preprint arXiv:2406.10229},
  year={2024}
}

@article{robertson2009probabilistic,
  title={The probabilistic relevance framework: BM25 and beyond},
  author={Robertson, Stephen and Zaragoza, Hugo and others},
  journal={Foundations and Trends{\textregistered} in Information Retrieval},
  volume={3},
  number={4},
  pages={333--389},
  year={2009},
  publisher={Now Publishers, Inc.}
}

@article{calderon2025alternative,
  title={The alternative annotator test for llm-as-a-judge: How to statistically justify replacing human annotators with llms},
  author={Calderon, Nitay and Reichart, Roi and Dror, Rotem},
  journal={arXiv preprint arXiv:2501.10970},
  year={2025}
}

@article{johnson2023inferno,
  title={Inferno: A guide to field experiments in online display advertising},
  author={Johnson, Garrett A},
  journal={Journal of economics \& management strategy},
  volume={32},
  number={3},
  pages={469--490},
  year={2023},
  publisher={Wiley Online Library}
}

@inproceedings{fiez_etal_2024,
  title={Best of three worlds: Adaptive experimentation for digital marketing in practice},
  author={Fiez, Tanner and Nassif, Houssam and Chen, Yu-Cheng and Gamez, Sergio and Jain, Lalit},
  booktitle={Proceedings of the ACM Web Conference 2024},
  pages={3586--3597},
  year={2024}
}

@article{guo2024ds,
  title={Ds-agent: Automated data science by empowering large language models with case-based reasoning},
  author={Guo, Siyuan and Deng, Cheng and Wen, Ying and Chen, Hechang and Chang, Yi and Wang, Jun},
  journal={arXiv preprint arXiv:2402.17453},
  year={2024}
}

@article{jansen2024automated,
  title={Automated Alignment: Engaging Customers with Visual Generative AI},
  author={Jansen, Tijmen and Heitmann, Mark and Reisenbichler, Martin and Schweidel, David A},
  journal={Available at SSRN 4656622},
  year={2024}
}

@article{ghosh2024exploring,
  title={Exploring the frontier of vision-language models: A survey of current methodologies and future directions},
  author={Ghosh, Akash and Acharya, Arkadeep and Saha, Sriparna and Jain, Vinija and Chadha, Aman},
  journal={arXiv preprint arXiv:2404.07214},
  year={2024}
}

@misc{wiedemer2025videomodelszeroshotlearners,
      title={Video models are zero-shot learners and reasoners}, 
      author={Thaddäus Wiedemer and Yuxuan Li and Paul Vicol and Shixiang Shane Gu and Nick Matarese and Kevin Swersky and Been Kim and Priyank Jaini and Robert Geirhos},
      year={2025},
      eprint={2509.20328},
      archivePrefix={arXiv},
      primaryClass={cs.LG},
      url={https://arxiv.org/abs/2509.20328}, 
}

@article{ferber2024context,
  title={In-context learning enables multimodal large language models to classify cancer pathology images},
  author={Ferber, Dyke and W{\"o}lflein, Georg and Wiest, Isabella C and Ligero, Marta and Sainath, Srividhya and Ghaffari Laleh, Narmin and El Nahhas, Omar SM and M{\"u}ller-Franzes, Gustav and J{\"a}ger, Dirk and Truhn, Daniel and others},
  journal={Nature Communications},
  volume={15},
  number={1},
  pages={10104},
  year={2024},
  publisher={Nature Publishing Group UK London}
}

@article{ba2022advertising,
  title={Advertising Media and Target Audience Optimization via High-dimensional Bandits},
  author={Ba, Wenjia and Harrison, J Michael and Nair, Harikesh S},
  journal={arXiv preprint arXiv:2209.08403},
  year={2022}
}

@article{zhang2025agentic,
  title={Agentic Context Engineering: Evolving Contexts for Self-Improving Language Models},
  author={Zhang, Qizheng and Hu, Changran and Upasani, Shubhangi and Ma, Boyuan and Hong, Fenglu and Kamanuru, Vamsidhar and Rainton, Jay and Wu, Chen and Ji, Mengmeng and Li, Hanchen and others},
  journal={arXiv preprint arXiv:2510.04618},
  year={2025}
}

@article{khattab2023dspy,
  title={Dspy: Compiling declarative language model calls into self-improving pipelines},
  author={Khattab, Omar and Singhvi, Arnav and Maheshwari, Paridhi and Zhang, Zhiyuan and Santhanam, Keshav and Vardhamanan, Sri and Haq, Saiful and Sharma, Ashutosh and Joshi, Thomas T and Moazam, Hanna and others},
  journal={arXiv preprint arXiv:2310.03714},
  year={2023}
}

@article{nandy2021b,
  title={A/b testing for recommender systems in a two-sided marketplace},
  author={Nandy, Preetam and Venugopalan, Divya and Lo, Chun and Chatterjee, Shaunak},
  journal={Advances in Neural Information Processing Systems},
  volume={34},
  pages={6466--6477},
  year={2021}
}

@book{kohavi2020trustworthy,
  title={Trustworthy online controlled experiments: A practical guide to a/b testing},
  author={Kohavi, Ron and Tang, Diane and Xu, Ya},
  year={2020},
  publisher={Cambridge University Press}
}

@article{hartmann2025power,
  title={The power of generative marketing: Can generative AI create superhuman visual marketing content?},
  author={Hartmann, Jochen and Exner, Yannick and Domdey, Samuel},
  journal={International Journal of Research in Marketing},
  volume={42},
  number={1},
  pages={13--31},
  year={2025},
  publisher={Elsevier}
}

@article{wang2025maestro,
  title={Maestro: Joint Graph \& Config Optimization for Reliable AI Agents},
  author={Wang, Wenxiao and Kattakinda, Priyatham and Feizi, Soheil},
  journal={arXiv preprint arXiv:2509.04642},
  year={2025}
}

@article{papenmeier2025understanding,
  title={Understanding high-dimensional bayesian optimization},
  author={Papenmeier, Leonard and Poloczek, Matthias and Nardi, Luigi},
  journal={arXiv preprint arXiv:2502.09198},
  year={2025}
}

@article{coffee2025aiadbuying,
author = {Coffee, Patrick},
title = {AI Will Soon Dominate Ad Buying, Whether Marketers Like It or Not},
journal = {The Wall Street Journal},
year = {2025},
month = {March},
day = {6},
url = {https://www.wsj.com/articles/ai-will-soon-dominate-ad-buying-whether-marketers-like-it-or-not-3d62b754},
urldate = {2025-09-07}
}

@article{agrawal2025gepa,
  title={GEPA: Reflective Prompt Evolution Can Outperform Reinforcement Learning},
  author={Agrawal, Lakshya A and Tan, Shangyin and Soylu, Dilara and Ziems, Noah and Khare, Rishi and Opsahl-Ong, Krista and Singhvi, Arnav and Shandilya, Herumb and Ryan, Michael J and Jiang, Meng and others},
  journal={arXiv preprint arXiv:2507.19457},
  year={2025}
}

@article{liu2024deepseek,
  title={Deepseek-v3 technical report},
  author={Liu, Aixin and Feng, Bei and Xue, Bing and Wang, Bingxuan and Wu, Bochao and Lu, Chengda and Zhao, Chenggang and Deng, Chengqi and Zhang, Chenyu and Ruan, Chong and others},
  journal={arXiv preprint arXiv:2412.19437},
  year={2024}
}

@inproceedings{scarlett2017lower,
  title={Lower bounds on regret for noisy gaussian process bandit optimization},
  author={Scarlett, Jonathan and Bogunovic, Ilija and Cevher, Volkan},
  booktitle={Conference on Learning Theory},
  pages={1723--1742},
  year={2017},
  organization={PMLR}
}

@article{srinivas2012information,
  title={Information-theoretic regret bounds for gaussian process optimization in the bandit setting},
  author={Srinivas, Niranjan and Krause, Andreas and Kakade, Sham M and Seeger, Matthias W},
  journal={IEEE transactions on information theory},
  volume={58},
  number={5},
  pages={3250--3265},
  year={2012},
  publisher={IEEE}
}

@article{whitehouse2023sublinear,
  title={On the sublinear regret of GP-UCB},
  author={Whitehouse, Justin and Ramdas, Aaditya and Wu, Steven Z},
  journal={Advances in Neural Information Processing Systems},
  volume={36},
  pages={35266--35276},
  year={2023}
}

@article{rutz_etal_2017,
  title={A new method to aid copy testing of paid search text advertisements},
  author={Rutz, Oliver J and Sonnier, Garrett P and Trusov, Michael},
  journal={Journal of Marketing Research},
  volume={54},
  number={6},
  pages={885--900},
  year={2017},
  publisher={SAGE Publications Sage CA: Los Angeles, CA}
}

@article{macinnis_etal_2002,
  title={Assessing when increased media weight of real-world advertisements helps sales},
  author={MacInnis, Deborah J and Rao, Ambar G and Weiss, Allen M},
  journal={Journal of Marketing Research},
  volume={39},
  number={4},
  pages={391--407},
  year={2002},
  publisher={SAGE Publications Sage CA: Los Angeles, CA}
}

@misc{biswas_2025,
  title={Investigating the effects of including discount information in advertising},
  author={Biswas, Shirsho},
  year={2025},
  howpublished={Working Paper}
}

@article{rosengren2020meta,
  title={A meta-analysis of when and how advertising creativity works},
  author={Rosengren, Sara and Eisend, Martin and Koslow, Scott and Dahlen, Micael},
  journal={Journal of Marketing},
  volume={84},
  number={6},
  pages={39--56},
  year={2020},
  publisher={SAGE Publications Sage CA: Los Angeles, CA}
}

@article{li2025llm,
  title={LLM Generated Persona is a Promise with a Catch},
  author={Li, Ang and Chen, Haozhe and Namkoong, Hongseok and Peng, Tianyi},
  journal={arXiv preprint arXiv:2503.16527},
  year={2025}
}

@article{lewis2015unfavorable,
  title={The unfavorable economics of measuring the returns to advertising},
  author={Lewis, Randall A and Rao, Justin M},
  journal={The Quarterly Journal of Economics},
  volume={130},
  number={4},
  pages={1941--1973},
  year={2015},
  publisher={MIT Press}
}

@article{peng2025mega,
  title={A Mega-Study of Digital Twins Reveals Strengths, Weaknesses and Opportunities for Further Improvement},
  author={Peng, Tiany and Gui, George and Merlau, Daniel J and Fan, Grace Jiarui and Sliman, Malek Ben and Brucks, Melanie and Johnson, Eric J and Morwitz, Vicki and Althenayyan, Abdullah and Bellezza, Silvia and others},
  journal={arXiv preprint arXiv:2509.19088},
  year={2025}
}

@article{cisse2025language,
  title={Language-Based Bayesian Optimization Research Assistant (BORA)},
  author={Ciss{\'e}, Abdoulatif and Evangelopoulos, Xenophon and Gusev, Vladimir V and Cooper, Andrew I},
  journal={arXiv preprint arXiv:2501.16224},
  year={2025}
}

@article{hauser2009website,
  title={Website morphing},
  author={Hauser, John R and Urban, Glen L and Liberali, Guilherme and Braun, Michael},
  journal={Marketing Science},
  volume={28},
  number={2},
  pages={202--223},
  year={2009},
  publisher={INFORMS}
}

@article{frazier2018tutorial,
  title={A tutorial on Bayesian optimization},
  author={Frazier, Peter I},
  journal={arXiv preprint arXiv:1807.02811},
  year={2018}
}

@article{wilson2018maximizing,
  title={Maximizing acquisition functions for Bayesian optimization},
  author={Wilson, James and Hutter, Frank and Deisenroth, Marc},
  journal={Advances in neural information processing systems},
  volume={31},
  year={2018}
}

@article{snell2024scaling,
  title={Scaling llm test-time compute optimally can be more effective than scaling model parameters},
  author={Snell, Charlie and Lee, Jaehoon and Xu, Kelvin and Kumar, Aviral},
  journal={arXiv preprint arXiv:2408.03314},
  year={2024}
}

@misc{liu2025pairjudgermperformbestofn,
      title={PairJudge RM: Perform Best-of-N Sampling with Knockout Tournament}, 
      author={Yantao Liu and Zijun Yao and Rui Min and Yixin Cao and Lei Hou and Juanzi Li},
      year={2025},
      eprint={2501.13007},
      archivePrefix={arXiv},
      primaryClass={cs.CL},
      url={https://arxiv.org/abs/2501.13007}, 
}

@techreport{hong2025context,
  title = {Context Rot: How Increasing Input Tokens Impacts LLM Performance},
  author = {Hong, Kelly and Troynikov, Anton and Huber, Jeff},
  year = {2025},
  month = {July},
  institution = {Chroma},
  url = {https://research.trychroma.com/context-rot},
}

@article{he2025position,
  title={Position: Beyond Euclidean--Foundation Models Should Embrace Non-Euclidean Geometries},
  author={He, Neil and Liu, Jiahong and Zhang, Buze and Bui, Ngoc and Maatouk, Ali and Yang, Menglin and King, Irwin and Weber, Melanie and Ying, Rex},
  journal={arXiv preprint arXiv:2504.08896},
  year={2025}
}

@article{pryzant2023automatic,
  title={Automatic prompt optimization with" gradient descent" and beam search},
  author={Pryzant, Reid and Iter, Dan and Li, Jerry and Lee, Yin Tat and Zhu, Chenguang and Zeng, Michael},
  journal={arXiv preprint arXiv:2305.03495},
  year={2023}
}

@article{williams1992simple,
  title={Simple statistical gradient-following algorithms for connectionist reinforcement learning},
  author={Williams, Ronald J},
  journal={Machine learning},
  volume={8},
  number={3},
  pages={229--256},
  year={1992},
  publisher={Springer}
}

@article{reimers2019sentence,
  title={Sentence-bert: Sentence embeddings using siamese bert-networks},
  author={Reimers, Nils and Gurevych, Iryna},
  journal={arXiv preprint arXiv:1908.10084},
  year={2019}
}

@article{goldie2024chip,
  title={That Chip Has Sailed: A Critique of Unfounded Skepticism Around AI for Chip Design},
  author={Goldie, Anna and Mirhoseini, Azalia and Dean, Jeff},
  journal={arXiv preprint arXiv:2411.10053},
  year={2024}
}

@article{jumper2021highly,
  title={Highly accurate protein structure prediction with AlphaFold},
  author={Jumper, John and Evans, Richard and Pritzel, Alexander and Green, Tim and Figurnov, Michael and Ronneberger, Olaf and Tunyasuvunakool, Kathryn and Bates, Russ and {\v{Z}}{\'\i}dek, Augustin and Potapenko, Anna and others},
  journal={nature},
  volume={596},
  number={7873},
  pages={583--589},
  year={2021},
  publisher={Nature Publishing Group UK London}
}

@article{shi2024judging,
  title   = {Judging the Judges: A Systematic Investigation of Position Bias in Pairwise Comparative Assessments by LLMs},
  author  = {Shi, Lin and Ma, Chiyu and Liang, Wenhua and Ma, Weicheng and Vosoughi, Soroush},
  journal = {arXiv},
  year    = {2024},
  note    = {arXiv:2406.07791 [cs.CL]},
  url     = {https://arxiv.org/abs/2406.07791}
}

@article{eisenstein2023helping,
  title={Helping or herding? reward model ensembles mitigate but do not eliminate reward hacking},
  author={Eisenstein, Jacob and Nagpal, Chirag and Agarwal, Alekh and Beirami, Ahmad and D'Amour, Alex and Dvijotham, DJ and Fisch, Adam and Heller, Katherine and Pfohl, Stephen and Ramachandran, Deepak and others},
  journal={arXiv preprint arXiv:2312.09244},
  year={2023}
}

@article{mudgal2023controlled,
  title={Controlled decoding from language models},
  author={Mudgal, Sidharth and Lee, Jong and Ganapathy, Harish and Li, YaGuang and Wang, Tao and Huang, Yanping and Chen, Zhifeng and Cheng, Heng-Tze and Collins, Michael and Strohman, Trevor and others},
  journal={arXiv preprint arXiv:2310.17022},
  year={2023}
}

@inproceedings{gao2023scaling,
  title={Scaling laws for reward model overoptimization},
  author={Gao, Leo and Schulman, John and Hilton, Jacob},
  booktitle={International Conference on Machine Learning},
  pages={10835--10866},
  year={2023},
  organization={PMLR}
}

@book{vershynin2018high,
  title={High-dimensional probability: An introduction with applications in data science},
  author={Vershynin, Roman},
  volume={47},
  year={2018},
  publisher={Cambridge university press}
}

@inproceedings{russo2016simple,
  title={Simple bayesian algorithms for best arm identification},
  author={Russo, Daniel},
  booktitle={Conference on learning theory},
  pages={1417--1418},
  year={2016},
  organization={PMLR}
}

@article{li2025test,
  title={Test-time preference optimization: On-the-fly alignment via iterative textual feedback},
  author={Li, Yafu and Hu, Xuyang and Qu, Xiaoye and Li, Linjie and Cheng, Yu},
  journal={arXiv preprint arXiv:2501.12895},
  year={2025}
}

@book{young1988introduction,
  title={An introduction to Hilbert space},
  author={Young, Nicholas},
  year={1988},
  publisher={Cambridge university press}
}

@article{ouyang2025reasoningbank,
  title={ReasoningBank: Scaling Agent Self-Evolving with Reasoning Memory},
  author={Ouyang, Siru and Yan, Jun and Hsu, I and Chen, Yanfei and Jiang, Ke and Wang, Zifeng and Han, Rujun and Le, Long T and Daruki, Samira and Tang, Xiangru and others},
  journal={arXiv preprint arXiv:2509.25140},
  year={2025}
}

@article{acikgoz2025self,
  title={Self-Improving LLM Agents at Test-Time},
  author={Acikgoz, Emre Can and Qian, Cheng and Ji, Heng and Hakkani-T{\"u}r, Dilek and Tur, Gokhan},
  journal={arXiv preprint arXiv:2510.07841},
  year={2025}
}

@article{hou2025model,
  title={Model context protocol (mcp): Landscape, security threats, and future research directions},
  author={Hou, Xinyi and Zhao, Yanjie and Wang, Shenao and Wang, Haoyu},
  journal={arXiv preprint arXiv:2503.23278},
  year={2025}
}

@article{lee2023rlaif,
  title={Rlaif vs. rlhf: Scaling reinforcement learning from human feedback with ai feedback},
  author={Lee, Harrison and Phatale, Samrat and Mansoor, Hassan and Mesnard, Thomas and Ferret, Johan and Lu, Kellie and Bishop, Colton and Hall, Ethan and Carbune, Victor and Rastogi, Abhinav and others},
  journal={arXiv preprint arXiv:2309.00267},
  year={2023}
}

@inproceedings{huang-etal-2023-large,
    title = "Large Language Models Can Self-Improve",
    author = "Huang, Jiaxin  and
      Gu, Shixiang  and
      Hou, Le  and
      Wu, Yuexin  and
      Wang, Xuezhi  and
      Yu, Hongkun  and
      Han, Jiawei",
    editor = "Bouamor, Houda  and
      Pino, Juan  and
      Bali, Kalika",
    booktitle = "Proceedings of the 2023 Conference on Empirical Methods in Natural Language Processing",
    month = dec,
    year = "2023",
    address = "Singapore",
    publisher = "Association for Computational Linguistics",
    url = "https://aclanthology.org/2023.emnlp-main.67/",
    doi = "10.18653/v1/2023.emnlp-main.67",
    pages = "1051--1068"
}

@article{silver2025welcome,
  title={Welcome to the era of experience},
  author={Silver, David and Sutton, Richard S},
  journal={Google AI},
  volume={1},
  year={2025}
}

@article{nikolaou2025language,
  title={Language Models are Injective and Hence Invertible},
  author={Nikolaou, Giorgos and Mencattini, Tommaso and Crisostomi, Donato and Santilli, Andrea and Panagakis, Yannis and Rodola, Emanuele},
  journal={arXiv preprint arXiv:2510.15511},
  year={2025}
}

@article{toubia2025twin,
  title={Twin-2K-500: A dataset for building digital twins of over 2,000 people based on their answers to over 500 questions},
  author={Toubia, Olivier and Gui, George Z and Peng, Tianyi and Merlau, Daniel J and Li, Ang and Chen, Haozhe},
  journal={arXiv preprint arXiv:2505.17479},
  year={2025}
}

@article{wang2023regret,
  title={Regret optimality of gp-ucb},
  author={Wang, Wenjia and Zhang, Xiaowei and Zou, Lu},
  journal={arXiv preprint arXiv:2312.01386},
  year={2023}
}

@article{
tang2025understanding,
title={Understanding {LLM} Embeddings for Regression},
author={Eric Tang and Bangding Yang and Xingyou Song},
journal={Transactions on Machine Learning Research},
issn={2835-8856},
year={2025},
url={https://openreview.net/forum?id=Wt6Iz5XNIO},
note={}
}

@article{cameron2011robust,
  title={Robust inference with multiway clustering},
  author={Cameron, A Colin and Gelbach, Jonah B and Miller, Douglas L},
  journal={Journal of Business \& Economic Statistics},
  volume={29},
  number={2},
  pages={238--249},
  year={2011},
  publisher={Taylor \& Francis}
}

@article{cameron2015practitioner,
  title={A practitioner’s guide to cluster-robust inference},
  author={Cameron, A Colin and Miller, Douglas L},
  journal={Journal of human resources},
  volume={50},
  number={2},
  pages={317--372},
  year={2015},
  publisher={University of Wisconsin Press}
}

@article{cui2025automatic,
  title={Automatic prompt optimization via heuristic search: A survey},
  author={Cui, Wendi and Zhang, Jiaxin and Li, Zhuohang and Sun, Hao and Lopez, Damien and Das, Kamalika and Malin, Bradley A and Kumar, Sricharan},
  journal={arXiv preprint arXiv:2502.18746},
  year={2025}
}

@article{nguyen2016understanding,
  title={Understanding innovation engines: Automated creativity and improved stochastic optimization via deep learning},
  author={Nguyen, Anh and Yosinski, Jason and Clune, Jeff},
  journal={Evolutionary computation},
  volume={24},
  number={3},
  pages={545--572},
  year={2016},
  publisher={MIT Press One Rogers Street, Cambridge, MA 02142-1209, USA journals-info~…}
}

@article{ding2023quality,
  title={Quality diversity through human feedback: Towards open-ended diversity-driven optimization},
  author={Ding, Li and Zhang, Jenny and Clune, Jeff and Spector, Lee and Lehman, Joel},
  journal={arXiv preprint arXiv:2310.12103},
  year={2023}
}

@article{liu2024evolution,
  title={Evolution of heuristics: Towards efficient automatic algorithm design using large language model},
  author={Liu, Fei and Tong, Xialiang and Yuan, Mingxuan and Lin, Xi and Luo, Fu and Wang, Zhenkun and Lu, Zhichao and Zhang, Qingfu},
  journal={arXiv preprint arXiv:2401.02051},
  year={2024}
}

@article{hu2024automated,
  title={Automated design of agentic systems},
  author={Hu, Shengran and Lu, Cong and Clune, Jeff},
  journal={arXiv preprint arXiv:2408.08435},
  year={2024}
}

@misc{better_together,
      title={Fine-Tuning and Prompt Optimization: Two Great Steps that Work Better Together}, 
      author={Dilara Soylu and Christopher Potts and Omar Khattab},
      year={2024},
      eprint={2407.10930},
      archivePrefix={arXiv},
      primaryClass={cs.CL},
      url={https://arxiv.org/abs/2407.10930}, 
}

@article{textgrad,
  title={Optimizing generative AI by backpropagating language model feedback},
  author={Yuksekgonul, Mert and Bianchi, Federico and Boen, Joseph and Liu, Sheng and Lu, Pan and Huang, Zhi and Guestrin, Carlos and Zou, James},
  journal={Nature},
  volume={639},
  pages={609--616},
  year={2025},
}

@inproceedings{zhou2022large,
title={Large language models are human-level prompt engineers},
author={Zhou, Yongchao and Muresanu, Andrei Ioan and Han, Ziwen and Paster, Keiran and Pitis, Silviu and Chan, Harris and Ba, Jimmy},
booktitle={The eleventh international conference on learning representations},
year={2022}
}

@misc{yang2024largelanguagemodelsoptimizers,
      title={Large Language Models as Optimizers}, 
      author={Chengrun Yang and Xuezhi Wang and Yifeng Lu and Hanxiao Liu and Quoc V. Le and Denny Zhou and Xinyun Chen},
      year={2024},
      eprint={2309.03409},
      archivePrefix={arXiv},
      primaryClass={cs.LG},
      url={https://arxiv.org/abs/2309.03409}, 
}

@inproceedings{
guo2024connecting,
title={Connecting Large Language Models with Evolutionary Algorithms Yields Powerful Prompt Optimizers},
author={Qingyan Guo and Rui Wang and Junliang Guo and Bei Li and Kaitao Song and Xu Tan and Guoqing Liu and Jiang Bian and Yujiu Yang},
booktitle={The Twelfth International Conference on Learning Representations},
year={2024},
url={https://openreview.net/forum?id=ZG3RaNIsO8}
}

@inproceedings{10.5555/3692070.3692611,
author = {Fernando, Chrisantha and Banarse, Dylan and Michalewski, Henryk and Osindero, Simon and Rockt""{a}schel, Tim},
title = {Promptbreeder: self-referential self-improvement via prompt evolution},
year = {2024},
publisher = {JMLR.org},
booktitle = {Proceedings of the 41st International Conference on Machine Learning},
articleno = {541},
numpages = {64},
location = {Vienna, Austria},
series = {ICML'24}
}

@misc{qwen3_technical_report,
      title={Qwen3 Technical Report}, 
      author={An Yang and Anfeng Li and Baosong Yang and Beichen Zhang and Binyuan Hui and Bo Zheng and Bowen Yu and Chang Gao and Chengen Huang and Chenxu Lv and Chujie Zheng and Dayiheng Liu and Fan Zhou and Fei Huang and Feng Hu and Hao Ge and Haoran Wei and Huan Lin and Jialong Tang and Jian Yang and Jianhong Tu and Jianwei Zhang and Jianxin Yang and Jiaxi Yang and Jing Zhou and Jingren Zhou and Junyang Lin and Kai Dang and Keqin Bao and Kexin Yang and Le Yu and Lianghao Deng and Mei Li and Mingfeng Xue and Mingze Li and Pei Zhang and Peng Wang and Qin Zhu and Rui Men and Ruize Gao and Shixuan Liu and Shuang Luo and Tianhao Li and Tianyi Tang and Wenbiao Yin and Xingzhang Ren and Xinyu Wang and Xinyu Zhang and Xuancheng Ren and Yang Fan and Yang Su and Yichang Zhang and Yinger Zhang and Yu Wan and Yuqiong Liu and Zekun Wang and Zeyu Cui and Zhenru Zhang and Zhipeng Zhou and Zihan Qiu},
      year={2025},
      eprint={2505.09388},
      archivePrefix={arXiv},
      primaryClass={cs.CL},
      url={https://arxiv.org/abs/2505.09388}, 
}

@inproceedings{hover_bench,
    title = "{H}o{V}er: A Dataset for Many-Hop Fact Extraction And Claim Verification",
    author = "Jiang, Yichen  and
      Bordia, Shikha  and
      Zhong, Zheng  and
      Dognin, Charles  and
      Singh, Maneesh  and
      Bansal, Mohit",
    editor = "Cohn, Trevor  and
      He, Yulan  and
      Liu, Yang",
    booktitle = "Findings of the Association for Computational Linguistics: EMNLP 2020",
    month = nov,
    year = "2020",
    address = "Online",
    publisher = "Association for Computational Linguistics",
    url = "https://aclanthology.org/2020.findings-emnlp.309/",
    doi = "10.18653/v1/2020.findings-emnlp.309",
    pages = "3441--3460"
}

@inproceedings{papillon_bench,
    title = "{PAPILLON}: Privacy Preservation from {I}nternet-based and Local Language Model Ensembles",
    author = "Li, Siyan  and
      Raghuram, Vethavikashini Chithrra  and
      Khattab, Omar  and
      Hirschberg, Julia  and
      Yu, Zhou",
    editor = "Chiruzzo, Luis  and
      Ritter, Alan  and
      Wang, Lu",
    booktitle = "Proceedings of the 2025 Conference of the Nations of the Americas Chapter of the Association for Computational Linguistics: Human Language Technologies (Volume 1: Long Papers)",
    month = apr,
    year = "2025",
    address = "Albuquerque, New Mexico",
    publisher = "Association for Computational Linguistics",
    url = "https://aclanthology.org/2025.naacl-long.173/",
    doi = "10.18653/v1/2025.naacl-long.173",
    pages = "3371--3390",
    ISBN = "979-8-89176-189-6"
}

@misc{ifbench_bench,
      title={Generalizing Verifiable Instruction Following}, 
      author={Valentina Pyatkin and Saumya Malik and Victoria Graf and Hamish Ivison and Shengyi Huang and Pradeep Dasigi and Nathan Lambert and Hannaneh Hajishirzi},
      year={2025},
      eprint={2507.02833},
      archivePrefix={arXiv},
      primaryClass={cs.CL},
      url={https://arxiv.org/abs/2507.02833}, 
}

@inproceedings{hotpotqa_bench,
  title={{HotpotQA}: A Dataset for Diverse, Explainable Multi-hop Question Answering},
  author={Yang, Zhilin and Qi, Peng and Zhang, Saizheng and Bengio, Yoshua and Cohen, William W. and Salakhutdinov, Ruslan and Manning, Christopher D.},
  booktitle={Conference on Empirical Methods in Natural Language Processing ({EMNLP})},
  year={2018}
}

@article{beirami2024theoretical,
  title={Theoretical guarantees on the best-of-n alignment policy},
  author={Beirami, Ahmad and Agarwal, Alekh and Berant, Jonathan and D'Amour, Alexander and Eisenstein, Jacob and Nagpal, Chirag and Suresh, Ananda Theertha},
  journal={arXiv preprint arXiv:2401.01879},
  year={2024}
}

@inproceedings{liularge,
  title={Large Language Models to Enhance Bayesian Optimization},
  author={Liu, Tennison and Astorga, Nicol{\'a}s and Seedat, Nabeel and van der Schaar, Mihaela},
  booktitle={The Twelfth International Conference on Learning Representations},
  year={2024}
}

@misc{
singh2025mechanistic,
title={Mechanistic Behavior Editing of Language Models},
author={Joykirat Singh and Subhabrata Dutta and Tanmoy Chakraborty},
year={2025},
booktitle={The Twelfth International Conference on Learning Representations}
}

@misc{kong2025metapromptoptimizationllmbasedsequential,
      title={Meta-Prompt Optimization for LLM-Based Sequential Decision Making}, 
      author={Mingze Kong and Zhiyong Wang and Yao Shu and Zhongxiang Dai},
      year={2025},
      eprint={2502.00728},
      archivePrefix={arXiv},
      primaryClass={cs.LG},
      url={https://arxiv.org/abs/2502.00728}, 
}

@inproceedings{geng2020online,
  title={Online evaluation of audiences for targeted advertising via bandit experiments},
  author={Geng, Tong and Lin, Xiliang and Nair, Harikesh S},
  booktitle={Proceedings of the AAAI Conference on Artificial Intelligence},
  year={2020}
}

@article{Majic2025VibeMarketing,
  author       = {Josipa Majic},
  journal = {The Wall Street Journal},
  year = {2025},
  title        = {VCs Wake Up To Vibe Marketing: AI Reshaping The \$250 Billion Industry},
  url          = {https://www.forbes.com/sites/josipamajic/2025/03/24/vcs-wake-up-to-vibe-marketing-ai-reshaping-the-250-billion-industry/}
}

@article{schwartz2017customer,
  title={Customer acquisition via display advertising using multi-armed bandit experiments},
  author={Schwartz, Eric M and Bradlow, Eric T and Fader, Peter S},
  journal={Marketing Science},
  volume={36},
  number={4},
  pages={500--522},
  year={2017},
  publisher={INFORMS}
}

@inproceedings{
schneider2025hyperbandbased,
title={Hyperband-based Bayesian Optimization for Black-box Prompt Selection},
author={Lennart Schneider and Martin Wistuba and Aaron Klein and Jacek Golebiowski and Giovanni Zappella and Felice Antonio Merra},
booktitle={Forty-second International Conference on Machine Learning},
year={2025},
url={https://openreview.net/forum?id=Lm9DXFrcHD}
}

@inproceedings{agliettifunbo,
  title={FunBO: Discovering Acquisition Functions for Bayesian Optimization with FunSearch},
  author={Aglietti, Virginia and Ktena, Ira and Schrouff, Jessica and Sgouritsa, Eleni and Ruiz, Francisco and Malek, Alan and Bellot, Alexis and Chiappa, Silvia},
  booktitle={Forty-second International Conference on Machine Learning},
  year={2025}
}

@article{aramayo2023multiarmed,
  title={A multiarmed bandit approach for house ads recommendations},
  author={Aramayo, Nicol{\'a}s and Schiappacasse, Mario and Goic, Marcel},
  journal={Marketing Science},
  volume={42},
  number={2},
  pages={271--292},
  year={2023},
  publisher={INFORMS}
}

@article{MantiaChatterjeeLee2025,
  author       = {Mantia, Linda and Chatterjee, Surojit and Lee, Vivian S.},
  title        = {Designing a Successful Agentic AI System},
  journal      = {Harvard Business Review},
  date         = {2025-10-24},
  year         = {2025},
  url          = {https://hbr.org/2025/10/designing-a-successful-agentic-ai-system},
  note         = {Accessed: 2025-11-01}
}

@article{chen2025xbench,
  title={xbench: Tracking Agents Productivity Scaling with Profession-Aligned Real-World Evaluations},
  author={Chen, Kaiyuan and Ren, Yixin and Liu, Yang and Hu, Xiaobo and Tian, Haotong and Xie, Tianbao and Liu, Fangfu and Zhang, Haoye and Liu, Hongzhang and Gong, Yuan and others},
  journal={arXiv preprint arXiv:2506.13651},
  year={2025}
}

@article{qiu2025alita,
  title={Alita-G: Self-Evolving Generative Agent for Agent Generation},
  author={Qiu, Jiahao and Qi, Xuan and Wang, Hongru and Juan, Xinzhe and Wang, Yimin and Zhao, Zelin and Geng, Jiayi and Guo, Jiacheng and Li, Peihang and Shi, Jingzhe and others},
  journal={arXiv preprint arXiv:2510.23601},
  year={2025}
}
